\documentclass[11pt]{article}
\usepackage{mystyle}
\usepackage{fullpage}

\usepackage{amsmath}
\usepackage{amssymb}
\usepackage{mathtools}
\usepackage{mathrsfs}
\usepackage{amsthm}

\usepackage{color}
\usepackage{algorithmic}
\usepackage{algorithm}
\usepackage{titletoc}
\usepackage{pdfpages}

\usepackage{hyperref,etoolbox}
\hypersetup{hidelinks}
\hypersetup{
colorlinks=true,
linkcolor=blue,
citecolor=blue
}

\def\##1\#{\begin{align}#1\end{align}}
\def\$#1\${\begin{align*}#1\end{align*}}

\usepackage{thm-restate}

\usepackage{xcolor}

\usepackage[utf8]{inputenc} 
\usepackage[T1]{fontenc}    
\usepackage{hyperref}       
\usepackage{url}            
\usepackage{booktabs}       
\usepackage{amsfonts}       
\usepackage{nicefrac} 
\usepackage{enumitem}
\usepackage{microtype}      
\usepackage{xcolor}         
\usepackage{graphicx}
\usepackage{titlesec}
\usepackage{subfigure}

\usepackage{titletoc}
\allowdisplaybreaks[4]
\usepackage{booktabs} 
\usepackage{caption}
\theoremstyle{plain}

\usepackage[textsize=tiny]{todonotes}
\usepackage{ifthen}
\newcommand{\compilehidecomments}{true}
\ifthenelse{ \equal{\compilehidecomments}{false} }{
	\newcommand{\nuoya}[1]{}
	\newcommand{\zhuoran}[1]{}
	\newcommand{\zhaoran}[1]{}
        \newcommand{\zhihan}[1]{}
}{
\newcommand{\nuoya}[1]{{\color{purple}[\text{Nuoya:} #1]}}
\newcommand{\zhuoran}[1]{{\color{red}[\text{zhuoran:} #1]}}
\newcommand{\zhihan}[1]{{\color{yellow} [\text{Zhihan:} #1]}}
\newcommand{\zhaoran}[1]{{\color{blue}[\text{zhaoran:} #1]}}
}

\newcommand{\compilefullversion}{true}
\ifthenelse{\equal{\compilefullversion}{false}}{%
	\newcommand{\OnlyInFull}[1]{}
	\newcommand{\OnlyInShort}[1]{#1}
}{%
	\newcommand{\OnlyInFull}[1]{#1}%
	\newcommand{\OnlyInShort}[1]{}%
}%

\let\hat\widehat
\let\tilde\widetilde

\usepackage{color,xcolor}
\usepackage{colortbl}
\usepackage{bbm}
\usepackage{tcolorbox}

\definecolor{red1}{HTML}{f47983}
\definecolor{blue1}{HTML}{3eede7}
\definecolor{yellow1}{HTML}{f5dd6f}

\title{\LARGE Sample-Efficient Multi-Agent RL: An Optimization Perspective}
\author{Nuoya Xiong\thanks{Equal contribution}~\thanks{IIIS, Tsinghua University. Email: \texttt{xiongny20@mails.tsinghua.edu.cn}} \qquad Zhihan Liu$^*$\thanks{Northwestern University. Email: \texttt{zhihanliu2027@u.northwestern.edu}}\qquad Zhaoran Wang\thanks{Northwestern University. Email: \texttt{zhaoranwang@gmail.com}} \qquad  Zhuoran Yang\thanks{Yale University. Email: \texttt{zhuoran.yang@yale.edu}}  }
\date{}
\begin{document}
\maketitle
\begin{sloppypar}
\begin{abstract}
    We study multi-agent reinforcement learning (MARL) for the general-sum Markov Games (MGs) under the general function approximation. 
    In order to find the minimum assumption for sample-efficient learning, we introduce a novel complexity measure called the Multi-Agent Decoupling Coefficient (MADC) for general-sum MGs. Using this measure, we propose the first unified algorithmic framework that ensures sample efficiency in learning Nash Equilibrium, Coarse Correlated Equilibrium, and Correlated Equilibrium for both model-based and model-free MARL problems with low MADC. We also show that our algorithm provides comparable sublinear regret to the existing works. Moreover, our algorithm combines an equilibrium-solving oracle with a single objective optimization subprocedure that solves for the regularized payoff of each deterministic joint policy, which avoids solving constrained optimization problems within data-dependent constraints \citep{jin2020sample, wang2023breaking} or executing sampling procedures with complex multi-objective optimization problems \citep{foster2023complexity}, thus being more amenable to empirical implementation. 
\end{abstract}
\newpage
\begin{Large} \textbf{Contents}\end{Large}

\begin{large} \startcontents
\printcontents{}{1}{}\end{large}
\newpage


\section{Introduction}
Multi-agent reinforcement learning (MARL) has achieved remarkable empirical successes in solving complicated  games involving sequential and strategic decision-making across multiple agents   
\citep{vinyals2019grandmaster, brown2018superhuman, silver2016mastering}. 
These achievements have catalyzed many research efforts focusing on developing efficient MARL algorithms in a  theoretically principled manner. 
Specifically, a multi-agent system is typically modeled as a  general-sum Markov Game  (MG)  \citep{littman1994markov},
with the primary aim of efficiently discerning a certain equilibrium notion among multiple agents from data collected via online interactions.
Some popular equilibrium notions include Nash equilibrium (NE), correlated equilibrium (CE), and coarse correlated equilibrium (CCE).



However, multi-agent general-sum Markov Games (MGs) bring forth various challenges.
In particular, empirical application suffers from the large state space. 
Such a challenge necessitates the use of the function approximation as an effective way to extract the essential features of RL problems and avoid dealing directly with the large state space. 
Yet, adopting function approximation in a general-sum MG brings about additional complexities not found in single-agent RL or a zero-sum MG. 
Many prevailing studies on single-agent RL or two-agent zero-sum MGs with the function approximation
  leverage the special relationships between the optimal policy and the optimal value function  \citep{jin2021bellman, du2021bilinear, zhong2022gec, jin2022power, huang2021towards}.
In particular, 
in single-agent RL, 
the optimal policy is the greedy policy with respect to the optimal value function. 
Whereas in a two-agent zero-sum MG,
the Nash equilibrium is obtained by solving a minimax estimation problem based on the optimal value function. 
Contrastingly, in a general-sum MG, individual agents possess distinct value functions, and thus there exists no unified optimal value function that characterizes the equilibrium behavior.
Moreover, unlike a zero-sum MG, a general-sum 
MG can admit diverse equilibrium notions, where each corresponds to a set of policies. Consequently, methodologies developed for single-agent RL or zero-sum MGs cannot be directly extended to general-sum MGs.

Recently, several works propose sample-efficient  RL algorithms for general-sum MGs. 
In particular, \cite{chen2022unified, foster2023complexity} 
propose model-based 
algorithms for learning NE/CCE/CE based on  multi-agent extensions  of the Estimation-to-Decision algorithm 
\citep{foster2021statistical}, and they establish regret upper bounds in terms of complexity metrics that extend  Decision-Estimation Coefficient \citep{foster2021statistical} to MGs.
In addition, \cite{wang2023breaking} study model-free RL for general-sum MGs with the general function approximation. 
They focus on developing a  decentralized and no-regret algorithm that finds a CCE. 
Thus, it seems unclear how to design a provably sample-efficient MARL algorithm for NE/CCE/CE for general-sum MGs in a \emph{model-free} manner. 
Furthermore, motivated by the recent development in single-agent RL  \citep{jin2021bellman, du2021bilinear, zhong2022gec, foster2021statistical, liu2023objective}, we aim to develop a unified algorithmic framework for MARL that covers both model-free and model-based approaches. 
Thus, 
we aim to address the following questions:

\begin{center}
Can we design a  unified algorithmic framework for general-sum MGs such that  
(i) it is provably sample-efficient in learning NE/CCE/CE in the context of the function approximation and (ii)  it covers both model-free and model-based MARL approaches?
\end{center}

In this paper, we provide an affirmative answer to the above questions. 
Specifically, we propose a unified algorithmic framework named \textbf{M}ulti-\textbf{A}gent \textbf{M}aximize-to-\textbf{EX}plore (MAMEX) for general-sum MGs with the general function approximation.   
MAMEX 
extends the framework of   Maximize-to-Explore    \citep{liu2023objective} to  
general-sum MGs 
by employing it together with an 
equilibrium solver for general-sum normal-form games defined over the policy space. 

Maximize-to-Explore (MEX) is a class  of  RL algorithms for single-agent MDP and two-agent zero-sum MGs  
where each new policy is updated by solving an optimization problem involving a hypothesis $f$, which can be regarded as the action-value function in the model-free version and the transition model in the model-based version.
The optimization objective of MEX contains two terms ---  (a) the optimal value with respect to the hypothesis $f$ and (b) a loss function computed from data that quantifies how far $f$ is from being the true hypothesis.
Here, the term (a) reflects the planning part of online RL and leverages the fact that the optimal policy is uniquely characterized by the given hypothesis. 
On the other hand, the term (b), which can be the mean-squared Bellman error or log-likelihood function, reflects the estimation part of online RL.
By optimizing the sum of (a) and (b) over the space of hypotheses \emph{without any data-dependent constraints}, MEX balances exploitation with exploration in the context of the function approximation.


However, the first term in MEX's optimization objective leverages the fact that the optimal policy can be uniquely constructed from the optimal value function or the true model, using a greedy step or dynamic programming. 
Such a nice property cannot be extended to general-sum MGs, where the relationship between the equilibrium policies and value function is more complicated, and each agent has its own value function. As a result,  it is impractical to construct a single-objective optimization problem in the style of MEX over the hypothesis space for general-sum MGs.


Instead of optimizing over the spaces of hypotheses, 
MAMEX optimizes over the policy space. 
Specifically, 
in each iteration, MAMEX updates the joint policy of all agents by solving for a desired equilibrium  (NE/CCE/CE) of a \emph{normal-form game}, where the pure strategies are a class of joint policies of the  $n$  agents, e.g., the class of deterministic joint policies.
Besides, for each pure strategy of this normal form game, the corresponding payoff function is obtained by solving a regularized optimization problem over the hypothesis space  \`{a} la  MEX. 
Thus,  policy updates in MAMEX  involve the following  two steps:
\begin{itemize}
    \item [(i)] For each pure strategy $\pi$, construct the payoff function $\overline V_i (\pi)$ for each agent $i$ by solving an  unconstrained and  regularized optimization problem;
    \item [(ii)] Compute the NE/CCE/CE of the normal-form game over the space of pure strategies with payoff functions $\{ \overline {V_i} (\pi) \}_{i=1}^n$, where $n$ is the number of agents. 
\end{itemize}
The implementation of MAMEX only requires an oracle for solving a single-objective and unconstrained optimization problem and an oracle for solving NE/CCE/CE of a normal-form game. 
Compared to existing works that either solve  constrained optimization subproblems within  data-dependent constraints \citep{wang2023breaking},  or complex multi-objective or minimax optimization subproblems \citep{foster2023complexity, chen2022unified}, 
MAMEX is more amenable to practical implementations. 
Furthermore, step (i) of MAMEX resembles MEX,
which enables both model-free and model-based instantiations.  


We prove that MAMEX is provably sample-efficient in a rich class of general-sum MGs. 
To this end, \
we introduce a novel complexity measure named \textbf{M}ulti-\textbf{A}gent \textbf{D}ecoupling \textbf{C}oefficient (MADC) to capture the exploration-exploitation tradeoff in MARL. 
Compared to  the decoupling coefficient and its variants \citep{dann2021provably, agarwal2022model, zhong2022gec} proposed for the single-agent setting, 
  MADC characterize the 
hardness of exploration in MGs in terms of the discrepancy between the out-of-sample \textit{prediction error}  and the in-sample 
\textit{training error} incurred by minimizing a discrepancy function $\ell$ on the historical data. 
MADC is defined based on the intuition that 
if a hypothesis attains a small training error on a well-explored dataset, it would also incur a small prediction error. 
When the MADC of an MG instance  is small, 
achieving a small training error ensures a small prediction error, and thus exploration is relatively easy.  We prove that MAMEX  achieves a   sublinear regret for learning NE/CCE/CE in classes with small MADCs, which includes multi-agent counterparts of models with low  Bellman eluder dimensions \citep{jin2021bellman, jin2022power, huang2021towards},  Bilinear Classes  \citep{du2021bilinear}, and models with low witness ranks \citep{sun2019model, huang2021towards}. 
When specialized to specific members within these classes, MAMEX yields comparable regret upper bounds 
to existing works.

\vspace{5pt}
{\noindent \bf Our Contributions.} In summary, our contributions are two-fold. 

$\bullet$ First,  we provide a unified algorithmic framework named \textbf{M}ulti-\textbf{A}gent \textbf{M}aximize-to-\textbf{EX}plore (MAMEX) for both model-free and model-based MARL, which is sample-efficient in finding the NE/CCE/CE in general-sum MGs with small MADCs. 
    Moreover, MAMEX leverages an equilibrium-solving oracle for normal-form games defined over a class of  joint policies for policy updates, and a 
    single-objective optimization procedure that solves for the payoff functions of these normal-form games. 
     To our best knowledge, the model-free version of  MAMEX is the first model-free algorithm for general-sum MGs that learns all three equilibria NE, CCE, and CE with sample efficiency.

$\bullet$    Second, we introduce a complexity measure,  \textbf{M}ulti-\textbf{A}gent \textbf{D}ecoupling \textbf{C}oefficient (MADC), to quantify the hardness of exploration in a general-sum MG in the context of the function approximation. 
    The class of MGs with  low MADCs  includes  a rich class of MG instances, 
    such as multi-agent counterparts of models with low  Bellman eluder dimensions \citep{jin2021bellman, jin2022power,huang2021towards},  Bilinear Classes \citep{du2021bilinear}, and models with low witness ranks \citep{sun2019model, huang2021towards}. 
    When specialized to specific  MG instances in these classes, we achieve comparable regret upper bounds to existing works.

\section{Related Work}\label{appendix: related work}
\vspace{5pt}
\paragraph{Markov Games}
 Markov Game (MG) \citep{littman1994markov} is a popular model of multi-agent reinforcement learning, which generalizes the Markov decision process to multiple agents.
A series of recent works design the sample-efficient algorithm for two-agent zero-sum games \citep{wei2017online,zhang2020model,xie2020learning,bai2020near,bai2020provable,bai2021sample,zhao2021provably,huang2021towards,jin2022power, chen2022unified,chen2022almost}. 
For instance, \cite{bai2020provable} provide a sample-efficient algorithm in an episodic MG based on optimistic value iteration. \cite{xie2020learning,chen2022almost} 
mainly focus on zero-sum MGs with a linear structure. 
\cite{huang2021towards,jin2022power,chen2022unified} further consider the two-player zero-sum MGs under general function approximation, and provide algorithms with a sublinear regret. 
Another line of research focuses on general-sum MGs with multiple players \citep{jin2020sample,liu2021sharp,tian2021online,jin2021v,song2021can,liu2022sample,daskalakis2022complexity,zhan2022decentralized,cui2023breaking,wang2023breaking}. Some of previous works \citep{liu2021sharp, tian2021online, liu2022sample} 
consider learning all three equilibrium notions ---   NE, CCE, and CE --- and their regret or sample complexity results are exponential in the number of agents. 
To break this exponential curse, some existing works propose decentralized algorithms for learning CCE or CE rather than NE \citep{jin2021v,daskalakis2022complexity,zhan2022decentralized,cui2023breaking,wang2023breaking}. 


\vspace{5pt}\paragraph{MARL with Function Approximation}
There are many papers working on multi-player general-sum MGs with the function approximation \citep{zhan2022decentralized,ni2022representation,chen2022unified,wang2023breaking,cui2023breaking, foster2023complexity} that build upon previous works for function approximation in the single-agent setting \citep{jiang2017contextual,sun2019model,jin2020provably,wang2020reinforcement,dann2021provably,du2021bilinear, jin2021bellman,foster2021statistical,chen2022abc,agarwal2022model, zhong2022gec,liu2023objective}. 
In recent years, \cite{chen2022unified} and  \cite {foster2023complexity}  generalize the complexity measure Decision-Estimation Coefficient (DEC), and learn the equilibria in model-based general-sum MGs. 
\cite{ni2022representation} provide both a model-based algorithm and a model-free algorithm for the low-rank MGs. Some previous works \citep{zhan2022decentralized,wang2023breaking,cui2023breaking}
 provide model-free algorithms that learn CCE and CE with polynomial sample complexity. Compared to their works, this paper provides a unified algorithmic framework for both model-free and model-based MARL problems, which learns NE/CCE/CE  efficiently under general function approximation and provides comparable regret to existing works. 
 In particular, our work seems to provide the first model-free algorithm for learning NE/CCE/CE of general-sum MGs in the context of the general function approximation.

\vspace{-0.5em}
\section{Models and Preliminaries}

\subsection{Notation}
For $n$ sets $\cF_1,\cdots, \cF_n$, we let $\otimes _{i=1}^n \cF_{i} $ denote $ \cF_1\times \cdots\times \cF_n$. For a set $\cA$, we denote $\Delta(\cA)$ as a set of probability distributions over $\cA$.
For a vector $x \in \RR^n$, we denote $\|x\|_1 = \sum_{i=1}^n |x_i|$, $\|x\|_2 = \sqrt{\sum_{i=1}^n x_i^2}$ and $\|x\|_\infty = \max_{i=1}^n |x_i|$. 
For a function $f:\cX \mapsto \cY$, we denote $\|f\|_\infty = \sup_{x \in \cX} |f(x)|$ as the infinity norm. 
For two functions $f,g:\cA \mapsto \RR$, we denote $\langle f, g\rangle_{\cA} = \EE_{a \in \cA}[f(x)g(x)]$ as the inner product with respect to the set $\cA$.  For a Hilbert space $\cV$ and $f, g \in \cV$, we denote $\langle f,g \rangle_{\cV}$  as the inner product defined in the Hilbert space $\cV,$ and $\|f\|_\cV$ is the norm defined in Hilbert space $\cV.$ For two distributions over $P,Q \in \Delta(\cX)$, the Hellinger distance is defined as $D_\text{H}^2(P\|Q)=\frac{1}{2}\EE_{x \sim P}[(\sqrt{dP(x)/dQ(x)}-1)^2].$ For a vector $ x\in \RR^d$, the softmax mapping is denoted by $\text{Softmax}(x) \in \RR^d$ with $\big(\text{Softmax}(x)\big)_i = e^{x_i}/\sum_{i \in [d]}e^{x_i}$. 

\subsection{Markov Games}

\textbf{General-Sum Markov Games}
In this work, we consider  general-sum Markov Games (MGs) in the episodic setting, which is  
denoted by  a tuple $(\cS,H, \cA,\{r_{h}^{(i)}\}_{i \in [n], h \in [H]},$ $ \{\PP_h\}_{h \in [H]}, \rho )$, 
where $n$ is the number of agents, $H$ is the length of one episode,  $\cS$ is the state set, and 
$\cA = \otimes _{i=1}^n \cA_{i} $ is the joint action set.
Here, $\cA_i$ is the action set of the agent $i$. 
Moreover, $r_h^{(i)}:\cS \times \cA \mapsto \RR$ is the known reward function\footnote{Our results can be extended to the unknown stochastic reward case \citep{agarwal2022model,zhong2022gec}. 
Note that learning the transition kernel is more difficult than learning the reward.} 
of the agent $i$ at step $h$,
 $\PP_h \colon \cS \times \cA \rightarrow \Delta(\cS)$ is the transition kernel at the $h$-th step, and $\rho  \in \Delta(\cS)$ is the distribution of the initial  state $s_1$. 
We assume the $n$ agents observe the same state at each step and each agent $i$ chooses an action within its own action set  $\cA_i$ simultaneously. 
In each episode, starting from $s_1 \sim p_0$, for each $h \in [H]$, the agents choose their joint action $a_h \in \cA$ in state $s_h$,
where $a_h = (a_h^{(1)}, \ldots, a_h^{(n)})$. Then, each agent $i$
receives its  own reward $r_h^{(i)}(s_h,a_h)$, and the game move to the next state $s_{h+1} \sim \PP_h(s_{h+1}\mid s_h,a_h) $. 
Moreover,  we assume $\sum_{h=1}^H r_h^{(i)}(s_h,a_h) \in[0,R]$ for any possible state-action sequences for some $1\leq R\leq H$. 


In MGs, the agents' policy can be stochastic and correlated. 
To capture such a property, we introduce the notion of \emph{pure policy} and \emph{joint policy} as follows. 
For each agent $i$, its local (Markov) policy maps a state $s$ to a distribution over the local action space $\cA_i$. 
We let $\Pi_i ^{ \textrm{pur}}\subseteq \{\pi:\cS \mapsto \Delta(\cA_i)\}$ denote a subset of the agent $i$'s local policies, which is called the set of  Markov \textit{pure policies}. 
We assume the agent $i$'s policy is a random variable taking values in $\Pi_i^{\textrm{pur}}$. 
Specifically, let $\omega \in \Omega$ be the random seed. 
The \textit{random} policy $\pi^{(i)} = \{ \pi^{(i)}_{h} \}_{h\in [H]}$ for the agent $i$ contains  $H$ mappings $\pi^{(i)}_{h} : \Omega \mapsto \Pi_i^{\text{pur}}$ such that $\pi^{(i)}_{h} (\omega) \in \Pi_i^{ \textrm{pur}}$ is a pure policy. 
To execute $\pi^{(i)}$, the agent $i$ first samples a random seed $\omega \in \Omega$, and then follows the policy $\pi^{(i)}_{h}(\omega)$ for all $h\in [H]$. 
 The \emph{joint policy} $\pi$ of the $n$ agents is a set of policies $\{\pi^{(i)}\}_{i=1}^n$ that all agents share \emph{the same random seed $\omega$}. 
 In other words, $\{\pi_h^{(i)} (\omega) \}_{i\in[n]}\in \otimes_{i=1}^n \Pi_i^{ \textrm{pur}}$ are  random policies of the $n$ agents whose  randomness is correlated  by the random seed $\omega$.
 Equivalently, we can regard $\pi$ as a random variable over $\otimes_{i=1}^n \Pi_i^{ \textrm{pur}}$. 
 We let $\pi_h(a\mid s)$ denote the probability of taking action $a$ in the state $s$ at step $h$.
 Furthermore, a special class of joint policy is the \textit{product policy}, where each agent executes their own policies independently. In other words, we have $\omega = (\omega_1, \ldots, \omega_n),$ where $\omega_1, \ldots, \omega_n$ are independent, and each $\pi^{(i)}$ depends on $\omega_i$ only. As a result, we have $\pi_h(a\mid s) = \prod_{i=1}^n \pi^{(i)}_h (a^{(i)}\mid s)$ for any product policy $\pi$.

Furthermore, using the notion of pure policy and joint policy, we can equivalently view the MG as a normal form game over $\Pi^{ \textrm{pur}} = \otimes_{i=1}^n \Pi_i^{ \textrm{pur}}$. 
That is, each pure policy can be viewed as a pure strategy of the normal form game, and each joint policy can be viewed as a mixed strategy. 
Such a view is without loss of generality, because we can choose $\Pi_i^{ \textrm{pur}}$ to be the set of all possible deterministic policies of the agent $i$. Meanwhile, using a general $\Pi_i^{ \textrm{pur}}$, we can also incorporate parametric policies as the pure policies, e.g., log-linear policies \citep{xie2021bellman,yuan2022linear,cayci2021linear}.

The value function $V_h^{(i),\pi}$ is the expected cumulative rewards received by  the agent $i$ from step $h$ to step $H$, when all the agents follow a joint policy $\pi$, which is defined as
\begin{align*}
    V_h^{(i),\pi}(s) = \EE_{\pi}\Big[\sum_{h'=h}^H r_{h'}^{(i)}(s_{h'},a_{h'})\Big|\, s_h=s\Big].
\end{align*}
We let $V^{(i), \pi}(\rho) = \EE_{s\sim \rho} [V_1^{(i),\pi}(s)]$ denote the agent $i$'s expected cumulative rewards  within the whole episode. 
Besides, the corresponding $Q$-function (action-value function) can be written as
 \begin{align}\label{eq:def_Qfunc}
     Q_h^{(i),\pi}(s,a) = \EE_{\pi}\Big[\sum_{h'=h}^H r_{h'}^{(i)}(s_{h'},a_{h'})\Big|\,  s_h = s, a_h = a\Big].
 \end{align}

 For a joint policy $\pi$ and any agent $i$,
  we let $\pi^{(-i)}$ denote the joint policy excluding the agent $i$. Given $\pi^{(-i)}$, the \textit{best response} of  the agent $i$ is defined as $\pi^{(i),\dag} = \argmax_{\nu \in \Delta(\Pi_i^{\text{pur}})} V ^{(i),\nu\times \pi^{(-i)}}(\rho )$, which is random policy of the agent $i$ that maximizes its expected rewards when other agents follow $\pi^{(-i)}$. 
  Besides, we denote $\mu^{(i),\pi}=(\pi^{(i),\dag}, \pi^{(-i)})$.


 \vspace{5pt}
\noindent\textbf{Online Learning and Solution Concepts}
We focus on three common  equilibrium notions in the game theory: Nash Equilibrium (NE), Coarse Correlated Equilibrium (CCE) and Correlated Equilibrium (CE).

First, a NE of a game is a \textit{product} policy that no individual player can improve its expected cumulative rewards by unilaterally deviating its local policy.
\begin{definition}[$\varepsilon$-Nash Equilibrium]
   A \textit{product  policy} $\pi$ is an $\varepsilon$-Nash Equilibrium if $V ^{(i), \mu^{(i),\pi}}(\rho ) \le V ^{(i),\pi} (\rho )+ \varepsilon$ for all $i \in [n]$, where  $\mu^{(i),\pi} = (\pi^{(i),\dag}, \pi^{(-i)})$ and $\pi^{(i),\dag}$ is the best response policy  with respect to  $\pi^{(-i)}$.
\end{definition}

In other words, a product policy $\pi$ is an  $\varepsilon$-Nash Equilibrium if and only if 
\begin{align*}
    \max_{i\in [n]}  \Bigl \{ \max_{\nu \in \Delta(\Pi_i^{\text{pur}})} V ^{(i),\nu\times \pi^{(-i)}}(\rho ) -  V ^{(i),\pi} (\rho )  \Bigr\}  \leq \varepsilon. 
\end{align*}

 In this work, we design algorithms for the online and self-play setting.
That is, we control the joint policy all agents,   interact with the environment over $K$ episodes, and aim to learn the desired equilibrium notion from bandit feedbacks. 
To this end, let  $\pi^k$ denote the joint policy that the agents execute in the $k$-th episode, $k \in [K]$.  
We define the Nash-regret as the cumulative suboptimality across all agents with respect to   NE.

\begin{definition}[Nash-Regret]\label{def:N regret}
    For all $k\in [K]$, let $\pi^k$ denote the \textit{product  policy}   deployed in the $k$-th episode, then the Nash-regret is defined as 
    \begin{align*}
\mathrm{Reg}_{\mathrm{NE}}(K) = \sum_{k=1}^K \sum_{i=1}^n \bigl(V ^{(i),\mu^{(i),\pi^k}}(\rho)-V ^{(i),\pi^k}(\rho)\bigr).
    \end{align*}
\end{definition}

A Coarse Correlated Equilibrium is a \textit{joint} policy $\pi$ such that no agent can achieve higher rewards by only changing its local policy.
Compared with a NE, a CCE allows different agents to be correlated, while NE only considers product policies. 

\begin{definition}[$\varepsilon$-Coarse Correlated Equilibrium]
    A \textit{joint  policy} $\pi$ is a $\varepsilon$-Coarse Correlated Equilibrium if $V^{(i),\mu^{(i),\pi}}(\rho) \le V^{(i),\pi}(\rho)+\varepsilon$ for all $i \in [n]$.
\end{definition}

Here, the definition of $\varepsilon$-CCE is similar to that of an $\varepsilon$-NE. But here $\pi$ is a joint policy, i.e., the randomness of the local  policies of the $n$ agents can be coupled together. As a result, CCE is a more general equilibrium notion than NE.
Similarly, we can define the CCE-regret, which represents the cumulative suboptimality across all agents with respect to   CCE.  
\begin{definition}[CCE-Regret]\label{def:CCE regret}
    For all $k\in [K]$, let $\pi^k$ denote the \textit{joint} policy that is  deployed in the $k$-th episode, then the CCE-regret is defined as 
    \begin{align*}
        \mathrm{Reg}_{\mathrm{CCE}}(K) = \sum_{k=1}^K \sum_{i=1}^n \bigl(V ^{(i),\mu^{(i),\pi^k}}(\rho)-V ^{(i),\pi^k}(\rho )\bigr).
    \end{align*}
\end{definition}

Last, the Correlated Equilibrium has been extensively studied in previous works for MARL \citep{jin2020sample,chen2022unified,cui2023breaking, wang2023breaking}. 
To introduce the concept of CE, we need first to introduce the \textit{strategy modification}. A strategy modification for the agent $i$ is a mapping $\phi_i:\Pi_i^{\text{pur}}\to \Pi_i^{\text{pur}}$. Given any random policy $\pi$, the best strategy modification for agent $i$ is defined as  $\argmax_{\phi_i}\EE_{\upsilon\sim \pi}[V ^{\phi_i(\upsilon^{(i)})\times \upsilon^{(-i)}} (\rho)]$. 
A CE is a joint policy $\pi$ such that no agent can achieve higher rewards by only changing its local policy through strategic modification. 
\begin{definition}[$\varepsilon$-Correlated Equilibrium]
    A \textit{joint  policy} $\pi$ is a  $\varepsilon$-Correlated Equilibrium if $\max_{\phi_i} \EE_{\upsilon\sim \pi}[V ^{\phi_i(\upsilon^{(i)}) \times \upsilon^{(-i)}} (\rho) ] \le   V ^{\pi }(\rho)  + \varepsilon$ for any agent $i \in [n]$. 
\end{definition}

We can similarly define  CE-regret as the sum of suboptimality terms with respect to CE. 
\begin{definition}[CE-Regret]
    For any $k\in [K]$, let  $\pi^k$ denote the joint policy that is deployed in the $k$-th episode,   the CE-regret is defined as 
    \begin{align*}
        \mathrm{Reg}_{\mathrm{CE}}(K) = \sum_{k=1}^K \sum_{i=1}^n \left(\max_{\phi_i}\EE_{\upsilon\sim \pi^k}\left(V ^{(i),\phi_i(\upsilon^{(i)})\times \upsilon^{(-i)}} (\rho )\right) - V ^{(i),\pi^k}(\rho)\right).
    \end{align*}
\end{definition}

Compared to the NE/CCE regret, the strategy modification of one agent in CE can be correlated to the policies of other agents. Instead, the best response is independent of the other agents.

We note that the definitions of NE, CCE, and CE align with those defined on the normal form game defined on the space of   pure policies. That is, each agent $i$'s ``pure strategy'' is a pure policy $\pi^{(i)} \in \Pi_i^{\text{pur}}$, and  the   ``payoff'' of the agent $i$ when the ``mixed strategy'' is $\pi$ is given by    $V^{(i),\pi } (\rho)$.



\subsection{Function Approximation}
To handle the large state space in MARL, we assume the access to a hypothesis class $\cF$, which captures the $Q$ function in the model-free setting and the transition kernel in the model-based setting. 

\vspace{5pt}
\noindent\textbf{Model-Based Function Approximation} In the model-based setting, the hypothesis class $\cF$ contains the model (transition kernel) of MGs. Specifically, we let $ \PP_f = \{\PP_{1,f} \cdots, \PP_{H,f}\}$ denote the transition kernel parameterized by $f \in \cF$. 
When the model parameters are $f$ and the   joint policy is $\pi$, we denote the value function  and $Q$-function of the agent $i$  at the $h$-th step as  $V_{h,f}^{(i),\pi}(s)$ and  $Q_{h,f}^{(i),\pi}(s,a)$ respectively. 
We have the Bellman equation $Q_{h,f}^{(i),\pi}(s,a)=r_{h}^{(i)}(s,a) + \EE_{s'\sim \PP_{h,f}(\cdot \mid s,a)}[V_{h+1,f}^{(i),\pi}(s')]$. 


\vspace{5pt}
\noindent\textbf{Model-Free Function Approximation}
In the model-free setting, 
we let 
$\cF = \otimes_{i=1}^n \cF^{(i)}=\otimes_{i=1}^n (\otimes_{h=1}^H \cF_h^{(i)})$ be  a  class of $Q$-functions of the $n$ agents, where  $\cF_h^{(i)}=\{f_h^{(i)}: \cS\times \cA \mapsto \RR\}$ is a class  of $Q$-functions of the agent $i$ at the $h$-th step.
For any $f\in \cF$, we denote 
$Q_{h,f}^{(i)}(s,a)= f_h^{(i)}(s,a)$ for all $i\in[n]$ and $h\in [H]$. 
Meanwhile, for any joint policy $\pi$ and any $f \in \cF$, 
we define 
 $$V_{h,f}^{(i),\pi}(s) = \EE_{a \sim \pi(s)}[f_h^{(i)}(s, a)] = \langle f_h^{(i)}(s,\cdot ), \pi_h(\cdot \mid s)\rangle _{\cA}.
 $$ 
Furthremore, for any joint policy $\pi$,  agent $i$, and step $h$, we define the Bellman operator $\cT_h^{(i),\pi}$ by letting  
\begin{align}\label{eq:define_bellman_operator}
    (\cT_h^{(i),\pi}(f_{h+1}))(s,a) = r_h^{(i)} (s,a) + \EE_{s' \sim \PP_h(s'\mid s,a)}\langle f_{h+1}(s',\cdot), \pi_{h+1}(\cdot \mid s')\rangle_{\cA}, \quad \forall f\in \cF^{(i)}. 
\end{align}
Note that the   Bellman operator depends on the index $i$ of the agent because the reward functions of the agents are different. 
Such a definition is an extension of the Bellman evaluation operator in the single-agent setting \citep{puterman2014markov} to the multi-agent MGs. By definition, $\{Q^{(i),\pi}_{h }\}$ defined in \eqref{eq:def_Qfunc}  is the fixed point of $\cT_h^{(i),\pi}$, i.e., $Q^{(i),\pi}_{h } = \cT_h^{(i),\pi}(Q^{(i),\pi}_{h+1 })$ for all $h\in [H]$.



For both the model-based and the model-free settings, we impose the realizability assumption, which requires that the hypothesis space $\cF$ is sufficiently expressive such that it contains the true transition model or the true $Q$-functions. Besides, for the model-free setting, we also require that the hypothesis classes be closed with respect to the Bellman operator.

\begin{assumption}[Realizability and Completeness]\label{assum: realizability}
For the model-based setting, we assume the true transition model $f^*$ lies in the hypothesis class $\cF$. 
Besides, for the model-free setting, for any \emph{pure policy} $\pi$ and any $i\in [n]$, we assume that $Q^{(i), \pi  } \in \cF^{(i)}$ and  $\cT^{(i),\pi}_h\cF_{h+1}^{(i)}\subseteq \cF_h^{(i)}$ for all $h \in [H]$. 
\end{assumption}

\vspace{5pt}
{\noindent \bf Covering Number and Bracketing Number.}
When a function class $\cF$ is infinite, the $\delta$-covering number and the $\delta$-bracketing number serve as surrogates of the cardinality of $\cF$. 
Such a tool is common in supervised learning.  
\begin{definition}[$\delta$-Covering Number]
    The $\delta$-covering number of  a function class $\cF$ with respect to distance metric $d$, denoted  as $\cN_\cF(\delta, d)$, is the minimum integer $q$ satisfying the following property: there exists a subset $\cF'\subseteq \cF$ with $|\cF'|=q$ such that  for any $f_1 \in \cF$ we can find $f_2 \in \cF'$ with 
    $d(f_1,f_2) \le \delta$. To simplify the notation,  we write  $\cN_\cF(\delta, \|\cdot\|_\infty)$ as $\cN_\cF(\delta).$
\end{definition}

\begin{definition}[$\delta$-Bracketing Number]
    A $\delta$-bracket of size $N$ is a bracket $\{g_1^i,g_2^i\}_{i=1}^N $, where $g_1^i$ and $g_2^i$ are functions mapping any policy $\pi$ and trajectory $\tau$ to $\RR$, such that for all $ i \in [N]$, $\pi \in \Pi$ we have $\|g_1^i(\pi,\cdot) - g_2^i (\pi,\cdot) \| \le \delta$. Also, for any $f \in \cF,$ there must exist an $i \in [N]$ such that  $g_1^i(\pi,\tau_H) \le \PP_f^\pi(\tau_H) \le g_2^i (\pi,\tau_H)$ for all possible $\tau_H$ and $\pi.$ The $\delta$-bracketing number of $\cF$, denoted by $\cB_\cF(\delta)$, is the minimum size of a $\delta$-bracket.  
\end{definition}

\vspace{5pt}
\noindent\textbf{Multi-Agent Decoupling Coefficient} 
Now we introduce a key complexity measure --- multi-agent decoupling coefficient (MADC) --- 
which captures the hardness of exploration in MARL.
Such a notion is an extension of the decoupling coefficient \citep{dann2021provably} to general-sum MGs. 

\begin{definition}[Multi-Agent Decoupling Coefficient] \label{def:decoup_coeff}
The Multi-Agent Decoupling Coefficient of a MG is defined as the smallest constant $d_{\mathrm{MADC}}\ge 1$ such that for any $i \in [n]$, $\mu>0$, $\{f^{k}\}_{k \in [K]}\subseteq \cF^{(i)}$, and $\{\pi^k\}_{k \in [K]}\subseteq   \Pi ^{\mathrm{pur} }$ 
the following inequality holds: 
    \begin{align}\label{eq:MADC}
        \underbrace{\sum_{k=1}^{K}(V_{ f^k}^{(i),\pi^k}(\rho)-V ^{(i),\pi^k}(\rho))}_{\displaystyle \text{prediction error}} \le  \frac{1}{\mu} \underbrace{\sum_{k=1}^K \sum_{s=1}^{k-1}\ell^{(i),s}(f^k,\pi^k)}_{\displaystyle \text{training error}} +\underbrace{\mu \cdot d_{\mathrm{MADC}}+ 6d_{\mathrm{MADC}}H }_{\displaystyle \text{gap}}, 
    \end{align}
    where  we define $V_{ f^k}^{(i),\pi^k}(\rho) = \EE_{s_1\sim \rho } [V_{1,f^k}^{(i),\pi^k}(s_1)]$, and $\ell^{(i),s}(f^k,\pi^k)$ is a discrepancy function that measures the inconsistency between $f^k$ and $\pi^k$, on the historical data. 
    The specific definitions of $\{ \ell^{(i),s}\}_{i\in[ n], s\in[K-1]}$
    under the model-free and model-based settings are given in  \eqref{eq:modelfree discrepancy} and  \eqref{def:modelbased discrepancy}, respectively. 
\end{definition}

Intuitively, for the  model-free setting, $\ell^{(i),s}(f,\pi)$ is defined in \eqref{eq:modelfree discrepancy} and  represents the mean-squared Bellman error of the function $f \in \cF^{(i)}$ for estimating the agent $i$'s value function under   policy $\pi$, 
     where $\{s_h,a_h\}_{h \in [H]} \sim \pi^s$, serving as an inconsistency measure between $f$ and $\pi$ under the previous data. 
For the model-based setting, the definition of $\ell^{(i),s}(f,\pi)$ in  \eqref{def:modelbased discrepancy} represents the expected Hellinger distance between $f^k$ and the true model $f^*$. 
Note that the discrepancy between $ f^k ,\pi^k$ in \eqref{eq:MADC}  is summed over $s\in [k-1]$. Thus, in both the model-free and model-based settings, the training error can be viewed as the in-sample error of $f^k$ on the historical data collected before the $k$-th episode. 
Thus, for an MG with a finite MADC, the prediction error is small whenever the training error is small. Specifically, when the training error is $\cO(K^{\alpha })$ for some $\alpha \in (0,2)$, then by choosing a proper $\mu$, we know that the prediction error grows as $\cO(\sqrt{K^{\alpha  }\cdot d_{\mathrm{MADC}}}) = o(K)$. In other words, as $K$ increases, the average prediction error decays to zero. In single-agent RL, when we adopt an optimistic algorithm, 
the prediction error serves as an upper bound of the regret 
\citep{dann2021provably, zhong2022gec, jin2021bellman}. 
Therefore, by quantifying how the prediction error is related to the training error, the MADC can be used to characterize the hardness of exploration in MARL.

Compared to the decoupling coefficient and its variants for single-agent MDP \cite{dann2021provably, agarwal2022model, zhong2022gec}, MADC selects the policy $\pi^k$ in a different way.
In the single-agent setting, the policy $\pi^k$ is always selected as the greedy policy of $f^k$, hence $V_{1,f^k}^{\pi^k}(\rho)$ is equivalent to the optimal value function. On the contrary, in our definition, the policy $\pi^k$ is not necessarily the greedy policy of $f^k$. In fact, $\{\pi^k\}_{k \in [K]}$ can be any pure policy sequence that is unrelated to $\{f^k\}_{k \in [K]}$.

\begin{assumption}[Finite MADC]\label{assum:decoup coeff}
    We assume that the MADC of the general-sum  MG of interest is finite, denoted by $d_{\mathrm{MADC}}$. As we will show in Section \ref{sec:relation}, the class of MGs with low MADCs include a rich class of MG instances, including multi-agent counterparts of models with low  Bellman eluder dimensions \citep{jin2021bellman, jin2022power,huang2021towards},  bilinear classes \citep{du2021bilinear}, and models with low witness ranks \citep{sun2019model, huang2021towards}. 
\end{assumption}

\section{Algorithm and Results}
In this section, we first introduce a unified algorithmic framework called \textbf{M}ulti-\textbf{A}gent \textbf{M}aximize-to-\textbf{EX}plore (MAMEX). 
Then, we present the  regret and sample complexity upper bounds of MAMEX, showing that both the model-free and model-based versions of MAMEX  are sample-efficient for learning NE/CCE/CE under the general function approximation.

\subsection{Algorithm}
\begin{algorithm}[t]
    \begin{algorithmic}[1]
      \small  
	\caption{\textbf{M}ulti-\textbf{A}gent \textbf{M}aximize-to-\textbf{EX}plore (MAMEX)}
	\label{alg:multi-agent io}
	\STATE \textbf{Input:} Hypothesis class $\cF$, parameter $\eta>0$, and an equilibrium solving oracle $\mathsf{EQ}$.
        \FOR{$k = 1,2,\cdots,K$}
        \STATE Compute  $\overline{V}_i^k(\pi)$  defined in \eqref{def:regularized_opt} for all $\pi \in \Pi^{\mathrm{pur}}$ and all $i\in[n]$.
        \STATE Compute the NE/CCE/CE 
        of the normal-form game defined  on $\Pi^{\mathrm{pur}}$ with payoff functions $  \{\overline{V}_i^k(\pi)\}_{i=1}^n$: 
        $ \pi^k \leftarrow \textsf{EQ}(\overline{V}_1^k , \overline{V}_2^k, \cdots, \overline{V}_n^k ) $. \label{line:equilibrium oracle}
        \STATE   Sample a pure joint policy $\zeta^k \sim \pi^k$, and collect a  trajectory 
        $\{s_h^k, a_h^k \}_{h\in [ H]} $ following $\zeta^k$.

        \STATE Update $\{L^{(i),k}\}_{i=1}^n $ according to \eqref{eq:def of L} (model-free) or  \eqref{eq:model based L} (model-based). 
        \ENDFOR
 \end{algorithmic}
 \end{algorithm}
In this subsection, we provide the  MAMEX algorithm for multi-agent RL under the general function approximation, which extends the MEX algorithm \citep{liu2023objective} to general-sum MGs. 
Recall that the definitions of NE/CCE/CE of general-sum MGs coincide with those defined in the normal-form game with pure strategies being the pure policies in $\Pi^{\mathrm{pur}}$.
Thus, when we know the payoffs $\{ V^{(i), \pi} (\rho)\}_{i\in [n]}$ for all $\pi \in  \Pi^{\mathrm{pur}}$, 
we can directly compute the desired NE/CCE/CE given an equilibrium solving oracle for the normal-form game.
However, each $ V^{(i), \pi} (\rho)$  is unknown and has to be estimated from data via online learning. 
Thus, in a nutshell, MAMEX is an iterative algorithm that consists of the following two steps: 

(a) Policy evaluation: For each $k\in [K]$, construct an estimator $\overline{V}_i^k(\pi)$ of $ V^{(i), \pi} (\rho)$ for each pure policy $\pi\in  \Pi^{\mathrm{pur}}$ and the agent $i\in[n]$ in each episode based on the historical data collected in the previous $k-1$ episodes. Here, the policy evaluation subproblem can be solved in both the model-free and model-based fashion. 

(b) Equilibrium finding: Compute an equilibrium (NE/CCE/CE) for the normal-form game over the space of pure policies with the estimated payoff functions $\{\overline{V}_i^k(\pi)\}_{i=1}^n$. The joint policy returned by the equilibrium finding step is then executed in the next episode to generate a new trajectory.

By the algorithmic design, 
to strike a balance between exploration and exploitation, it is crucial to construct $\{\overline{V}_i^k(\pi)\}_{i=1}^n$ in such a way that promotes exploration. 
To this end, we solve a regularized optimization problem over the hypothesis class $\cF^{(i)}$ to obtain $\overline{V}_i^k(\pi)$, where the objective function balances exploration with exploitation. 
We introduce the details of MAMEX as follows. 

{\noindent \bf Policy Evaluation.} For each $k\in[K]$, before the $k$-th episode, 
we have collected $k-1$ trajectories $\tau^{1:k-1} = \cup_{t=1}^{k-1}\{s_1^t,a_1^t,r_1^t,\cdots,s_H^t,a_H^t,r_H^t\}$. 
For any $i\in [n]$, $\pi \in \Pi^{\mathrm{pur}}$ and $f \in \cF^{(i)}$\footnote{For ease of notation, under the model-based setting, we denote $\cF^{(i)} = \cF $ for all agent $i\in[n]$.}, we can define a data-dependent discrepancy function  $L^{(i),k-1}(f, \pi,\tau^{1:k-1})$.
Such a function measures the in-sample error of the hypothesis $f$ with respect a policy $\pi$, evaluated on the historical data $\tau^{1:k-1}$.
The specific form of such a function differs under the model-free and model-based settings. 
In particular, as we will show in \eqref{eq:def of L} and \eqref{eq:model based L} below, under the model-free setting, $L^{(i),k-1}(f, \pi,\tau^{1:k-1})$ is constructed based on the mean-squared Bellman error with respect to the Bellman operator $\cT_h^{(i),\pi}$ in \eqref{eq:define_bellman_operator}, 
while under the model-based setting, $L^{(i),k-1}(f, \pi,\tau^{1:k-1})$ is constructed based on the negative log-likelihood loss. 
Then, for each $\pi \in \Pi^{\mathrm{pur}}$ and $i\in[n]$, we define  $\overline{V}_i^k(\pi)$  as   
\begin{align}
    \overline{V}_i^k(\pi) =\sup_{f \in \cF^{(i)}}\Big\{ \hat{V}^{(i),\pi,k}(f)  := \underbrace{V_{f}^{(i),\pi}(\rho)}_{\displaystyle  \text{(a)} } \underbrace{-\eta \cdot L^{(i),k-1}(f, \pi,\tau^{1:k-1})}_{\displaystyle \text{(b)}} \Big\}  .\label{def:regularized_opt}
\end{align}

{\noindent \bf Equilibrium Finding.} Afterwards, the algorithm utilizes the equilibrium oracle $\textsf{EQ}$ (Line \ref{line:equilibrium oracle} of Algorithm \ref{alg:multi-agent io}) to compute an equilibrium (NE/CCE/CE) for the normal-form game over $\Pi^{\mathrm{pur}}$ with payoff functions $\{\overline{V}_i^k(\pi)\}_{i=1}^n$. 
The solution to the equilibrium oracle is a mixed strategy $\pi^k$, i.e., a probability distribution over $\Pi^{\mathrm{pur}}$.

Finally,  we sample a random pure policy $\zeta^k$ from $\pi^k$ and execute $\zeta^k$ in the $k$-th episode to generate a new trajectory. 
See Algorithm \ref{alg:multi-agent io} for the details of MAMEX. 
Here, we implicitly assume that $\Pi^{\mathrm{pur}}$ is finite for ease of presentation. For example, $\Pi^{\mathrm{pur}}$ is the set of all deterministic policies.
When $\Pi^{\mathrm{pur}}$ is infinite, we can replace $\Pi^{\mathrm{pur}}$ by a $1/K$-cover of $\Pi^{\mathrm{pur}}$ with respect to the distance   $d^{(i)}(\pi^{(i)},\tilde \pi^{(i)}) = \max_{s \in \cS} \| \pi^{(i)} (\cdot  \mid s) - \tilde \pi^{(i)}(\cdot \mid s)\|_1$.   

Furthermore,  the objective $\hat{V}^{(i),\pi,k}(f)$  in \eqref{def:regularized_opt}  is constructed by a sum of  (a) the value function $V_{f}^{(i),\pi}(\rho)$ of $\pi$ under the hypothesis $f$ and (b) a regularized term $-\eta \cdot L^{(i),k-1}(f, \pi,\tau^{1:k-1})$, and the   payoff function $\overline{V}_i^k(\pi)$ 
is obtained by solving  a maximization problem over $\cF^{(i)}$. 
The two terms (a) and (b)   represent the "exploration" and "exploitation" objectives, respectively, and the parameter $\eta > 0$ controls the trade-off between them. 
To see this, 
consider the case where we only have the term (b) in the objective function. 
In the model-based setting, \eqref{def:regularized_opt} reduces to the maximum likelihood estimation (MLE) of the model $f$ given the historical data $\tau^{1:k-1}$.
Then $\pi^k$ returned by Line \ref{line:equilibrium oracle} is the equilibrium policy computed from the MLE model. 
Thus, without term (a) in $\hat{V}^{(i),\pi,k}(f)$, the algorithm only performs exploitation. 
In addition to fitting the model, the term (a) also encourages the algorithm to find a model with a large value function under the given policy $\pi$, which promotes exploration. 
Under the model-free setting, only having term (b) reduces to least-squares policy evaluation (LSPE) \citep{sutton2018reinforcement},  and thus term (b) also performs exploitation only.

 \vspace{5pt}
{\noindent \bf Comparison with Single-Agent MEX \citep{liu2023objective}.} 
When reduced to the single-agent MDP, MAMEX 
can be further simplified to the single-agent MEX algorithm \citep{liu2023objective}.
In particular, when $n = 1$, equilibrium finding is reduced to maximizing the function defined in \eqref{def:regularized_opt} over single-agent policies, i.e., 
$\max_{\pi} \max_{f\in \cF} \hat V^{\pi, k}(f)$. 
By exchanging the order of the two maximizations, we obtain an optimization problem over the hypothesis class $\cF$, which recovers the single-agent MEX   \citep{liu2023objective}. 
In contrast, in general-sum MGs, the equilibrium policy 
can no longer be obtained by a single-objective optimization problem.   Hence,  it is unviable to directly extend MEX to optimize over hypothesis space in MARL. Instead, MAMEX solves an optimization over $\cF$ in the style of MEX  for each pure policy $\pi \in \Pi^{\mathrm{pur}}$, and then computes the NE/CCE/CE of the normal-form game over the space of pure policies. 

 
 
 \vspace{10pt}
{\noindent \bf Comparison with Existing MARL  Algorithms with Function Approximation} Previous RL algorithms for  MGs with the general function approximation usually require solving minimax optimization  \citep{chen2022unified,zhan2022decentralized,foster2023complexity} or constrained optimization subproblems within data-dependent constraints \citep{wang2023breaking}.
In comparison, 
the optimization subproblems of MEX are 
single-objective and do not have data-dependent constraints, and thus seem easier to implement. 
For example, in practice, the inner problem can be solved by a regularized version of TD learning \citep{liu2023objective}, and the outer equilibrium finding can be realized by any fast method to calculate equilibrium \citep{hart2000simple,  anagnostides2022faster}.

 In the following, we instantiate the empirical discrepancy function $L^{(i), k-1}$ for both the model-free setting and the model-based setting. 
 
\vspace{5pt}\noindent\textbf{Model-Free Algorithm}
 Under the model-free setting, we define the empirical discrepancy function $L$ as follows.
 For any $h \in [H]$ and $k \in [K]$, let  $\xi_h^k = \{s_h^k,a_h^k,s_{h+1}^k\}$.
 For any $i\in [n]$, $\pi \in \Pi^{\mathrm{pur}}$ and $f \in \cF^{(i)}$,  we define 
\begin{align}
    L^{(i),k-1}(f, \pi,\tau^{1:k-1}) &= \sum_{h=1}^H \sum_{j=1}^{k-1}\Big[\big (l_h^{(i)}(\xi_h^j,f,f,\pi)\big)^2-\inf_{f'_h \in \cF_h^{(i)}}  \big (l_h^{(i)}(\xi_h^j,f',f,\pi) \big)^2\Big]\label{eq:def of L},
\end{align}
where 
$
    l_h^{(i)}(\xi_h^j, f,g, \pi) = (f_h(s_h^j, a_h^j)-r_h^{(i)}(s_h^j, a_h^j)-\langle g_{h+1}(s_{h+1}^j,\cdot), \pi_{h+1}(\cdot \mid s_{h+1}^j)\rangle_{\cA})^2$ is the mean-squared Bellman error involving $f_h$ and $g_{h+1}$. 

As we will show in  Lemma \ref{lemma:concentration} in \S \ref{subsection:analysis}, using martingale concentration techniques, we can show that $L^{(i),k-1}(f,\pi,\tau^{1:k-1})$ servers as an upper bound of $\sum_{s=1}^{k-1} \ell^{(i),s}(f,\pi) $, where $\ell^{(i),s} $ is defined in \eqref{eq:modelfree discrepancy}.    Thus, the empirical discrepancy function $L^{(i), k-1}$  can be used to control the training error in the definition of MADC.

\vspace{5pt}\noindent\textbf{Model-Based Algorithm}
For the model-based setting, 
we define $L^{(i), k-1}$ as the negative log-likelihood: 
\begin{align}\label{eq:model based L}
    L^{(i),k-1}(f, \pi,\tau^{1:k-1}) = \sum_{h=1}^H \sum_{j=1}^{k-1} -\log \PP_{h,f}(s_{h+1}^j\mid s_h^j,a_h^j).
\end{align}
As we will show in Lemma \ref{lemma:concentration:model-based}, $L^{(i),k-1}$ can be used to control the training error in \eqref{eq:MADC}, 
where $\ell^{(i),s}$ is defined in \eqref{def:modelbased discrepancy}.

 \subsection{Theoretical Results}
 In this subsection, we present our main theoretical results and show that MAMEX (Algorithm \ref{alg:multi-agent io}) is sample-efficient for learning NE/CCE/CE in the context of  general function approximation.
 \begin{theorem}\label{thm:mainresult}
    Let the discrepancy function $\ell^{(i),s}$  in \eqref{eq:MADC} be defined in \eqref{eq:modelfree discrepancy} and \eqref{def:modelbased discrepancy} for model-free and model-based settings,  respectively.
Suppose Assumptions~\ref{assum: realizability} and  \ref{assum:decoup coeff} hold. By setting $K\ge 16$ and $\eta = 4/\sqrt{K}\le 1$, with probability at least $1-\delta$, the regret of Algorithm~\ref{alg:multi-agent io} after $K$ episodes is upper bounded by 
 \begin{align*}
     \mathrm{Reg}_{\mathrm{NE,CCE,CE}}(K)  & \le \widetilde{\cO}\Big(nH\sqrt{K}\Upsilon_{\cF,\delta}  + nd_{\mathrm{MADC}}\sqrt{K}+nd_{\mathrm{MADC}}H  \Big ) ,
\end{align*} 
where $\widetilde{\cO}(\cdot)$ hides absolute constants and polylogarithmic terms in $H$ and $K$,  and $\Upsilon_{\cF,\delta}$ is a term that quantifies the complexity of the hypothesis class $\cF$.
In particular, we have 
$ \Upsilon_{\cF,\delta} =  R^2\log(\max_{i \in [n]}\cN_{\cF^{(i)}}(1/K) \cdot | \Pi^{\mathrm{pur}} |/\delta) $ in the model-free setting and $\Upsilon_{\cF,\delta} = \log\left(\cB_{\cF}(1/K)/\delta\right)   $ in the model-based setting. 

 \end{theorem}
Theorem \ref{thm:mainresult} shows that our MAMEX achieves a sublinear  $\sqrt{K}$-regret for learning NE/CCE/CE, 
where the multiplicative factor depends polynomially on the number of agents $n$ and horizon $H$. 
Thus, MAMEX is sample-efficient in the context of the general function approximation. 
Moreover, the regret depends on the complexity of the hypothesis class via two quantifies -- the MADC $d_{\mathrm{MADC}} $, which captures the inherent challenge of exploring the dynamics of the MG, and the quantity $\Upsilon_{\cF,\delta}$, which characterizes the complexity of estimating the true hypothesis $f^*$ based on data. 
To be more specific, in the model-free setting, since we need to evaluate each pure policy, $\Upsilon_{\cF,\delta}$ contains $\log |\Pi^{\mathrm{pur}}|$ due to uniform concentration. 
When reduced to the tabular setting, we can choose $\Pi^{\mathrm{pur}}$ to be the set of deterministic policies, and both $\Upsilon_{\cF,\delta} $ and $d_{\mathrm{MADC}} $ are polynomials of $|\cS|$ and $|\cA|$. 
Furthermore, when specialized to tractable special cases with function approximation and some special pure policy class such as log-linear policy class \cite{cayci2021linear}, we show in \S \ref{sec:relation} that Theorem \ref{thm:bilinear class} yields regret upper bounds comparable  to existing works. 
Moreover, 
using the standard online-to-batch techniques, we can transform the regret bound into a  sample complexity result. 
Specifically, after running MEMAX for $K$ episodes, the random policy that outputs $\pi^k, k\in [K]$ uniformly random is an $\varepsilon$-approximate NE/CCE/CE, where $K$ depends on $\varepsilon$. 
\begin{corollary}\label{coroll:sample complexity}
    Under the same setting as in  Theorem \ref{thm:mainresult}, with probability at least $1-\delta$, when $K\ge \widetilde{\cO}\left(\left(n^2H^2 + n^2d_{\mathrm{MADC}}^2\Upsilon_{\cF,\delta}^2\right)\cdot \varepsilon^{-2}\right)$, if we output the mixture policy $\pi_{\mathrm{out}} = \mathrm{Unif}(\{     \pi^k \}_{k\in[K]})$, the output policy $\pi_{\mathrm{out}}$ is a $\varepsilon$-\{NE, CCE, CE\}.
\end{corollary}
\begin{proof}
    See \S\ref{sec: corollary} for the proof.
\end{proof}
Corollary \ref{coroll:sample complexity} shows that MAMEX 
is sample-efficient for learning all three equilibria of general-sum MGs under general 
function approximation.

\section{Relationships between MADC and Tractable RL Problems}\label{sec:relation}
In this section, we show that the class of MGs with finite MADCs contains a rich class of models. Thus, when applied to these 
 concrete  MARL models, Theorem \ref{thm:mainresult} shows that MAMEX learns NE/CCE/CE with provable sample efficiency. 

 
 In the sequel, we instantiate the discrepancy function $\ell^{(i),s}$ for both model-free and model-based MARL, and introduce some concrete general-sum MG models that satisfy   Assumption \ref{assum:decoup coeff}. 
\subsection{Model-Free MARL Problems}
In the model-free setting, for $\{\pi^k\}_{k \in [K]} \subseteq \Pi^{\mathrm{pur}}$ in \eqref{eq:MADC}, the discrepancy function $\ell^{(i),s}(f,\pi)$ for $\pi \in \Pi^{\text{pur}}$ is defined as
\vspace{-0.5em}
\begin{align}\label{eq:modelfree discrepancy}
        \ \ \ \ \ \ell^{(i),s}(f,\pi)
        = \sum_{h=1}^H \EE_{(s_h,a_h)\sim \pi_h^s }((f_h-\cT_h^{(i),\pi}(f_{h+1}))(s_h,a_h))^2, \qquad \forall f \in \cF^{(i)}, \forall s \in [K].
    \end{align}
That is,   $\ell^{(i),s}(f,\pi)$ measures agent $i$'s mean-squared Bellman error for evaluating $\pi$, when the trajectory is sampled by letting all agents follow policy $\zeta^s$. 

Now we provide function classes with small MADCs including multi-agent counterparts of models with low Bellman eluder dimensions \citep{jin2021bellman, huang2021towards} and Bilinear Classes \citep{du2021bilinear}. Then, we introduce some concrete examples in these members and show that the regret upper bound of MAMEX in Theorem \ref{thm:mainresult}, when specialized to these special cases, 
are comparable to existing works.

\vspace{5pt}\noindent\textbf{Multi-Agent Bellman Eluder Dimension}
Recently, \cite{jin2021bellman} introduce a model-free complexity measure called Bellman Eluder dimension (BE dimension) and show that function classes with low BE dimensions contain a wide range of RL problems such as linear MDP \citep{jin2020provably}, kernel MDP \citep{jin2021bellman} and function classes with low eluder dimension \citep{wang2020provably}.  
In this subsection, we extend the notion of  BE dimension to MARL. 
First, we introduce the definition of $\varepsilon$-independence between distributions and the concept of distribution eluder dimension.
\begin{definition}[$\varepsilon$-Independent Distributions]
Let $\cG$ be a function class on $\cX$, and $\upsilon, \mu_1,\cdots,\mu_n$ are probability distributions  over $\cX$. We called $\upsilon$ is $\varepsilon$-independent of $\{\mu_1\cdots \mu_n\}$ with respect to $\cG$ if there exists a function $g \in \cG$ such that $\sqrt{\sum_{i=1}^n (\EE_{\mu_i}[g])^2}\le\varepsilon$ and $|\EE_{\upsilon}[g]|>\varepsilon.$
\end{definition}

By this definition, if $\nu$ is $\varepsilon$-dependent of $\{\mu_1,\cdots,\mu_n\}$, 
whenever   we have    $\sqrt{\sum_{i=1}^n (\EE_{\mu_i}[g])^2}\le\varepsilon$ for   some  $g \in \cG$, we also  have  $|\EE_{\upsilon}[g]|\leq \varepsilon$. 

\begin{definition}[Distribution Eluder Dimension]\label{def:dist_be} 
    Let $\cG$ be a function class on $\cX$ and $\cD$ be a family of probability measures over $\cX$. The distributional eluder dimension $\mathrm{dim}_{\mathrm{DE}}(\cG,\cD, \varepsilon)$ is the length of the longest sequence ${\rho_1,\cdots, \rho_n} \subseteq \cD$ such that there exists $\varepsilon'\ge \varepsilon$ where $\rho_i$ is $\varepsilon'$-independent of $\{\rho_1,\cdots,\rho_{i-1}\}$ for all $i \in [n].$ 
\end{definition}

In other words, distributional eluder dimension $\mathrm{dim}_{\mathrm{DE}}(\cG,\cD, \varepsilon)$ is the length of the longest sequences of distributions in $\cD$ such that each element is $\varepsilon'$-independent of its predecessors with respect to $\cG$, from some  $\varepsilon' \geq \epsilon$.
Such a notion generalizes the standard eluder dimension \cite{russo2013eluder} to the distributional setting. When we set $\cD$ to be the set of Dirac measures $ \{ \delta_{x} (\cdot) \}_{x \in \cX}$,  the distributional eluder dimension $\mathrm{dim}_{\mathrm{DE}}(\cG -\cG,\cD, \varepsilon)$ reduces to the standard eluder dimension introduced in \cite{russo2013eluder}. Here, $\cG - \cG = \{ g_1 - g_2 \colon g_1, g_2 \in \cG\}$. 

For any agent $i$ and any pure policy $\pi \in \Pi^{\mathrm{pur}}$, we denote the function class of the Bellman residual as $\cF^{(i),\pi}_h = \{f_h-\cT^{(i),\pi}f_{h+1}\mid f \in \cF^{(i)}\}$.
Now we introduce the definition of the multi-agent BE dimension with respect to a class of distributions.

\begin{definition}[Multi-Agent Bellman Eluder Dimension]
    Let $D =  \{ D_h\}_{h \in [H]}$ be a  set of $H$ classes of distributions over $\cS \times \cA$, one for each step of an episode. 
    The multi-agent Bellman eluder (BE) dimension with respect to $D$ is defined as 
    \begin{align}\label{eq:define_mabe}
    \mathrm{dim}_{\mathrm{MABE}}(\cF, \cD ,   \varepsilon) =  \max_{h \in [H]}\max_{i \in [n]}~\Bigl \{ \mathrm{dim}_{\mathrm{DE}}\Big(\bigcup_{\pi \in \Pi^{\mathrm{pur}}}\cF_h^{(i),\pi}, \cD_h, \varepsilon\Big) \Big\} .\end{align}
\end{definition}

In other words, the multi-agent BE dimension is defined as the maximum of the distribution eluder dimensions with respect to $D_h$,  based on the agent-specific Bellman residue classes $\bigcup _{\pi \in \Pi^{\mathrm{pur}}} \cF_h^{(i), \pi}$. 
Compared with the  BE dimension for single-agent RL~\citep{jin2021bellman}, the multi-agent version takes the maximum over the agent index   $i \in [n]$, and the function class involves the union of the function class $\cF_h^{(i),\pi}$ for all $\pi \in \Pi^{\textrm{pur}}$. 
In comparison, leveraging the facts that the optimal policy is the greedy policy of the optimal value function and that the optimal value function is the fixed point of the Bellman optimality operator, it suffices to only consider residues of the Bellman optimality operator in the definition of single-agent BE dimension. 
In contrast, for general-sum MGs, 
finding the desired equilibrium policies is not a single-objective policy optimization problem, and the notion of the Bellman optimality operator is not well-defined. 
As a result, to extend the concept of Bellman eluder dimension to general-sum MGs, in the function class, we take into account $\cF_h^{(i), \pi}$ for all $\pi \in \Pi^{\mathrm{pur}}$, which correspond to evaluating the performance of all the pure policies. 
Besides, in \eqref{eq:define_mabe}, we also take the maximum over all agents $i\in[n]$ and all steps $h \in [H]$, which aligns with the definition of single-agent BE dimension. 


Furthermore, in the definition of multi-agent BE dimension, we need to specify a set of distributions $D = \{D_h\}_{h \in [H]}$ over $\cS \times \cA$. 
We consider two classes. 
First, let $D_\Delta = \{D_{\Delta,h}\}_{h \in [H]}$ denote a class of probability measures over $\cS \times \cA$ 
with
   $ D_{\Delta,h} = \{\delta_{(s,a)}(\cdot)\mid (s,a) \in \cS \times \cA\}$, 
which contains all the Dirac measures that put mass one to a state-action pair at step $h$.
Second, given the set of pure policies  $\Pi^{\mathrm{pur}}$, we let  $D_\Pi = \{D_{\Pi,h}\}_{h \in [H]}$ denote a class of probability measures induced $\Pi^{\mathrm{pur}}$ as follows. For any $\pi \in \Pi^{\mathrm{pur}}$, when all the agents follow $\pi$ on the true MG model, they generate a Markov chain $\{s_h, a_h\}_{h\in[H]}$ whose joint distribution is determined by $\pi$, denoted by $\PP^{\pi}$. Then, for any $h\in[H]$, we define 
$D_{\Pi,h} = \{\rho \in \Delta(\cS \times \cA)\mid \rho(\cdot) = \PP^{\pi}((s_h,a_h)=\cdot), \pi \in \Pi^{\mathrm{pur}}\}$, i.e., $D_{\Pi,h}$  
denotes the collection of all marginal distributions of $(s_h, a_h)$ induced by pure policies.


In the following, to simplify the notation, we denote 
\begin{align}\label{eq:final_mabe}
\mathrm{dim}_{\mathrm{MABE}}(\cF,  \varepsilon) = \min \bigl\{ \mathrm{dim}_{\mathrm{MABE}}(\cF,  \cD_{\Delta},\varepsilon) , ~ \mathrm{dim}_{\mathrm{MABE}}(\cF,  \cD_{\Pi},\varepsilon) \bigr \}.
\end{align}
The following theorem shows that, when $\cF$ satisfies realizability and completeness (Assumption \ref{assum: realizability}), 
for a general-sum MG with a finite multi-agent BE dimension given by \eqref{eq:final_mabe}, its multi-agent decoupling coefficient (Definition \ref{def:decoup_coeff}) is also bounded. 
In other words, Assumption \ref{assum:decoup coeff} holds for any general-sum MG model with a low multi-agent BE dimension. 
As a result, the class of MGs with finite multi-agent BE dimensions is a subclass of MGs with finite multi-agent decoupling coefficients.

\begin{theorem}[Low Multi-Agent BE Dimension $\subseteq $ Low MADC]\label{thm:multi BE in MADC}
Let $K$ be any integer and let   $\cF$ be a hypothesis class under the model-free setting, i.e., a class of $Q$-functions. 
Assume that $\cF$ satisfies the realizability and completeness condition specified in Assumption \ref{assum: realizability}. 
Suppose that $\cF$ has a finite multi-agent BE dimension $d = \mathrm{dim}_{\mathrm{MABE}}(\cF, 1/K)$, with the discrepancy function $\ell^{(i),s}$ given  in \eqref{eq:modelfree discrepancy}, the multi-agent decoupling coefficient of $\cF$ satisfies 
    $d_{\mathrm{MADC}} =  \cO (dH \log K)$, where $ \cO (\cdot)$ omits absolute constants. 
    \end{theorem}
    \begin{proof}
    See \S\ref{appendix:proof of BE} for detailed proof. 
    \end{proof}

Combining Theorem \ref{thm:mainresult} and Theorem  \ref{thm:multi BE in MADC},  we obtain that MAMEX  achieves a  sublinear $\tilde{\cO}(ndH\sqrt{K} + ndH^2 + nHR^2\sqrt{K}\log \Upsilon_{\cF,\delta})$ regret for function classes with a finite multi-agent BE dimension $d$. 
It remains to see  
 that function classes with low multi-agent BE dimensions contain a wide range of RL problems.
 To this end, 
   we prove that if the eluder dimension \citep{russo2013eluder}  of the function class $\cF_h^{(i)}$ is small for all $h \in [H]$ and $i \in [n]$, $\cF = \otimes_{i=1}^n (\otimes_{h=1}^H \cF_h^{(i)})$ has a low multi-agent BE dimension. 
   Function classes with finite eluder dimension contain linear, generalized linear, and kernel functions  \citep{russo2013eluder}, and thus contain a wide range of MG models. On these MG problems, the model-free version of MAMEX achieves sample efficiency provably.

\begin{theorem}\label{thm:low eluder dimension}
Suppose $\cF$ satisfies Assumption~\ref{assum: realizability}.
For any $i\in [n]$ and $h \in [H]$, let $\mathrm{dim}_{\mathrm{E}}(\cF_h^{(i)}, \varepsilon)$ denote the eluder dimension of $\cF_h^{(i)}$, which is a special case of the distributional eluder dimension introduced in Definition \ref{def:dist_be}. 
That is, $\mathrm{dim}_{\mathrm{E}}(\cF_h^{(i)}, \varepsilon)$ is equal to $\mathrm{dim}_{\mathrm{DE}}(\cF_h^{(i)} - \cF_h^{(i)}, D_{\Delta} ,  \varepsilon)$, where $\cF_h^{(i)} - \cF_h^{(i)} = \{ g  \colon  g =   f_1-f_2, f_1, f_2 \in \cF_h^{(i)}\}$ and $D_{\Delta}$ contains the class of Dirac measures on $\cS \times \cA$. 
Then, the multi-agent BE dimension defined in \eqref{eq:final_mabe} satisfy
    $$\mathrm{dim}_{\mathrm{MABE}}(\cF, \varepsilon) \le \max_{h \in [H]}\max_{i \in [n]}\mathrm{dim}_{\mathrm{E}}(\cF_h^{(i)}, \varepsilon).$$
\end{theorem}
\begin{proof}
See  \S \ref{appendix:low eluder dimension} for detailed proof. 
\end{proof} 


\vspace{5pt}\noindent\textbf{Multi-Agent Bilinear Classes}
Bilinear Classes \citep{du2021bilinear} 
consists of    MDP models where the Bellman error admits a bilinear structure.
On these models, \cite{du2021bilinear}  propose online RL algorithms that are provably sample-efficient. Thus, the Bilinear Classes is a family of tractable MDP models with the general function approximation. 
In the sequel,  we extend the Bilinear Classes to general-sum  MGs and show that such an extension covers some notable special cases studied in the existing works.  
 Then, we prove that multi-agent Bilinear Classes have a small MADC, thus satisfying the Assumption \ref{assum:decoup coeff}.  
Therefore, when applied to these problems,    MAMEX provably achieves sample efficiency. 


\begin{definition}[Multi-Agent Bilinear Classes]\label{def:bilinear}
Let $\cV$ be a Hilbert space and let $\langle \cdot, \cdot \rangle _{\cV}$ and $\| \cdot \|_{\cV}$ denote the inner product and norm on $\cV$.  
    Given a multi-agent general-sum MG with a hypothesis class $\cF $ satisfying Assumption \ref{assum: realizability}, it belongs to multi-agent Bilinear Classes if there exist $H$ functions $\{W_h^{(i)}: \cF^{(i)}\times \Pi^{\mathrm{pur}} \mapsto \cV\}_{h=1}^H $ for each agent $i \in [n]$ and $\{X_h: \Pi^{\mathrm{pur}}\mapsto \cV\}_{h=1}^H $  
    such that the Bellman error of each agent $i$ can be factorized using $W_h^{(i)}$ and $X_h$. 
    That is, for each $i\in [n]$, 
    $   f \in \cF^{(i)}, h \in [H], \pi ,\pi' \in \Pi^{\mathrm{pur}},$  we have 
    \begin{align}
        \Big|\EE_{(s_h,a_h)\sim\pi'}\big[&f_h(s_h,a_h)-r_h^{(i)}(s_h,a_h)-\EE_{s'\sim \PP_h(s'\mid s_h,a_h)}\langle f_{h+1}(s',\cdot),\pi_{h+1}(\cdot \mid s')\rangle_{\cA}\big]\Big| \nonumber\\&= \Big|\big\langle W_h^{(i)}(f,\pi) - W_h^{(i)}(f^{(i),\mu^{(i),\pi}},\mu^{(i),\pi}), X_h(\pi')\big\rangle_{\cV}\Big|, \label{eq: ma bilinear class}
    \end{align}
    where  $ \mu^{(i),\pi} = (\pi^{(i),\dag}, \pi^{(-i)})$ is the best response for the agent $i$ given that the  other agents all follow $\pi$.
Here, the function $f^{(i), \mu^{(i),\pi}}$ 
 is the fixed point of $\cT^{(i),\mu^{(i),\pi}}$, i.e.,  \begin{align}\label{eq:some_fix_point}
 f^{(i),\mu^{(i),\pi}}_h= \cT^{(i),\mu^{(i),\pi}}f_{h+1}^{\mu^{(i),\pi}}.\end{align}
Moreover, we require that $\{ W_h^{(i)}, X_h \}_{h\in [H]}$ satisfy a regularity condition
\begin{align}\label{eq:regularity_bilinear}
\sup_{\pi \in \Pi^{\text{pur}}, h \in [H]}\Vert X_h(\pi)\Vert_\cV\le 1, \qquad \sup_{i \in [n], f \in \cF^{(i)}, \pi \in\Pi^{\text{pur}}, h \in [H]}\Vert W_h^{(i)}(f,\pi)\Vert_\cV\le B_W, 
\end{align}
where $B_W$ is a constant. 
\end{definition}

 In this definition, for any $\pi \in \Pi^{\mathrm{pur}}$ and $f\in \cF^{(i)}$, 
 $$f_h(s_h,a_h)-r_h^{(i)}(s_h,a_h)-\EE_{s'\sim \PP_h(s'\mid s_h,a_h)}\langle f_{h+1}(s',\cdot),\pi_{h+1}(\cdot \mid s')\rangle_{\cA}$$  is the Bellman error of $f$ at $(s_h, a_h)$ for evaluating  policy $\pi$ on behalf of agent $i$.  On the left-hand side of \eqref{eq: ma bilinear class}, we evaluate such a Bellman error with respect to the distribution induced by another policy $\pi'$. Equation \eqref{eq: ma bilinear class} shows that this error can be factorized into the inner product between $W_h^{(i)}$ and $X_h^{(i)}$, where both $W_h^{(i)}$ only involves $(f, \pi)$ while $X_h^{(i)}$ only involves $\pi'$.  Thus, the multi-agent Bilinear Classes specifies a family of Markov games whose Bellman error satisfies a factorization property. 
 Furthermore, recall that the best response  $ \pi^{(i),\dag} = \max_{\nu\in \Delta(\Pi_i ^{\mathrm{pur}})} V^{\nu, \pi^{(-i)}}$ is attained at some pure policies, we have 
 $\mu^{(i),\pi} \in \Pi^{\mathrm{pur}}$. 
Under  Assumption \ref{assum: realizability}, the  fixed point $f^{(i),\mu^{(i),\pi}}$ in \eqref{eq:some_fix_point} is guaranteed to exist and belongs to $\cF$.

We define $\cX_h = \{X_h(\pi) : f \in \cF, \pi \in \Pi^{\text{pur}}\}$ and $\cX = \bigcup_{h=1}^H \cX_h.$ The complexity of the multi-agent Bilinear Classes essentially is determined by the complexity of the Hilbert space $\cV$. 
To allow $\cV$ to be infinite-dimensional, we introduce the notion of information gain, which characterizes the intrinsic complexity of $\cV$ in terms of exploration. 
 
\begin{definition}[Information Gain]
    Suppose $\cV$ is a Hilbert space and  $\cX\subseteq \cV$. For $\varepsilon>0$ and integer $K>0$, the information gain $\gamma_K(\varepsilon,\cX)$ is defined by 
    \begin{align*}
        \gamma_K(\varepsilon,\cX) = \max_{x_1,\cdots,x_K \in \cX} \log \det\left(I + \frac{1}{\varepsilon}\sum_{k=1}^K x_kx_k^\top\right).
    \end{align*}
\end{definition}
The following theorem shows that multi-agent Bilinear Classes with small information gain have low MADCs.  
\begin{theorem}[Multi-Agent Bilinear Classes $\subseteq$ Low MADC]\label{thm:bilinear class}
    For a general-sum MG in the multi-agent bilinear class with a hypothesis class $\cF$, let   $\gamma_K(\varepsilon,\cX) = \sum_{h=1}^H \gamma_K(\varepsilon,\cX_h)$ be the information gain.
 Then,  Assumption~\ref{assum:decoup coeff} holds with the discrepancy function $\ell^{(i),s}$ given in  \eqref{eq:modelfree discrepancy}. In particular, we have  $$d_{\mathrm{MADC}} \le \max\big \{1,8 R^2 \cdot \gamma_K(1/(KB_W^2),\cX) \big \}, $$
 where $B_{W}$ is given in \eqref{eq:regularity_bilinear} and $R \in (0, H]$ is an upper bound on $\sum_{h=1}^H r_h $.  
 \end{theorem}
\begin{proof}
See \S\ref{appendix: bilinear class} for a detailed proof.
\end{proof}

Now we introduce some concrete members of the multi-agent Bilinear Classes, which are general-sum MGs with the linear function approximation. 
In the single-agent RL setting, linear Bellman complete MDPs \citep{wang2019optimism} assume that the MDP model satisfies the Bellman completeness condition with respect to linear $Q$-functions.  We can extend such a model to general-sum MGs.
\begin{example}[Linear Bellman Complete MGs]
    We say a Markov game is a \textit{linear Bellman complete MG} of dimension $d$, if for any step $h \in [H]$  there exists a known feature $\phi_h: \cS \times \cA \mapsto \RR^d$ with   $\Vert \phi_h(s,a)\Vert \le 1$ for all $(s,a)\in \cS\times \cA$ such that 
    Assumption \ref{assum: realizability} holds for linear functions of $\phi_h$. 
    In other words, the Markov game satisfies Assumption~\ref{assum: realizability}  with 
    $\cF_h^{(i)} \subseteq \{\phi_h^\top \theta\mid \theta \in \RR^d, \|\theta \|_2 \le \sqrt{d_\theta}\}$ for  all $i \in [n]$ and $h\in [H]$, where  $d_\theta > 0$ is a parameter. 
    \end{example}
    
It is easy to see that   Linear Bellman complete MGs belong to multi-agent Bilinear Classes by choosing $$X_h(\pi) = \EE_{\pi}[\phi(s_h,a_h)] \in \RR^d, \qquad W_h^{(i)}(f,\pi) = \theta_{f,h}- w_{f,h}^{(i)}, $$ where  $\theta_{f,h}$ satisfies that $f(s_h,a_h)=\theta_{f,h}^\top \phi_h(s_h,a_h),$ and $w_{f,h}^{(i)}$ satisfies that\footnote{If there are multiple $\theta$ satisfying the requirement, we can break the tie arbitrarily. } \begin{align*}(w_{f,h}^{(i)})^\top  \phi_h(s_h,a_h) &= r_h^{(i)}(s_h,a_h)+ \EE_{s'\sim \PP_h(\cdot \mid s_h,a_h)}\langle f_{h+1}(s',\cdot), \pi_{h+1}(\cdot\mid s')\rangle_\cA \\&= \cT^{(i),\pi}(f_{h+1})\in \cF_h^{(i)}.\end{align*} 
Then, we have  $\cX_h  \subseteq \cV = \{\phi\in \RR^d : \|\phi\|_2 \le 1\}$ for all $h \in [H]$ and $B_W = 2\sqrt{d}$. It can be shown that the  logarithm of $1/K$-covering number of $\cF$ is $\log(\cN_\cF(1/K)) = \widetilde{\cO}(d)$, and the information gain can  bounded by \begin{align*}
    \gamma_K(1/B_W^2K,\cX) = 
\sum_{h=1}^H \gamma_K(1/B_W^2K,\cX_h) \le \sum_{h=1}^H \gamma_K(1/4dK, \cX_h)=  \widetilde{\cO}(Hd),
\end{align*}
where $\widetilde{\cO}$ omits absolute constants and logarithmic factors 
 \citep{du2021bilinear,  wang2020reinforcement}. 
 Thus, by Theorem \ref{thm:mainresult}, MAMEX achieves a   $\widetilde{\cO}(ndHR^2\sqrt{K}+nHR^2\sqrt{K}\log|\Pi^{\mathrm{pur}}| + ndH^2)$ regret. 
For the single-agent setting, comparing to the state-of-the-art $\widetilde{\cO}(dH\sqrt{K})$ regret when $R = 1$ \citep{zanette2020learning, chen2022abc}, our result matches their results in terms of $d,H$ and $K$ with an extra factor $|\Pi^{\mathrm{pur}}|$ in the logarithmic term.
Note that when the pure policy set of the agent $i$ is selected as some particular policy classes such as log-linear policy $$\Pi_i^{\text{pur}}=\{\pi_\vartheta:\pi_\vartheta(\cdot\mid s) = \text{Softmax}(\vartheta^\top  \psi(s,\cdot)), \|\vartheta\|_2 \le 1, \|\psi(\cdot,\cdot)\|\le 1, \vartheta\in \RR^{d_\pi}\},$$ we can select a cover by $$\hat{\Theta}=\{\hat{\vartheta}: \hat{\vartheta}_i = \lfloor \vartheta_i/\varepsilon\rfloor \times \varepsilon, \|\vartheta\|_2 \le 1, \vartheta \in \RR^{d_\pi}\}.$$ 
   \cite{zanette2021provable} prove that the logarithm of cardinality of the induced covering $\{\pi_\vartheta:\vartheta \in \hat{\Theta}\}$ is bounded by  $\widetilde{\cO}(nd_\pi)$, and then MAMEX provides a $\widetilde{\cO}((nd+n^2d_\pi)HR^2\sqrt{K} + ndH^2)$ regret. 

In particular, as one of the examples of Linear Bellman Complete MGs,  \cite{xie2020learning} consider a similar linear structure for two-player zero-sum games. 
\begin{example}[Zero-Sum Linear MGs~\citep{xie2020learning}]\label{example: zero-sum linear mg}
    In a zero-sum linear MG, for each $(s,a,b) \in \cS \times \cA \times \cB$ and $h \in [H]$, we have reward $r_h(s,a,b) \in [0,1]$. Also, there are $H$ known vectors $\theta_h \in \RR^d$, a known feature map $\phi: \cS \times \cA\times \cB\to \RR^d$, and a vector of $d$ unknown measures $\mu_h = \{\mu_{h,d'}\}_{d' \in [d]}$ on $\cS$, such that $\Vert \phi(\cdot,\cdot,\cdot)\Vert_2 \le 1 $, $\Vert \theta_h \Vert_2 \le \sqrt{d}$, and $\Vert \mu_h(\cS)\Vert_2 \le \sqrt{d}$. Moreover, they satisfy that  
    \begin{align*}
    r_h(s,a,b)=\phi(s,a,b)^\top \theta_h,\ \ \PP_h(\cdot \mid s,a,b) = \phi(s,a,b)^\top \mu_h(\cdot).
    \end{align*}
\end{example}

\vspace{-0.3em}
Zero-sum linear MG is a special case of linear Bellman complete MG with two players and $d_\theta = 2H\sqrt{d}$, and our algorithm provides a $\widetilde{\cO}(dH^3\sqrt{K}+H^3\sqrt{K}\log(|\Pi^{\mathrm{pur}}|))$ regret by choosing $R = H$ and the fact that $\log \cN_\cF^{(i)}(1/K)=\widetilde{\cO}(d)$. 
 The previous works provide a $\tilde{\cO}(d^{3/2}H^{2}\sqrt{K})$ sublinear regret \citep{xie2020learning} and a $\Omega(dH^{3/2}\sqrt{K})$    information-theoretic lower bound \citep{chen2022almost} for zero-sum linear MGs. Thus, our regret matches the lower bound in terms of $d$. The regret also has a higher order in $H$ compared to \cite{xie2020learning} and an extra factor $\log|\Pi^{\mathrm{pur}}| $.  Again, we can adopt the class of log-linear policies with a policy cover, which leads to $\log|\Pi^{\mathrm{pur}}| = \widetilde{\cO}(d_\pi)$. Thus, MAMEX yields a $\widetilde{\cO}((d+d_\pi) H^3 \sqrt{K})$ regret. 


\subsection{Model-Based RL Problems}
For model-based RL problems,   we choose the discrepancy function $\ell^{(i),s}$ in Assumption \ref{assum:decoup coeff} as the squared Hellinger distance
\vspace{-0.3em}
\begin{align}
    \ell^{(i),s}(f^k,\pi^k)  = \sum_{h=1}^H \EE_{(s_h,a_h)\sim \pi_h^s} D_{\mathrm{H}}^2 \big (\PP_{h,f^k}(\cdot \mid s_h,a_h) \Vert \PP_{h,f^*}(\cdot \mid s_h,a_h) \big),\label{def:modelbased discrepancy}
\end{align}
where $D_{\mathrm{H}}$ denotes the Hellinger distance, and $\EE_{(s_h,a_h)\sim \pi_h^s}$ means that the expectation is taken with respect to the randomness of the trajectory induced by $\pi^s$ on the true model $f^*$. Intuitively, it represents the expected in-sample distance of the model $f^k$ and the true model $f^*$. 

\cite{sun2019model} provide a complexity measure --- \textit{witness rank} --- to characterize the exploration hardness of the model-based RL problems.  
In the following, we extend the notion of the witness rank to MARL. 
\begin{example}[Multi-Agent Witness Rank]\label{example:ma witness rank}
    Let $\cV = \{\cV_h:\cS \times \cA \times \cS \mapsto [0,1]\}_{h \in [H]}$ denote a class of discriminators and let $\cF$ be a hypothesis class such that the true model, denoted by $f^*$, belongs to $\cF$.
    We say a  multi-agent witness rank of a general-sum MG is at most   $d$, if for any model $f \in \cF $ and any policy $\pi \in \Pi^{\mathrm{pur}}$, there exist  mappings $\{X_h:\Pi^{\mathrm{pur}}\to \RR^d\}_{h=1}^H $ and $\{W_h: \cF \to \RR^d\}_{h=1}^H $ such that \begin{align}
        \max_{v \in \cV_h}\EE_{(s_h,a_h)\sim \pi}[(\EE_{s'\sim \PP_{h,f}(\cdot \mid s_h,a_h)}-\EE_{s'\sim \PP_{h,f^*}(\cdot \mid s_h,a_h)})v(s_h,a_h,s')]&\ge \langle W_h^{(i)}(f), X_h(\pi)\rangle, \label{condition: witness rank 1}\\
        \kappa_{\mathrm{wit}}\cdot \EE_{(s_h,a_h)\sim \pi}[(\EE_{s'\sim \PP_{h,f}(\cdot \mid s_h,a_h)}-\EE_{s'\sim \PP_{h,f^*}(\cdot \mid s_h,a_h)})V^{(i),\pi}_{h+1,f}(s')]&\le \langle W_h^{(i)}(f), X_h(\pi)\rangle \label{condition: witness rank 2}
    \end{align}
    for all $h \in [H]$, where $    \kappa_{\mathrm{wit}}$ is a parameter. 
    Here, $V^{(i),\pi}_{h+1,f}$ is the value function of $\pi$ associated with agent $i$ under model $f$. Moreover, these mappings satisfy the following regularity condition:
   $$ \sup_{h \in [H], \pi \in \Pi^{\mathrm{pur}}}\|X_h(\pi)\| \le 1, \qquad \sup_{h \in [H], f \in \cF, i\in [n]}\|W_h^{(i)}(f)\| \le B_W.$$ 
\end{example}
Compared with the single-agent witness rank \citep{sun2019model}, the policy $\pi$ in the mapping $X_h(\pi)$ and the expectation $\EE_{(s_h,a_h)\sim \pi}$ in \eqref{condition: witness rank 1} and \eqref{condition: witness rank 2} can be an arbitrary pure policy instead of the optimal policy $\pi_f$ of the model $f$.  This stricter assumption is essential for general-sum MGs because we are interested in various equilibrium notions and each equilibrium can be non-unique.  
The following theorem shows that model classes with small multi-agent witness ranks have small MADCs. 

\begin{theorem}[Multi-Agent Witness Rank $\subseteq$ Low MADC]\label{thm:witness rank}
Let $\cF$ be a class of general-sum MGs whose multi-agent witness rank is no more than $d$.
Then, for any $f^*\in \cF$, we have $d_{\mathrm{MADC}} = \widetilde{\cO}(Hd/\kappa^2_{\mathrm{wit}}),$ where $d_{\mathrm{MADC}}$ is the multi-agent decoupling coefficient of $f^*$.
    \end{theorem}
 \begin{proof}
See \S \ref{appendix:proof of witnessrank} for detailed proof.
 \end{proof}
 This theorem shows that the multi-agent decoupling coefficient is upper bounded by the multi-agent witness rank, which shows that the class of MG models with a finite multi-agent decoupling coefficient contains models with a finite multi-agent witness rank. Hence, many concrete MG models such as the multi-agent version of factor MDP and linear kernel MDP all have finite multi-agent decoupling coefficients.
Therefore, applying  Theorem~\ref{thm:mainresult} to models with a finite Multi-Agent witness rank, the model-based version of MAMEX achieves a  $\widetilde{\cO}(nHd\sqrt{K}/\kappa_{\mathrm{wit}}^2 + nH \sqrt{K})$ regret with witness rank $d$. Note that for the model-based RL problems, our regret does not have the term $\log(|\Pi^{\mathrm{pur}}|)$, because the discrepancy function $\ell^{(i),s}$ in  \ref{def:modelbased discrepancy} is independent with $\pi^k.$
When applying our results to the single-agent setting, Theorem \ref{thm:witness rank} provides a similar regret result as in previous works \citep{sun2019model,zhong2022gec}. 

Another example of model-based RL problems is the linear mixture MGs \citep{chen2022almost}, which assumes that the transition kernel $\PP(s'\mid s,a)$ is a linear combination of $d$ feature mappings $\{\phi_i(s', s,a)\}_{i \in [d]}$, i.e. $\PP(s'\mid s,a) = \sum_{i=1}^d \theta_i \phi_i(s',s,a),$ where $a$ is a joint action.
\begin{example}[Multi-Agent Linear Mixture MGs]
    We call one general-sum MG is a \textit{linear mixture MG} with dimension $d$,  if there exist $h$ vectors $\{\theta_{h} \in \RR^d\}_{h \in [H]}$ and a known feature $\phi(s'\mid s,a) \in \RR^d$, such that $\Vert \theta_h\Vert_2 \le \sqrt{d}$ and $\PP_h(s'\mid s,a) = \langle \theta_h, \phi(s'\mid s,a)\rangle$ for any state-action pair $(s',s,a) \in \cS \times \cS \times \cA$. 
\end{example}

The following theorem shows that a linear mixture general-sum MG has a finite multi-agent decoupling coefficient.   Thus, MAMEX can be readily applied to these models with sample efficiency. 
 
\begin{theorem}[Multi-Agent Linear Mixture MGs $\subset$ Low MADC]\label{thm:linear mixture}
    For a linear mixture MG with dimension $d$, we have $d_{\mathrm{MADC}} = \widetilde{\cO}(dHR^4)$, where $R$ is an upper bound on $\sum_{h=1}^H r_h$.
\end{theorem}
\begin{proof}
See \S \ref{appendix:proof of linear mixture} for a detailed proof. 
\end{proof}
 \cite{chen2022almost} provides a minimax-optimal $\widetilde{\cO}(dH\sqrt{K})$ regret for two-player zero-sum MGs for $r_h \in [ 0,1]$.
Now choose $\cF_h = \{\theta_h \in \RR^d\}$. 
Combining with Theorem \ref{thm:linear mixture} and Theorem \ref{thm:mainresult}, and the fact that $\log(\cB_\cF(1/K)) =\widetilde{\cO}(Hd)$ \citep{liu2022partially}, MAMEX achieves a $\widetilde{\cO}(ndH^5\sqrt{K} + ndH^4)$ regret, where we set $R = H$. 
Compared with their regret upper bound, when applying our result to two-player zero-sum MGs by choosing $n=2$, the leading term of our regret $\widetilde{\cO}(dH^5\sqrt{K})$ matches the minimax-optimal result in terms of $d$ and $K$ but with an extra multiplicative factor $H^4$. 






\section{Conclusion}
In this paper, we study multi-player general-sum MGs under the general function approximation. 
We propose a unified algorithmic framework MAMEX for both model-free and model-based RL problems in the context of the general function approximation. 
Compared with previous works that either solve constrained optimization subproblems within data-dependent sub-level sets \citep{wang2023breaking}, or complex multi-objective minimax optimization subproblems \citep{chen2022unified,foster2023complexity}, the implementation of MAMEX requires only an oracle for solving a single-objective unconstrained optimization problem with an equilibrium oracle of a normal-form game, thus being more amenable to empirical implementation. 
To establish the theoretical guarantees for MAMEX, we introduce a novel complexity measure MADC to capture the exploration-exploitation tradeoff for general-sum MGs.  
We prove that MAMEX is provably sample-efficient in learning NE/CCE/CE on RL problems with small MADCs, which covers a rich class of MG models. When specialized to the special examples with small MADCs, the regret of MAMEX is comparable to existing algorithms that are designed for specific MG subclasses.


\OnlyInFull{
\newpage
\appendix
\centerline{\begin{Large} \textbf{Appendix}\end{Large}}

\section{Proof of Main Results}\label{appendix:analysis}
 \subsection{Proof of Model-Free Version of Theorem~\ref{thm:mainresult}}
 \label{subsection:analysis}

 \begin{proof}We first consider learning Nash equilibrium  and coarse correlated equilibrium.
\paragraph{NE/CCE} 
 First, by Assumption~\ref{assum: realizability}, for any pure joint policy $\upsilon$,  there exists a function $f^{(i),\upsilon} \in \cF^{(i)}$ satisfies that it has no Bellman error with Bellman operator $\cT^{(i),\upsilon}$ for any pure joint policy $\upsilon$,  i.e. \begin{align}\cT_h^{(i),\upsilon}f_{h+1}^{(i),\upsilon} = f_h^{(i),\upsilon}.\label{def:f i pi}\end{align} 
 Hence, $\{f_h^{(i),\upsilon}\}_{h \in [H]}$ is the $Q$-function of the agent $i$ when all agents follow the policy $\upsilon.$ Thus, we have 
 \begin{align}
     V^{(i),\upsilon}_{f^{(i),\upsilon}}(\rho) = \EE_{s_1\sim \rho, a\sim \upsilon(s_1)}[f_1^{(i),\upsilon
     }(s,a)]=\EE_{s_1\sim \rho, a\sim \upsilon(s_1)}[Q_1^{(i),\upsilon
     }(s,a)]=V^{(i),\upsilon}(\rho).
 \end{align}
 Also,
 denote $\hat{f}^{(i),\upsilon} = \arg \sup_{f \in \cF^{(i)}}\hat{V}_i^{\upsilon}(f)$ as the optimal function with respect to the regularized value $\hat{V}^{(i),\pi}(f)$ for the pure joint policy $\pi$ and agent $i$. 
 Now we have 
 \begin{align}
    &  \EE_{\upsilon\sim \pi^k}\Big[V^{(i),\upsilon}_{\hat{f}^{(i),\upsilon}}(\rho)-\eta L^{(i),k-1}(\hat{f}^{(i),\upsilon},\upsilon,\tau^{1:k-1})\Big]  
      = \EE_{\upsilon \sim \pi^k}\Big[\sup_{f \in \cF^{(i)}}\hat{V}^{(i),\upsilon}(f)\Big]\nonumber\\
     &\qquad \ge \max_{\upsilon^{(i)} \in \Pi^{\mathrm{pur}}}\EE_{\upsilon \sim \upsilon^{(i)} \times \pi^{(-i),k}}\Big[\sup_{f \in \cF^{(i)}}\hat{V}^{(i),\upsilon}(f)\Big].\label{eq: connect 1}
\end{align}
The inequality holds because of the property of  Nash Equilibrium or Coarse Correlated Equilibrium. Then, since the best response $\pi^{(i),k,\dag}$ is a pure policy, we have 
\begin{align}
&\max_{\upsilon^{(i)} \in \Pi^{\mathrm{pur}}}\EE_{\upsilon \sim \upsilon^{(i)} \times \pi^{(-i),k}}\Big[\sup_{f \in \cF^{(i)}}\hat{V}^{(i),\upsilon}(f)\Big]\nonumber
     \\&\quad \ge \EE_{\upsilon\sim \pi^{(i),k,\dag}\times\pi^{(-i),k}}\Big[\sup_{f \in \cF^{(i)}}\hat{V}^{(i),\upsilon}(f)\Big] = \EE_{\upsilon\sim \pi^{(i),k,\dag}\times\pi^{(-i),k}}\Big[\hat{V}^{(i),\upsilon}(f^{(i),\upsilon})\Big]\nonumber\\
     &\quad \ge \EE_{\upsilon\sim \mu^{(i),\pi^k}}\Big[V^{(i),\upsilon}_{f^{(i),\upsilon}}(\rho)- \eta L^{(i),k-1}(f^{(i),\upsilon},\upsilon,\tau^{1:k-1})\Big], \label{eq: connect 2}
 \end{align}
where $\upsilon \in \Pi_i^{\text{pur}}$,  $\mu^{(i),\pi^k} = (\pi^{(i),k, \dag}, \pi^{(-i),k})$ and $\pi^{(i),k,\dag}$ is the best response given the action of other agents $\pi^{(-i),k}.$
 Thus, combining \eqref{eq: connect 1} and \eqref{eq: connect 2}, we can derive 
 \begin{align}
     &  \EE_{\upsilon\sim \mu^{(i),\pi^k}}\Big[V^{(i),\upsilon}_{f^{(i),\upsilon}}(\rho)\Big]-\EE_{\upsilon\sim \pi^k}\Big[V^{(i),\upsilon}_{\hat{f}^{(i),\upsilon}}(\rho)\Big] \nonumber\\
     & \quad \le \eta \EE_{\upsilon\sim \mu^{(i),\pi^k}} \left[L^{(i),k-1}(f^{(i),\upsilon},\upsilon,\tau^{1:k-1})\right] - \eta \EE_{\upsilon\sim \pi^k}\left[L^{(i),k-1}(\hat{f}^{(i),\upsilon},\upsilon,\tau^{1:k-1})\right].\label{ineq:dif}
 \end{align}

Now we provide the concentration lemma, which shows that the empirical discrepancy function $L^{(i),k}(f, \pi,\tau^{1:k})$ is an estimate of the true discrepancy function $\sum_{s=0}^{k-1} \ell^{(i),s}(f,\pi)$.

\begin{lemma}[Concentration Lemma]\label{lemma:concentration} For any $k \in [K]$ pure joint policy $\pi$, and  $\{\zeta^{s}\}_{s=1}^{k-1} \in \Pi$ that be executed in Algorithm~\ref{alg:multi-agent io} in the first $k-1$ episodes, with probability at least $1-\delta$,
    \begin{align*}
        L^{(i),k-1}(f,\pi,\tau^{1:k-1}) -\frac{1}{4}\left(\sum_{s=0}^{k-1}\ell^{(i),s}(f,\pi)\right) \ge -\varepsilon_{\mathrm{conc}},
    \end{align*}
where $\varepsilon_{\mathrm{conc}} =\max\{\cO(HR^2 \log(HK\max_{i \in [n]}\cN_{\cF^{(i)}}(1/K) |\Pi^{\mathrm{pur}}|/\delta)), H\}$ and 
    \begin{align*}
        \ell^{(i),s}(f,\pi) = \sum_{h=1}^H \EE_{(s_h,a_h)\sim \zeta_h^s}\Big[((f_h-\cT_h^{(i),\pi}f_{h+1})(s_h,a_h))^2\Big].
    \end{align*}
\end{lemma}

\begin{proof}
   See \S\ref{sec: proof of lemma concentration} for a detailed proof.
   \end{proof}

In other words, if we define the event as 
\begin{align*}
    \cE_1 = \left\{L^{(i),k}(f,\pi,\tau^{1:k}) -\frac{1}{4}\left(\sum_{s=0}^{k-1}\ell^{(i),s}(f,\pi)\right) \ge \varepsilon_{\mathrm{conc}}, \forall f \in \cF^{(i)}, \pi \in \Pi^{\text{pur}}, k \in [K]
    \right\},
\end{align*}
we have  $\Pr\{\cE_1\} \ge 1-\delta.$
 Note that the $\varepsilon_{\mathrm{conc}}$ contains  $\log(|\Pi^{\mathrm{pur}}|/\delta)$ in the logarithmic term, which arises from our policy-search style algorithm.
\begin{lemma}[Optimal Concentration Lemma]\label{lemma:optimal concentration}
    For all index $i \in [n]$, all $\pi \in \Pi^{\mathrm{pur}}$ and function $f^{(i),\pi} \in \cF^{(i)}$ such that $\cT^{(i),\pi}f^{(i),\pi} = f^{(i),\pi}$, with probability at least $1-\delta$, we have 
    \begin{align*}
        L^{(i),k}(f^{(i),\pi},\pi,\tau^{1:k}) \le \varepsilon_{\mathrm{conc}}.
    \end{align*}
    \end{lemma}
   \begin{proof}
   See \S\ref{sec: proof of lemma optimal concentration} for a detailed proof.
   \end{proof}
    In other words, if we define the event as
    \begin{align*}
        \cE_2 = \{\forall\ i \in [n], \pi \in \Pi^{\mathrm{pur}},L^{(i),k}(f^{(i),\pi},\pi,\tau^{1:k}) \le \varepsilon_{\mathrm{conc}} \},
    \end{align*}
    we have $\Pr\{\cE_2\}\ge 1-\delta.$  
    Lemma \ref{lemma:optimal concentration} shows that the empirical discrepancy function $L^{(i),k}(f,\pi,\tau^{1:k})$ is small if the function $f$ and the policy $\pi$ are consistent, i.e. $f = f^{(i),\pi}$. Now by  \eqref{ineq:dif} and Lemma \ref{lemma:optimal concentration}, for any $i \in [n]$, 
under the event $\cE_2$,
\begin{align}
     &\EE_{\upsilon\sim \mu^{(i),\pi^k}}\Big[V^{(i),\upsilon}(\rho)\Big] - \EE_{\upsilon\sim \pi^k}\Big[V^{(i),\upsilon}(\rho)\Big] \nonumber\\&\quad= \EE_{\upsilon\sim \mu^{(i),\pi^k}}\Big[V^{(i),\upsilon}_{f^{(i),\upsilon}}(\rho)\Big]-\EE_{\upsilon\sim \pi^k}\Big[V^{(i),\upsilon}(\rho)\Big]\nonumber\\&\quad=\underbrace{\EE_{\upsilon\sim \mu^{(i),\pi^k}}\Big[V^{(i),\upsilon}_{f^{(i),\upsilon}}(\rho)\Big]- \EE_{\upsilon\sim \pi^k}\Big[V^{(i),\upsilon}_{\hat{f}^{(i),\upsilon}}(\rho)\Big]}_{\displaystyle \mathrm{(a)}} + \EE_{\upsilon\sim \pi^k}\Big[V^{(i),\upsilon}_{\hat{f}^{(i),\upsilon}}(\rho)\Big] -\EE_{\upsilon\sim \pi^k}\Big[V^{(i),\upsilon}(\rho)\Big].\nonumber
\end{align}
By \eqref{ineq:dif} and Lemma \ref{lemma:optimal concentration}, under event $\cE_2$, $(a)$ can be bounded by 
\begin{align}
     (a)&\le \eta \EE_{\upsilon\sim \mu^{(i),\pi^k}}\Big [L^{(i),k-1}(f^{(i),\upsilon},\upsilon,\tau^{1:k-1})\Big] - \eta \EE_{\upsilon\sim \pi^k}\Big[L^{(i),k-1}(\hat{f}^{(i),\upsilon},\upsilon,\tau^{1:k-1})\Big]\\
    &\le \eta \varepsilon_{\mathrm{conc}} - \eta \EE_{\upsilon\sim \pi^k}\Big[L^{(i),k-1}(\hat{f}^{(i),\upsilon},\upsilon,\tau^{1:k-1})\Big].\label{eq:similar CE}
\end{align}
Now by 
  Assumption~\ref{assum:decoup coeff}, on the events $\cE_1$ and $\cE_2$ we have 
\begin{align}
    \mathrm{Reg}(K) &= \sum_{k=1}^K\sum_{i=1}^{ n} \Big(V^{(i),\mu^{(i),\pi^k}}(\rho) - V^{(i),\pi^k}(\rho)\Big)\nonumber\\
    &= \sum_{k=1}^K\sum_{i=1}^{ n}  \left(\EE_{\upsilon\sim \mu^{(i),\pi^k}}\Big[ V^{(i),\upsilon}(\rho)\Big] -\EE_{\upsilon\sim \pi^k} \Big[V^{(i),\upsilon}(\rho)\Big]\nonumber\right)\\
    &\le \sum_{k=1}^K \sum_{i=1}^{ n} \Big(\eta \varepsilon_{\mathrm{conc}} - \eta \EE_{\upsilon\sim \pi^k}\Big[L^{(i),k-1}(\hat{f}^{(i),\upsilon},\upsilon,\tau^{1:k-1})\Big]\nonumber\\&\qquad \quad + \EE_{\upsilon\sim \pi^k}\Big[V^{(i),\upsilon}_{\hat{f}^{(i),\upsilon}}(\rho)\Big] -\EE_{\upsilon\sim \pi^k}\Big[V^{(i),\upsilon}(\rho)\Big]\Big)\label{ineq:regret step 1}.
\end{align}
Now since $\hat{f}^{(i),\upsilon} = \argmax_{f \in \cF^{(i)}}\left[V^{(i),\upsilon}_f(\rho)-\eta L^{(i),k-1}(f,\upsilon,\tau^{1:k-1})\right]$ is the optimal function with respect to the regularized value,  under the event $\cE_2$ we have 
\begin{align*}
V^{(i),\upsilon}_{\hat{f}^{(i),\upsilon}}(\rho) - \eta L^{(i),k-1}(\hat{f}^{(i),\upsilon},\upsilon,\tau^{1:k-1}) \ge V^{(i),\upsilon}_{f^{(i),\upsilon}}(\rho) - \eta L^{(i),k-1}(f^{(i),\upsilon},\upsilon,\tau^{1:k-1}),
\end{align*}
then we have $\eta L^{(i),k-1}(\hat{f}^{(i),\upsilon},\upsilon,\tau^{1:k-1})\ge 0$ and by $\eta \le 1,$ 
\begin{align*}
    \eta L^{(i),k-1}(\hat{f}^{(i),\upsilon},\upsilon,\tau^{1:k-1}) &\le V^{(i),\upsilon}_{\hat{f}^{(i),\upsilon}}(\rho) - V^{(i),\upsilon}_{f^{(i),\upsilon}}(\rho) + \eta L^{(i),k-1}(f^{(i),\upsilon},\upsilon,\tau^{1:k-1}) \\&\le R + \eta\varepsilon_{\mathrm{conc}} \le 2\varepsilon_{\mathrm{conc}},
\end{align*}
where the last inequality follows the Lemma \ref{lemma:optimal concentration}. 
If we define $$L^{(i),k-1}_{2\varepsilon_{\mathrm{conc}}}(\hat{f}^{(i),\upsilon},\upsilon,\tau^{1:k-1}) = L^{(i),k-1}(\hat{f}^{(i),\upsilon},\upsilon,\tau^{1:k-1})\cdot \II\{\eta L^{(i),k-1}(\hat{f}^{(i),\upsilon},\upsilon,\tau^{1:k-1})\le 2\varepsilon_{\mathrm{conc}}\}$$ and the event as 
\begin{align*}
    \cE_3=\left\{\forall \ i \in [n], \upsilon \in \Pi^{\mathrm{pur}}, L^{(i),k-1}_{2\varepsilon_{\mathrm{conc}}}(\hat{f}^{(i),\upsilon},\upsilon,\tau^{1:k-1}) = L^{(i),k-1}(\hat{f}^{(i),\upsilon},\upsilon,\tau^{1:k-1})\right\},
\end{align*}
we will have $\cE_3\subseteq \cE_2$. 
Since the policy $\zeta^k$ that algorithm executes is sampled from $\pi^k$, then the sequence $\{Y_k\}_{k=1}^K$ that is defined by 
\begin{align*}
Y_k&=\EE_{\upsilon\sim \pi^k}\Big[V^{(i),\upsilon}_{\hat{f}^{(i),\upsilon}}(\rho)-V^{(i),\upsilon}(\rho)-\eta L_{2\varepsilon_{\mathrm{conc}}}^{(i),k-1}(\hat{f}^{(i),\upsilon},\upsilon,\tau^{1:k-1})\Big]  \\&\qquad\quad-\left(V^{(i),\zeta^k}_{\hat{f}^{(i),\zeta^k}}(\rho)-V^{(i),\upsilon}(\rho)-\eta L_{2\varepsilon_{\mathrm{conc}}}^{(i),k-1}(\hat{f}^{(i),\zeta^k},\zeta^k,\tau^{1:k-1})\right)
\end{align*}
is a martingale difference sequence. 
Now by Azuma-Hoeffding's inequality 
 and $Y_k\le R + 2\varepsilon_{\mathrm{conc}}\le 3\varepsilon_{\mathrm{conc}}$, with probability at least $1-\delta$ we have
\begin{align}
    &\ \ \Bigg|\sum_{k=1}^K\left[ \EE_{\upsilon\sim \pi^k}\Big[V^{(i),\upsilon}_{\hat{f}^{(i),\upsilon}}(\rho)-V^{(i),\upsilon}(\rho)-\eta L_{2\varepsilon_{\mathrm{conc}}}^{(i),k-1}(\hat{f}^{(i),\upsilon},\upsilon,\tau^{1:k-1})\Big] \nonumber\right.\\&\qquad\quad\left.- \left(V^{(i),\zeta^k}_{\hat{f}^{(i),\zeta^k}}(\rho)-V^{(i),\upsilon}(\rho)-\eta L_{2\varepsilon_{\mathrm{conc}}}^{(i),k-1}(\hat{f}^{(i),\zeta^k},\zeta^k,\tau^{1:k-1})\right)\right]\Bigg | \le \cO(\varepsilon_{\mathrm{conc}}\sqrt{K}).\label{ineq:azuma}
\end{align}Define the event $\cE_4$ as the  \eqref{ineq:azuma} holds.
Now by choosing $\frac{\eta}{4} = \frac{1}{\mu} = \frac{1}{\sqrt{K}}$ and taking the union bound over the event $\cE_1,\cE_2,\cE_3$ and $\cE_4$, with probability at least $1-4\delta$, we can get
\begin{small}
\begin{align}
    &  \mathrm{Reg}(K)\nonumber\\
    & \quad  \le \sum_{i=1}^n\sum_{k=1}^K\Big( \eta \varepsilon_{\mathrm{conc}} - \eta \EE_{\upsilon\sim \pi^k}\Big[L^{(i),k-1}(\hat{f}^{(i),\upsilon},\upsilon,\tau^{1:k-1})\Big]+ \EE_{\upsilon\sim \pi^k}\Big[V^{(i),\upsilon}_{\hat{f}^{(i),\upsilon}}(\rho)\Big] -\EE_{\upsilon\sim \pi^k}\Big[V^{(i),\upsilon}(\rho)\Big]\Big)\nonumber\\
    &\quad = \sum_{i=1}^n\sum_{k=1}^K \Big(\eta \varepsilon_{\mathrm{conc}} - \eta \EE_{\upsilon\sim \pi^k}\Big[L_{2\varepsilon_{\mathrm{conc}}}^{(i),k-1}(\hat{f}^{(i),\upsilon},\upsilon,\tau^{1:k-1})\Big]+ \EE_{\upsilon\sim \pi^k}\Big[V^{(i),\upsilon}_{\hat{f}^{(i),\upsilon}}(\rho)\Big] -\EE_{\upsilon\sim \pi^k}\Big[V^{(i),\upsilon}(\rho)\Big]\Big)\nonumber\\
    &\quad \le \underbrace{\sum_{i=1}^n\sum_{k=1}^K\left(\eta \varepsilon_{\mathrm{conc}} - \eta L^{(i),k-1}(\hat{f}^{(i),\zeta^k},\zeta^k,\tau^{1:k-1})+ V^{(i),\zeta^k}_{\hat{f}^{(i),\zeta^k}}(\rho) -V^{(i),\zeta^k}(\rho)\right)}_{\displaystyle \mathrm{(b)}} + \widetilde{\cO}(n\varepsilon_{\mathrm{conc}}\sqrt{K}).\label{reganalysis:mid step}
\end{align}
\end{small}
\hspace{-0.4em}The first inequality holds because of Eq~\eqref{ineq:regret step 1}. The equality in the second line holds under Lemma \ref{lemma:optimal concentration} (event $\cE_3\subseteq \cE_2$). The second inequality is derived from Azuma-Hoeffding's inequality (event $\cE_4$). 
Now using Lemma \ref{lemma:concentration} and MADC assumption, we can get 
\begin{align*}
    (b) &\le - \sum_{i=1}^n\sum_{k=1}^K\left(\frac{\eta}{4}\left(\sum_{s=0}^{k-1}\ell^{(i),s}(f,\zeta^k)\right)\right) +\sum_{i=1}^n\sum_{k=1}^K \left(V^{(i),\zeta^k}_{\hat{f}^{(i),\zeta^k}}(\rho) -V^{(i),\zeta^k}(\rho)\right)  + 4n\sqrt{K} \cdot \eta \varepsilon_{\mathrm{conc}}\\
    &\le n\mu \cdot d_{\mathrm{MADC}}  + 6d_{\mathrm{MADC}}H+4n\sqrt{K}\varepsilon_{\mathrm{conc}}.
\end{align*}
The second inequality uses  Assumption \ref{assum:decoup coeff}. Now the regret can be bounded by 
\begin{align*}
    \mathrm{Reg}(K) &\le n\sqrt{K}\cdot d_{\mathrm{MADC}} + 6d_{\mathrm{MADC}}H + 4n\sqrt{K}\varepsilon_{\mathrm{conc}}+\cO(n\varepsilon_{\mathrm{conc}}\sqrt{K})\\
    &= \cO(n\varepsilon_{\mathrm{conc}}\sqrt{K} + nd_{\mathrm{MADC}}H + nd_{\mathrm{MADC}}\sqrt{K}).
\end{align*}
Hence, we complete the proof by noting that $\varepsilon_{\mathrm{conc}} = \widetilde{\cO}(HR^2\log \Upsilon_{\cF,\delta})$.

\paragraph{CE}
By changing the best response to the strategy modification, we can derive a proof for Correlated Equilibrium (CE).
We simplify the notation of strategy modification as $\phi_i(\upsilon^{(i)})\times \upsilon^{(-i)}$ as $\phi_i(\upsilon)$. Now we have 
 \begin{align}
     &\EE_{\upsilon\sim \pi^k}\Big[V^{(i),\upsilon}_{\hat{f}^{(i),\upsilon}}(\rho)-\eta L^{(i),k-1}(\hat{f}^{(i),\upsilon},\upsilon,\tau^{1:k-1})\Big]\nonumber\\
     &\quad = \EE_{\upsilon \sim \pi^k}\Big[\sup_{f \in \cF^{(i)}}\hat{V}^{(i),\upsilon}(f)\Big]\nonumber\\
     &\quad = \max_{\phi_i}\EE_{\upsilon \sim\pi^k}\Big[\sup_{f \in \cF^{(i)}}\hat{V}^{(i),\phi_i(\upsilon^{(i)}) \times \upsilon^{(-i)}}(f)\Big].\label{eq: CE connect 1}
    \end{align}
The second equality holds because of the property of Correlated Equilibrium.   Now we have  
    \begin{align}
    & \max_{\phi_i}\EE_{\upsilon \sim\pi^k}\Big[\sup_{f \in \cF^{(i)}}\hat{V}^{(i),\phi_i(\upsilon^{(i)}) \times \upsilon^{(-i)}}(f)\Big]\nonumber\\  
     &\quad \ge\max_{\phi_i}\EE_{\upsilon \sim\pi^k}\Big[V^{(i),\phi_i(\upsilon)}_{f^{(i),\phi_i(\upsilon)}}(\rho) - \eta L^{(i),k-1}(f^{(i),\phi_i(\upsilon)}, \phi_i(\upsilon), \tau^{1:k-1})\Big]\nonumber\\
     &\quad \ge \max_{\phi_i}\EE_{\upsilon \sim\pi^k}\Big[V^{(i),\phi_i(\upsilon)}_{f^{(i),\phi_i(\upsilon)}}(\rho) - \eta \varepsilon_{\mathrm{conc}})\Big].\label{eq: CE connect 2}
 \end{align}
  The first equality holds by $f^{(i),\phi_i(\upsilon)} \in \cF^{(i)}$ in \eqref{def:f i pi}, and
 the last inequality is derived from Lemma~\ref{lemma:optimal concentration} and $\phi_i(\upsilon)$ is a pure joint policy.
 Then, by combining \eqref{eq: CE connect 1} and \eqref{eq: CE connect 2}, we can get
 \begin{align*}
&\max_{\phi_i}\EE_{\upsilon\sim\pi^k}\Big[V^{(i),\phi_i(\upsilon)}_{f^{(i),\phi_i(\upsilon)}}(\rho)\Big]-\EE_{\upsilon\sim \pi^k}\Big[V^{(i),\upsilon}_{\hat{f}^{(i),\upsilon}}(\rho)\Big] \\&\quad\le \eta \varepsilon_{\mathrm{conc}} - \eta \EE_{\upsilon\sim \pi^k}\Big[L^{(i),k-1}(\hat{f}^{(i),\upsilon},\upsilon,\tau^{1:k-1})\Big].
 \end{align*}
 Hence, we can upper bound the regret of the agent $i$ at $k$-th episode as 
\begin{align}
    &\max_{\phi_i}\EE_{\upsilon\sim \pi^k}\Big[V^{(i),\phi_i(\upsilon^{(i)})\times \upsilon^{(-i)}}(\rho)\Big] - \EE_{\upsilon\sim \pi^k}\Big[V^{(i),\upsilon}(\rho)\Big] \nonumber\\&\quad= \max_{\phi_i}\EE_{\upsilon\sim\pi^k}\Big[V^{(i),\phi_i(\upsilon)}_{f^{(i),\phi_i(\upsilon)}}(\rho)\Big]-\EE_{\upsilon\sim \pi^k}\Big[V^{(i),\upsilon}(\rho)\Big]\nonumber\\&\quad=\max_{\phi_i}\EE_{\upsilon\sim\pi^k}\Big[V^{(i),\phi_i(\upsilon)}_{f^{(i),\phi_i(\upsilon)}}(\rho)\Big]- \EE_{\upsilon\sim \pi^k}\Big[V^{(i),\upsilon}_{\hat{f}^{(i),\upsilon}}(\rho)\Big] + \EE_{\upsilon\sim \pi^k}\Big[V^{(i),\upsilon}_{\hat{f}^{(i),\upsilon}}(\rho)\Big] -\EE_{\upsilon\sim \pi^k}\Big[V^{(i),\upsilon}(\rho)\Big]\nonumber\\
    &\quad\le \eta \varepsilon_{\mathrm{conc}} - \eta \EE_{\upsilon\sim \pi^k}L^{(i),k-1}(\hat{f}^{(i),\upsilon},\upsilon,\tau^{1:k-1}) + \EE_{\upsilon\sim \pi^k}\Big[V^{(i),\upsilon}_{\hat{f}^{(i),\upsilon}}(\rho)\Big] -\EE_{\upsilon\sim \pi^k}\Big[V^{(i),\upsilon}(\rho)\Big]\nonumber.
\end{align}
 The rest of the proof is the same as in NE/CCE after \eqref{eq:similar CE}.
\end{proof}

\subsection{Proof of Model-Based Version of Theorem~\ref{thm:mainresult}}

\begin{proof} We first consider NE/CCE. 
\paragraph{NE/CCE}
Denote $\hat{f}^{(i),\pi} = \arg\sup_{f \in \cF}\hat{V}_i^\pi(f)$ as the optimal model with respect to the regularized value $\hat{V}^{(i),\pi}(f)$. Since for model-based RL problems, the empirical discrepancy function $L(f,\pi,\tau)$ and $\ell^{(i),s}(f,\pi)$ is independent with policy $\pi$, we simplify 
 it as $L(f,\tau)$ and $\ell^{(i),s}(f)$. Then, from the definition of regularized value function $\hat{V}^{(i),\pi}(f)$, we have 
\begin{align}
     &\EE_{\upsilon\sim \pi^k}\Big[V_{\hat{f}^{(i),\upsilon}}^{(i),\upsilon}(\rho)-\eta L^{(i),k-1}(\hat{f}^{(i),\upsilon},\tau^{1:k-1})\Big]\nonumber\\
     &\quad= \EE_{\upsilon \sim \pi^k}\Big[\sup_{f \in \cF}\hat{V}^{(i),\upsilon}(f)\Big]  \ge \max_{\upsilon^{(i)} \in \Pi^{\mathrm{pur}}}\EE_{\upsilon \sim \upsilon^{(i)} \times \pi^{(-i),k}}\Big[\sup_{f \in \cF}\hat{V}^{(i),\upsilon}(f)\Big].\label{modelbased connect 1}
\end{align}
The inequality holds by the fact that $\pi^k$ is the NE/CCE of the regularized value function $\hat{V}^{(i),\pi}(f)$. Now since the best response $\pi^{(i),k,\dag}$ is a pure policy, we have 
\begin{align}
& \max_{\upsilon^{(i)} \in \Pi^{\mathrm{pur}}}\EE_{\upsilon \sim \upsilon^{(i)} \times \pi^{(-i),k}}\Big[\sup_{f \in \cF}\hat{V}^{(i),\upsilon}(f)\Big] \notag \\
     &\qquad \ge \EE_{\upsilon\sim \pi^{(i),k,\dag}\times\pi^{(-i), k}}\Big[\sup_{f \in \cF}\hat{V}^{(i),\upsilon}(f)\Big]\nonumber\\
     &\qquad\ge \EE_{\upsilon\sim \mu^{(i),\pi^k}}\Big[V_{f^*}^{(i),\upsilon}(\rho)- \eta L^{(i),k-1}(f^*,\tau^{1:k-1})\Big].  \label{modelbased connect 2}
 \end{align}
 Thus, by combining \ref{modelbased connect 1} and \ref{modelbased connect 2}, we have 
 \begin{align}
     &\EE_{\upsilon\sim \mu^{(i),\pi^k}}\Big[V_{f^*}^{(i),\upsilon}(\rho)\Big]-\EE_{\upsilon\sim \pi^k}\Big[V_{\hat{f}^{(i),\upsilon}}^{(i),\upsilon}(\rho)\Big]\nonumber \\&\quad\le \eta  L^{(i),k-1}(f^*,\tau^{1:k-1}) - \eta \EE_{\upsilon\sim \pi^k}\Big[L^{(i),k-1}(\hat{f}^{(i),\upsilon},\tau^{1:k-1})\Big]. \label{CCE similar}
 \end{align}
Now we provide our concentration lemma for model-based RL problems.
\begin{lemma}[Concentration Lemma for Model-Based RL Problems]\label{lemma:concentration:model-based}
With probability at least $1-\delta$, for any $k \in [K], f \in \cF$, for the executed policy $\{\zeta^s\}_{s=1}^{k-1}$ in Algorithm~\ref{alg:multi-agent io}, we have 
    \begin{align}
    L^{(i),k-1}(f^*, \tau^{1:k-1})-L^{(i),k-1}(f, \tau^{1:k-1}) \le -\sum_{s=1}^{k-1}\ell^{(i),s}(f) + \kappa_{\mathrm{conc}},\label{eq: concentration lemma model based eq}
    \end{align}
    where $\kappa_{\mathrm{conc}} = \max\{2H \log \frac{H\cB_\cF(1/K)}{\delta}, H \}$, where $\cB_\cF(1/K)$ is the $1/K$-bracketing number of the model class $\cF$. We also define the event $\cE_5$ as the situation when \eqref{eq: concentration lemma model based eq} holds.
\end{lemma}
\begin{proof}
   See \S\ref{sec: proof of lemma concentration model based} for detailed proof.
   \end{proof}
 By Lemma \ref{lemma:concentration:model-based}, for any $i \in [n]$, 
 \begin{small}
\begin{align}
     &\EE_{\upsilon\sim \mu^{(i),\pi^k}}\Big[V^{(i),\upsilon}(\rho)\Big] - \EE_{\upsilon\sim \pi^k}\Big[V^{(i),\upsilon}(\rho) \Big]\nonumber\\&\quad=\underbrace{\EE_{\upsilon\sim \mu^{(i),\pi^k}}\left[V_{f^*}^{(i),\upsilon}(\rho)\right]- \EE_{\upsilon\sim \pi^k}\left[V_{\hat{f}^{(i),\upsilon}}^{(i),\upsilon}(\rho)\right]}_{\displaystyle \text{(a)}}+ \EE_{\upsilon\sim \pi^k}\left[V_{\hat{f}^{(i),\upsilon}}^{(i),\upsilon}(\rho)\right] -\EE_{\upsilon\sim \pi^k}\left[V^{(i),\upsilon}(\rho)\right].\label{model connect 3}
     \end{align}
     \end{small}
Now substitute into equation \eqref{CCE similar},  
\begin{align}
    \text{(a)} &\le \eta L^{(i),k-1}(f^{*},\tau^{1:k-1}) - \eta \EE_{\upsilon\sim \pi^k}\left[L^{(i),k-1}(\hat{f}^{(i),\upsilon},\tau^{1:k-1})\right]\nonumber\\
    &= \EE_{\upsilon\sim \pi^k}\Big[\eta L^{(i),k-1}(f^{*},\tau^{1:k-1})-\eta L^{(i),k-1}(\hat{f}^{(i),\upsilon},\tau^{1:k-1}) \Big].\label{model connect4}
\end{align}
Hence, combining with \eqref{model connect 3} and \eqref{model connect4}, we can get
     \begin{align}
     &\EE_{\upsilon\sim \mu^{(i),\pi^k}}\Big[V^{(i),\upsilon}(\rho)\Big] - \EE_{\upsilon\sim \pi^k}\Big[V^{(i),\upsilon}(\rho) \Big]\nonumber\\&\quad \le \text{(a)}+ \EE_{\upsilon\sim \pi^k}\left[V_{\hat{f}^{(i),\upsilon}}^{(i),\upsilon}(\rho)\right] -\EE_{\upsilon\sim \pi^k}\left[V^{(i),\upsilon}(\rho)\right]\nonumber\\
     &\quad \le \EE_{\upsilon\sim \pi^k}\left[\eta L^{(i),k-1}(f^*, \tau^{1:k-1})-\eta L^{(i),k-1}(\hat{f}^{(i),\upsilon},\tau^{1:k-1}) + V_{\hat{f}^{(i),\upsilon}}^{(i),\upsilon}(\rho)-V^{(i),\upsilon}(\rho)\right]. \label{ineq:dif model-based}
\end{align}
By summing over $k \in [K]$ and $I \in [n]$, the regret can be obtained by
\begin{small}
\begin{align}
    &\mathrm{Reg}(K) \nonumber\\&\quad\le \sum_{i=1}^n \sum_{k=1}^K \EE_{\upsilon\sim \pi^k}\Big[\eta L^{(i),k-1}(f^*, \tau^{1:k-1})-\eta L^{(i),k-1}(\hat{f}^{(i),\upsilon},\tau^{1:k-1}) + V_{\hat{f}^{(i),\upsilon}}^{(i),\upsilon}(\rho)-V^{(i),\upsilon}(\rho)\Big]. \label{eq:regret model based}
\end{align}
\end{small}
Now we want to use Azuma-Hoeffding's inequality to transform $\upsilon\sim \pi^k$ to executed policy $\zeta^k$. To achieve this goal, note that by  Lemma~\ref{lemma:concentration:model-based},  under event $\cE_5,$ we have 
\begin{align}
     L^{(i),k-1}(f^*, \tau^{1:k-1})- L^{(i),k-1}(\hat{f}^{(i),\upsilon},\tau^{1:k-1}) \le \kappa_{\mathrm{conc}}.\label{lemma hold}
\end{align}
Moreover, since $\hat{f}^{(i),\upsilon}$ achieves the maximum value of the regularized value function $\hat{V}^{(i),\pi}(f) = V^{(i),\upsilon}_f(\rho) - L^{(i),k-1}(f^*, \tau^{1:k-1})$, we have
\begin{align*}
    L^{(i),k-1}(f^*, \tau^{1:k-1})- L^{(i),k-1}(\hat{f}^{(i),\upsilon},\tau^{1:k-1})&\ge \EE_{\upsilon\sim \mu^{(i),\pi^k}}\Big[V_{f^*}^{(i),\upsilon}(\rho)\Big]-\EE_{\upsilon\sim \pi^k}\Big[V_{f^{(i),\upsilon}}^{(i),\upsilon}(\rho)\Big]\\&\ge -R \ge -\kappa_{\mathrm{conc}}.
\end{align*}
Thus, if we define \begin{align*}
\cL_{\varepsilon}^{(i),\upsilon}&=\left(L^{(i),k-1}(f^*, \tau^{1:k-1})- L^{(i),k-1}(\hat{f}^{(i),\upsilon},\tau^{1:k-1})\right)\\&\hspace{10em}\cdot \II\Big\{|L^{(i),k-1}(f^*, \tau^{1:k-1})- L^{(i),k-1}(\hat{f}^{(i),\upsilon},\tau^{1:k-1}) |\le \varepsilon\Big\},  
\end{align*}
we can have $|\cL_{\kappa_{\mathrm{conc}}}^{(i),\upsilon}|
\le \kappa_{\mathrm{conc}}$ is bounded under event $\cE_5$. Then, with probability at least $1-\delta$, $\cL_{\kappa_{\mathrm{conc}}}^{(i),\upsilon} = L^{(i),k-1}(f^*, \tau^{1:k-1})- L^{(i),k-1}(\hat{f}^{(i),\upsilon},\tau^{1:k-1})$.
Then, we can apply Azuma-Hoeffding's inequality to transform the expectation to the executed policy $\zeta^k$.
\begin{align}
    \label{eq:azuma apply zeta^k}
    \Bigg| \sum_{k=1}^K \cL_{\kappa_{\mathrm{conc}}}^{(i),\zeta^k} -\sum_{k=1}^K\EE_{\upsilon\sim \pi^k}\Big[\cL_{\kappa_{\mathrm{conc}}}^{(i),\upsilon}\Big]\Bigg| = \cO(\kappa_{\mathrm{conc}}\cdot \log K).
\end{align}
Now by taking the union bound of Azuma-Hoeffding's inequality and event $\cE_5$, with probability at least $1-2\delta$, 
\begin{small}
\begin{align}
    \mathrm{Reg}(K) &\le \sum_{i=1}^n\sum_{k=1}^K \EE_{\upsilon\sim \pi^k}\Big[\eta L^{(i),k-1}(f^*, \tau^{1:k-1})-\eta L^{(i),k-1}(\hat{f}^{(i),\upsilon},\tau^{1:k-1}) + V_{\hat{f}^{(i),\upsilon}}^{(i),\upsilon}(\rho)-V^{(i),\upsilon}(\rho)\Big]\nonumber\\
    &= \sum_{i=1}^n\sum_{k=1}^K \EE_{\upsilon\sim \pi^k}\Big[\eta \cL_{\kappa_{\mathrm{conc}}}^{\upsilon} + V_{\hat{f}^{(i),\upsilon}}^{(i),\upsilon}(\rho)-V^{(i),\upsilon}(\rho)\Big]\nonumber\\
    &\le \underbrace{\sum_{i=1}^n\sum_{k=1}^K \left(\eta \cL_{\kappa_{\mathrm{conc}}}^{\zeta^k} + V_{\hat{f}^{(i),\zeta^k}}^{(i),\zeta^k}(\rho)-V^{(i),\zeta^k}(\rho)\right)}_{\displaystyle \text{(b)}} + \widetilde{\cO}(n\kappa_{\mathrm{conc}}),\nonumber
\end{align}
\end{small}
where the first inequality holds by \eqref{eq:regret model based}, the equality holds under event $\cE_5$, and the last inequality holds by \eqref{eq:azuma apply zeta^k}. Then, by Lemma \ref{lemma:concentration:model-based}, under event $\cE_5$, we have 
\begin{align}
    \text{(b)} &= \sum_{i=1}^n\sum_{k=1}^K \left(\eta L^{(i),k-1}(f^*, \tau^{1:k-1})- \eta L^{(i),k-1}(\hat{f}^{(i),\zeta^k},\tau^{1:k-1}) + V_{\hat{f}^{(i),\zeta^k}}^{(i),\upsilon}(\rho)-V^{(i),\zeta^k}(\rho)\right)\nonumber\\&\le \sum_{i=1}^n\sum_{k=1}^K \left(-\eta \sum_{s=1}^{k-1}\ell^{(i),s}(\hat{f}^{(i),\zeta^k}) + \eta \kappa_{\mathrm{conc}}  + V_{\hat{f}^{(i),\zeta^k}}^{(i),\upsilon}(\rho)-V^{(i),\zeta^k}(\rho)\right).\nonumber
\end{align}
Then, by Assumption \ref{assum:decoup coeff}, (b) can be further upper bounded by 
\begin{align*}
    \text{(b)}&\le \sum_{i=1}^n\sum_{k=1}^K \left(-\eta \sum_{s=1}^{k-1}\ell^{(i),s}(\hat{f}^{(i),\zeta^k}) + \eta \kappa_{\mathrm{conc}}  + V_{\hat{f}^{(i),\zeta^k}}^{(i),\upsilon}(\rho)-V^{(i),\zeta^k}(\rho)\right)\\
    &\le n \eta K \kappa_{\mathrm{conc}} + \frac{n}{\eta} d_{\mathrm{MADC}} + 6nd_{\mathrm{MADC}}H\\
    &= \widetilde{\cO}(n\kappa_{\mathrm{conc}}\sqrt{K} + nd_{\mathrm{MADC}}\sqrt{K} + nd_{\mathrm{MADC}}H).
\end{align*}
The first inequality holds by Lemma \ref{lemma:concentration:model-based}.
 The last equality holds by 
 $\eta = 4/\sqrt{K}$.
 Finally, the regret can be bounded by 
 \begin{align}
     \text{Reg}(K) \le \text{(b)} + \widetilde{\cO}(n\kappa_{\mathrm{conc}}) = \widetilde{\cO}(n\kappa_{\mathrm{conc}}\sqrt{K} + nd_{\mathrm{MADC}}\sqrt{K} + nd_{\mathrm{MADC}}H).\nonumber
 \end{align}
 Thus, we complete the proof by noting that $\kappa_{\mathrm{conc}} = \cO(H )$

\paragraph{Correlated Equilibrium}
Similar to model-free problems, we only need to replace the best response with strategy modification. 
\begin{align}
     &\EE_{\upsilon\sim \pi^k}\Big[V_{\hat{f}^{(i),\upsilon}}^{(i),\upsilon}(\rho)-\eta L^{(i),k-1}(\hat{f}^{(i),\upsilon},\tau^{1:k-1})\Big]\nonumber\\
     &\quad= \EE_{\upsilon \sim \pi^k}\Big[\sup_{f \in \cF}\hat{V}^{(i),\upsilon}(f)\Big]\nonumber\\
     &\quad= \max_{\phi_i}\EE_{\upsilon \sim \pi^k}\Big[\sup_{f \in \cF}\hat{V}^{(i),\phi_i(\upsilon)}(f)\Big].\nonumber
     \end{align}
     The last equality uses the property that $\pi^k$ is a CE with respect to the payoff function $\sup_{f \in \cF}\hat{V}^{(i),\upsilon}(f).$ Then, since $f^* \in \cF$, we can further derive 
    \begin{align}
    &\EE_{\upsilon\sim \pi^k}\Big[V_{\hat{f}^{(i),\upsilon}}^{(i),\upsilon}(\rho)-\eta L^{(i),k-1}(\hat{f}^{(i),\upsilon},\tau^{1:k-1})\Big]\nonumber\\
     &\quad\ge \max_{\phi_i}\EE_{\upsilon \sim \pi^k}\Big[\hat{V}^{(i),\phi_i(\upsilon)}(f^*)\Big]\nonumber\\
     &\quad= \max_{\phi_i}\EE_{\upsilon \sim \pi^k}\Big[V_{f^*}^{(i),\phi_i(\upsilon)}(\rho)- \eta L^{(i),k-1}(f^*, \tau^{1:k-1})\Big].\label{CE connect 1 model basd}
\end{align}
The second equality holds by the property of CE.  
 Thus, we have 
 \begin{align}
&\max_{\phi_i}\EE_{\upsilon\sim \pi^k}[V_{f^*}^{(i),\phi_i(\upsilon)}(\rho)]-\EE_{\upsilon\sim \pi^k}[V_{\hat{f}^{(i),\upsilon}}^{(i),\upsilon}(\rho)] \nonumber\\&\quad\le \eta  L^{(i),k-1}(f^*,\tau^{1:k-1}) - \eta \EE_{\upsilon\sim \pi^k}L^{(i),k-1}(\hat{f}^{(i),\upsilon},\tau^{1:k-1}).\label{CE connect 2 model based}
\end{align}

Hence, combining with \eqref{CE connect 1 model basd} and \eqref{CE connect 2 model based}, we can upper bound the regret of the agent $i$ at $k$-th episode as 
\begin{align}
    &\max_{\phi_i}\EE_{\upsilon\sim \pi^k}\Big[V^{(i),\phi_i(\upsilon^{(i)})\times \upsilon^{(-i)}}(\rho)\Big] - \EE_{\upsilon\sim \pi^k}\Big[V^{(i),\upsilon}(\rho)\Big] \nonumber\\&\quad= \max_{\phi_i}\EE_{\upsilon\sim\pi^k}\Big[V^{(i),\phi_i(\upsilon)}_{f^{(i),\phi_i(\upsilon)}}(\rho)\Big]-\EE_{\upsilon\sim \pi^k}\Big[V^{(i),\upsilon}(\rho)\Big]\nonumber\\&\quad=\max_{\phi_i}\EE_{\upsilon\sim\pi^k}\Big[V^{(i),\phi_i(\upsilon)}_{f^{(i),\phi_i(\upsilon)}}(\rho)\Big]- \EE_{\upsilon\sim \pi^k}\Big[V^{(i),\upsilon}_{\hat{f}^{(i),\upsilon}}(\rho)\Big] + \EE_{\upsilon\sim \pi^k}\Big[V^{(i),\upsilon}_{\hat{f}^{(i),\upsilon}}(\rho)\Big] -\EE_{\upsilon\sim \pi^k}\Big[V^{(i),\upsilon}(\rho)\Big]\nonumber\\
    &\quad\le \eta \varepsilon_{\mathrm{conc}} - \eta \EE_{\upsilon\sim \pi^k}L^{(i),k-1}(\hat{f}^{(i),\upsilon},\upsilon,\tau^{1:k-1}) + \EE_{\upsilon\sim \pi^k}\Big[V^{(i),\upsilon}_{\hat{f}^{(i),\upsilon}}(\rho)\Big] -\EE_{\upsilon\sim \pi^k}\Big[V^{(i),\upsilon}(\rho)\Big]\nonumber.
\end{align}
 The rest of the proof is the same as NE/CCE after \eqref{CCE similar}.
\end{proof}


\section{Proof of Theorems and Lemmas}
\subsection{Proof of Theorem~\ref{thm:multi BE in MADC}}\label{appendix:proof of BE}
\begin{proof}
The proof follows Proposition 3 in \cite{dann2021provably}.  First, we provide the following lemma in \cite{dann2021provably}.
\begin{lemma}\label{lemma:dann}
    For any positive real number sequence $x_1,\cdots, x_n$, we have 
    \begin{align*}
        \frac{\sum_{i=1}^n x_i}{\sqrt{\sum_{i=1}^n ix_i^2}}\le \sqrt{1+\log n}.
    \end{align*}
\end{lemma}
Now denote $E_{\varepsilon} = \mathrm{dim}_{\mathrm{MABE}}(\cF,\varepsilon)$. We fix  $i \in [n]$, and ignore both $h$ and $i$ for simplicity. Also denote $\hat{e}^{s,k}_h = \EE_{\pi_s}[\phi_t]$ and $e^{s,k}_h = \hat{e}^{s,k}_h \cdot \II\{\hat{e}^{s,k}_h >\varepsilon\}$, where $\phi_t = (I - \cT_h^{(i),\pi^k})f_h \in \cF_h^{(i)}$.
We initialize $K$ buckets $B_h^0, \cdots, B_h^{K-1}$, and we want to add element $e^{k,k}_h$  for $k \in [K]$ into these buckets one by one. The rule for adding elements is as follows:
If $e^{k,k}_h = 0$, we do not add it to any buckets. Otherwise we go through all buckets from $B_h^0$ to $B_h^{K-1}$, and add $e^{k,k}_h$ to $B_h^i$ whenever 
\begin{align*}
    \sum_{s\le t-1, s \in B_h^i}(e_h^{s,k})^2 < (e_h^{k,k})^2.
\end{align*}
Now assume we add $e^{k,k}_h$ into the bucket $b_h^k$. 
Then, for all $1\le i\le b_h^k-1$, we have  $(e_h^{k,k})^2 \le \sum_{s\le k-1, s \in B_h^i}(e_h^{s,k})^2$. Thus, 
\begin{align}
    \sum_{k=1}^K \sum_{s=1}^k(e_h^{s,k})^2 \ge \sum_{k=1}^K \sum_{0\le i\le b_h^k-1}\sum_{s\le t-1, s \in B_h^i}(e_h^{s,k})^2\ge \sum_{k=1}^K b_h^k(e_h^{k,k})^2.\label{connect 0 BE dimension}
\end{align}
Now note that by the definition of $\varepsilon$-independent sequence,
for the measures in $B_h^i$ $\{\pi^{k_1}, \cdots, \pi^{k_j}\}$, $\pi^{k}$ is a $\varepsilon'$-independent from all predecessors $\pi^{k_1},\cdots,\pi^{k_{j-1}}$ such that $\varepsilon'>\varepsilon$. (We can choose $\varepsilon' = e_h^{k,k}-c$ for enough small $c$ such that $\sqrt{\sum_{s\le t-1, s \in B_h^i} (e_h^{s,k})^2}\le \varepsilon'$ and $\varepsilon'>\varepsilon$ by $e_h^{k,k}>\varepsilon$.) Thus, from the definition of BE dimension, the size of each bucket cannot exceed $E_\varepsilon.$ Now by Jensen's inequality, we can get 
\begin{align}
    \sum_{k=1}^K b_h^k(e_h^{k,k})^2  = \sum_{i=1}^{K-1} i \cdot \sum_{s \in B_h^i}(e_h^{s,s})^2 \ge \sum_{i=1}^{K-1} i|B_h^i|\left(\sum_{s \in B_h^i}\frac{e_h^{s,s}}{|B_h^i|}\right)^2\ge \sum_{i=1}^{K-1} iE_\varepsilon\left(\sum_{s \in B_h^i}\frac{e_h^{s,s}}{E_\varepsilon}\right)^2, \label{connect 2 BE dimension}
\end{align}
where the last inequality uses the fact that $|B_h^i|\le E_\varepsilon.$ Let $x_i = \sum_{s \in B_h^i}(e_h^{s,s})$. By  Lemma~\ref{lemma:dann}, we have 
\begin{align}
    \sum_{i=1}^{K-1} iE_\varepsilon\left(\sum_{s \in B_h^i}\frac{e_h^{s,s}}{E_\varepsilon}\right)^2 = \frac{1}{E_\varepsilon} \sum_{i=1}^{K-1}i \cdot \left(\sum_{s \in B_h^i}e_h^{s,s}\right)^2 \ge \frac{1}{E_\varepsilon(1+\log K)} \left(\sum_{s \in [K]\setminus B_h^0}e_h^{s,s}\right)^2.\label{connect 1 BE dimension}
\end{align}
Hence, combining \eqref{connect 0 BE dimension}, \eqref{connect 2 BE dimension} and \eqref{connect 1 BE dimension}, we can get 
\begin{align}
    \sum_{s \in [K]\setminus B_h^0}e_h^{s,s} &\le \left(E_\varepsilon (1+\log K) \sum_{k=1}^K b_h^k (e_h^{k,k})^2\right)^{1/2}\nonumber\\&\le \left(E_\varepsilon (1+\log K) \sum_{k=1}^K\sum_{s=1}^k  (e_h^{s,k})^2\right)^{1/2}.\nonumber
\end{align}
Now by the definition $e^{s,k}_h = \hat{e}^{s,k}_h \cdot \II\{\hat{e}^{s,k}_h >\varepsilon\}$ and the fact that $|B_h^0| \le E_\varepsilon$,  we can have 
\begin{align}
    \sum_{h=1}^H \sum_{k=1}^K \hat{e}_h^{k,k} &\le HK\varepsilon + \sum_{h=1}^H \sum_{k=1}^K e_h^{k,k} \nonumber\\&\le HK\varepsilon + \min\{HE_\varepsilon, HK\} + \sum_{h=1}^H \sum_{s \in [K]\setminus B_h^0}e_h^{s,s}.\nonumber
\end{align}
Then, by \eqref{connect 1 BE dimension}, we can further bounded it by
\begin{align}
\sum_{h=1}^H \sum_{k=1}^K \hat{e}_h^{k,k}
    &\le HK\varepsilon + \min\{HE_\varepsilon, HK\} + \sum_{h=1}^H \left( E_\varepsilon(1+\log K)\sum_{k=1}^K\sum_{s=1}^{k-1}(e_h^{s,k})^2\right)^{1/2}\nonumber\\
    &\le HK\varepsilon + \min\{HE_\varepsilon, HK\}+\left( E_\varepsilon H(1+\log K)\sum_{h=1}^H\sum_{k=1}^K\sum_{s=1}^{k-1}(e_h^{s,k})^2\right)^{1/2}.\label{ineq:ehkk}
\end{align}
The last inequality uses the Jensen's inequality
Now we can use a similar technique in \citet{xie2021bellman}. Define $(r'_h)^{(i)}(s,a) = f_h^k(s,a) - \EE_{s'\sim \PP_h(\cdot \mid s,a)}\langle f_{h+1}^k(s',\cdot), \pi_{h+1}^k (\cdot \mid s')\rangle$. Then, we have
\begin{align*}
    \EE_{s_1\sim \rho}\left[f_1^k(s_1,\pi_1^k(s_1))\right] &= \EE_{\pi^k}\left[\sum_{h=1}^H \left(f_h^k(s_h,\pi_{h}^k(s_h))-f_{h+1}^k(s_{h+1},\pi_{h+1}^k (s_{h+1}))\right)\right]\\
    &=\EE_{\pi^k}\left[\sum_{h=1}^H \left(f_h^k(s_h,\pi_{h}^k(s_h))-\EE_{s'\sim \PP_h(\cdot \mid s,a)}\langle f_{h+1}^k(s',\cdot), \pi_{h+1}^k (\cdot \mid s')\rangle\right)\right]\\
    &= \EE_{\pi^k}\left[ \sum_{h=1}^H (r'_h)^{(i)}(s,a)\right].
\end{align*}
Hence, we can rewrite the regret of the $k$-th episode as 
\begin{align}
   \EE_{s_1\sim \rho}\left[(f_1^k(s_1,\pi_1^k(s_1))-V^{(i),\pi^k}(s_1))\right]
   = \EE_{\pi^k}\left[\sum_{h=1}^H ((r'_h)^{(i)}(s,a) - r_h^{(i)}(s,a))\right].\label{transform mdp BE dimension}
   \end{align}
   The last inequality uses the fact that 
   Then, substitute into the definition of $(r'_h)^{(i)}$, we can get 
   \begin{align}&\EE_{s_1\sim \rho}\left[(f_1^k(s_1,\pi_1^k(s_1))-V^{(i),\pi^k}(s_1))\right]\nonumber\\&\quad=
   \EE_{\pi^k}\left[\sum_{h=1}^H (f_h^k(s,a) - \EE_{s'\sim \PP_h(\cdot \mid s,a)}\langle f_{h+1}^k(s',\cdot), \pi_{h+1}^k (\cdot \mid s')\rangle - r_h^{(i)}(s,a))\right]\nonumber\\
   &\quad= \EE_{\pi^k}\left[\sum_{h=1}^H (I-\cT_h^{(i),\pi^k})(f_h^k)\right] = \sum_{h=1}^H \hat{e}_h^{k,k},\nonumber
\end{align}
where the first equality holds by the definition of $M_k$, the second equality holds by decomposing the value function to the expected cumulative sum of the reward function, and the last equality is derived by the definition of $\hat{e}_h^{k,k}.$
Now we can get
\begin{align}
\text{Reg}(K) \nonumber&\le \sum_{k=1}^K \sum_{h=1}^H \hat{e}_h^{k,k} \nonumber\\& \le HK\varepsilon + \min\{HE_\varepsilon, HK\}+\left( (E_\varepsilon H\cdot 2\log K)\sum_{h=1}^H\sum_{k=1}^K\sum_{s=1}^{k-1}(e_h^{s,k})^2\right)^{1/2}.\label{BE dimension connect 3}
\end{align}
The last inequality holds by \eqref{ineq:ehkk}. Now by the definition $e_h^{s,k} \le \hat{e}_h^{s,k} = \EE_{\pi^s}[(I-\cT_h^{(i),\pi^k})(f_h)]$ and the basic inequality $\sqrt{ab} \le \mu a + b/\mu$ for $\mu >0$, we can derive 
\begin{align}
    &\text{Reg}(K)\nonumber\\ 
    &\quad\le HK\varepsilon + \min\{HE_\varepsilon, HK\} + \mu \cdot (E_\varepsilon H\cdot 2\log K) + \frac{1}{\mu}\sum_{h=1}^H \sum_{k=1}^K \sum_{s=1}^{k-1}\left(\EE_{\pi^s}\left[(I-\cT_h^{(i),\pi^k})(f_h)\right]\right)^2\nonumber\\
    &\quad\le HK\varepsilon + \min\{HE_\varepsilon, HK\} + \mu \cdot (E_\varepsilon H\cdot 2\log K) + \frac{1}{\mu}\sum_{h=1}^H \sum_{k=1}^K \sum_{s=1}^{k-1}\EE_{\pi^s}\left[\left(I-\cT_h^{(i),\pi^k})(f_h)\right)^2\right].\nonumber
\end{align}
The last inequality holds by $(\EE[X])^2 \le \EE[X^2]$.  
Thus, by choosing $\varepsilon = 1/K$, we can derive 
\begin{align}
    \text{Reg}(K) &\le H + HE_{1/K} + \mu (E_{1/K} H\cdot 2 \log K) + \frac{1}{\mu}\sum_{h=1}^H \sum_{k=1}^K \sum_{s=1}^{k-1}\EE_{\pi^s}\left[\left(I-\cT_h^{(i),\pi^k})(f_h)\right)^2\right]\nonumber\\
    &\le 6 d_{\mathrm{MADC}}H + \mu d_{\mathrm{MADC}} + \frac{1}{\mu}\sum_{h=1}^H \sum_{k=1}^K \sum_{s=1}^{k-1}\EE_{\pi^s}\left[\left(I-\cT_h^{(i),\pi^k})(f_h)\right)^2\right],\nonumber
\end{align}
where $d_{\mathrm{MADC}} = \max \{2E_{1/K}H \log K,1\} = \cO(2\text{dim}_{\text{MABE}}(\cF,1/K)H\log K)$. 
The last inequality uses the fact that $\log K \ge 1$ and $H \ge 1.$
\end{proof}
\subsection{Proof of Theorem \ref{thm:low eluder dimension}}\label{appendix:low eluder dimension}
\begin{proof}
    For any policy $\pi$ and $i \in [n]$, assume $\delta_{z_1},\cdots,\delta_{z_m}$ is an $\varepsilon$-independent sequence with respect to $\bigcup_{\pi \in \Pi^{\text{pur}}} \cF_h^{(i),\pi}= \bigcup_{\pi \in \Pi^{\text{pur}}}(I-\cT^{(i),\pi})\cF$, where $\delta_{z_1}, \delta_{z_2},\cdots,\delta_{z_m} \in \cD_\Delta$, i.e. $\delta_{z_i}$ is a Dirichlet probability measure over $\cS\times \cA$ that  $\delta_{z_i}=\delta_{(s,a)}(\cdot)$. Then, for each $j \in [m]$, there exist function $f^j \in \cF^{(i)}$ and policy $\pi^j \in \Pi$ such that 
    $|(I-\cT^{(i),\pi^j})f_j(z_j)|>\varepsilon$ and $\sqrt{\sum_{p=1}^{j-1} |(f_h^p - \cT^{(i),\pi^p}f_{h+1}^p)(z_p)|^2} \le \varepsilon$. Define $g_h^j = \cT^{(i),\pi^j}f_{h+1}^j$, by Assumption~\ref{assum: realizability}, we have $g_h^j \in \cF_h^{(i)}\subseteq \cF_h.$ Thus, $|(f_h^j-g_h^j)(z_j)| > \varepsilon$ and $\sqrt{\sum_{p=1}^j |(f_h^p - g_h^p)(z_p)|^2}<\varepsilon.$ Thus, by the definition of eluder dimension, we have $m\le \mbox{dim}_{\mathrm{E}}(\cF_h, \varepsilon)$. Hence, for all $i$ and policy $\pi$, $$\dim_{\mathrm{E}}(\cF_h,\varepsilon) \ge m \ge \max_{h \in [H]} \max_{i \in [n]}\mbox{dim}_{\mathrm{DE}}\left(\bigcup_{\pi \in \Pi^{\text{pur}}}\cF_h^{(i),\pi}, \cD_{h,\Delta}, \varepsilon\right),$$
which concludes the proof.  
\end{proof}

\subsection{Proof of Theorem~\ref{thm:bilinear class}}\label{appendix: bilinear class}
\begin{proof}
 First,  by the elliptical potential lemma  introduced  in Lemma  \ref{lemma:elliptical }, if we define $\Lambda_{k,h} = \varepsilon I + \sum_{s=1}^{k-1}x_{k,h}x_{k,h}^T$, for any $\{x_{k,h}\}_{k=1}^K \in \cX_h$ we  have  
\begin{align}\label{ineq:information gain}
    \sum_{k=1}^K \sum_{h=1}^H \min\left\{1, \Vert x_{k,h}\Vert^2_{\Lambda_{k,h}^{-1}}\right\}\le \sum_{h=1}^H 2\log \det\left(I + \frac{1}{\varepsilon}\sum_{k=1}^K x_{k,h}x_{k,h}^T\right) = 2\gamma_K(\varepsilon,\cX).
\end{align}
Now denote $\Sigma_{k,h} = \varepsilon I + \sum_{s=1}^{k-1}X_h(\pi^s)X_h(\pi^s)^T$.
Similar to Section \ref{appendix:proof of BE},  define $(r'_h)^{(i)}(s,a) = f_h^k(s,a) - \EE_{s'\sim \PP_h(\cdot \mid s,a)}\langle f_{h+1}^k(s',\cdot), \pi_{h+1}^k (\cdot \mid s')\rangle \in [-1,1]$, then we can have 
\begin{align*}
   \EE_{s_1\sim \rho}\left[f_1^k(s_1,\pi_1^k(s_1))-V^{(i),\pi^k}(s_1)\right]= \EE_{\pi^k}\left[\sum_{h=1}^H ((r'_h)^{(i)}(s,a) - r_h^{(i)}(s,a))\right].\end{align*}
   Then, we can substitute the definition of $(r'_h)^{(i)}$ and derive \begin{align*}&\EE_{s_1\sim \rho}\left[f_1^k(s_1,\pi_1^k(s_1))-V^{(i),\pi^k}(s_1)\right]\\&\quad=
   \EE_{\pi^k}\left[\sum_{h=1}^H (f_h^k(s,a) - \EE_{s'\sim \PP_h(\cdot \mid s,a)}\langle f_{h+1}^k(s',\cdot), \pi_{h+1}^k (\cdot \mid s')\rangle - r_h^{(i)}(s,a))\right]\\
   &\quad= \sum_{h=1}^H\min\left\{\left|\langle W_h^{(i)}(f^k,\pi^k) - W_h^{(i)}(f^{\mu^{(i),\pi^k}},\mu^{(i),\pi^k}), X_h(\pi^k)\rangle_{\cV}\right|,2R\right\}.
\end{align*}
Then, by $\min\{x,2R\} \le 2R\min\{x,1\}$, we have 
\begin{small}
\begin{align}
    \text{Reg}(K) &=\sum_{k=1}^K \EE_{s_1\sim \rho}\left[(f_1^k(s_1,\pi_1^k(s_1))-V^{(i),\pi^k}(s_1))\right] \nonumber\\&\le 2R\sum_{k=1}^K\sum_{h=1}^H\min\left\{\left|\langle W_h^{(i)}(f^k,\pi^k) - W_h^{(i)}(f^{\mu^{(i),\pi^k}},\mu^{(i),\pi^k}), X_h(\pi^k)\rangle_{\cV}\right|,1\right\}\nonumber\\
    &= 2R\sum_{k=1}^K\sum_{h=1}^H\min\left\{\left|\langle W_h^{(i)}(f^k,\pi^k) - W_h^{(i)}(f^{\mu^{(i),\pi^k}},\mu^{(i),\pi^k}), X_h(\pi^k)\rangle_{\cV}\right|,1\right\}\nonumber\\&\qquad\cdot \left(\II\left\{\left\Vert X_h(\pi^k)\right\Vert_{\Sigma_{k,h}^{-1}}\le 1\right\} +\II\left\{\left\Vert X_h(\pi^k)\right\Vert_{\Sigma_{k,h}^{-1}}> 1\right\} \right)\label{decompose two terms}.\end{align}\end{small}
    The last inequality is because $1 = \II\{\cE\} + \II\{\lnot E\}$ for any event. 
    Now we decompose the \eqref{decompose two terms} into two terms $A+ B$, where
    \begin{small}
    \begin{align}
        A &= 2R\sum_{k=1}^K\sum_{h=1}^H\min\left\{\left|\langle W_h^{(i)}(f^k,\pi^k) - W_h^{(i)}(f^{\mu^{(i),\pi^k}},\mu^{(i),\pi^k}), X_h(\pi^k)\rangle_{\cV}\right|,1\right\}\cdot \II\left\{\left\Vert X_h(\pi^k)\right\Vert_{\Sigma_{k,h}^{-1}}\le 1\right\},\\
        B &= 2R\sum_{k=1}^K\sum_{h=1}^H\min\left\{\left|\langle W_h^{(i)}(f^k,\pi^k) - W_h^{(i)}(f^{\mu^{(i),\pi^k}},\mu^{(i),\pi^k}), X_h(\pi^k)\rangle_{\cV}\right|,1\right\}\cdot \II\left\{\left\Vert X_h(\pi^k)\right\Vert_{\Sigma_{k,h}^{-1}}> 1\right\}.\label{def of B bilinear}
    \end{align}
    \end{small}
    Now we bound $A$ and $B$ respectively. For $A$, we can use Cauchy's inequality and get 
    \begin{align}
    A&\le 2R\sum_{k=1}^K\sum_{h=1}^H\left\Vert W_h^{(i)}(f^k,\pi^k) - W_h^{(i)}(f^{\mu^{(i),\pi^k}},\mu^{(i),\pi^k})\right\Vert_{\Sigma_{k,h}}\cdot\left\Vert X_h(\pi^k)\right\Vert_{\Sigma_{k,h}^{-1}}\cdot\II\left\{\left\Vert X_h(\pi^k)\right\Vert_{\Sigma_{k,h}^{-1}}\le 1\right\} 
    \nonumber\\
    &\quad\le 2R\sum_{k=1}^K\sum_{h=1}^H\left\Vert W_h^{(i)}(f^k,\pi^k) - W_h^{(i)}(f^{\mu^{(i),\pi^k}},\mu^{(i),\pi^k})\right\Vert_{\Sigma_{k,h}}\cdot\min\left\{\left\Vert X_h(\pi^k)\right\Vert_{\Sigma_{k,h}^{-1}},1\right\}.\label{eq: A bound bilinear}
\end{align}
The first inequality holds by Cauchy's inequality that $|\langle X, Y\rangle| \le \Vert X\Vert_{\Sigma} \Vert Y\Vert_{\Sigma^{-1}}.$

Now by the definition $\Sigma_{k,h} = \varepsilon I + \sum_{s=1}^{k-1}X_h(\pi^s)X_h(\pi^s)^T$, we expand the term $\Vert W_h^{(i)}(f^k,\pi^k) - W_h^{(i)}(f^{\mu^{(i),\pi^k}},\mu^{(i),\pi^k})\Vert_{\Sigma_{k,h}}$ as 
\begin{small}
\begin{align*}
    &\left\Vert W_h^{(i)}(f^k,\pi^k) - W_h^{(i)}(f^{\mu^{(i),\pi^k}},\mu^{(i),\pi^k})\right\Vert_{\Sigma_{k,h}} \\&\quad= \left[\varepsilon \cdot \Big\Vert W_h^{(i)}(f^k,\pi^k) - W_h^{(i)}(f^{\mu^{(i),\pi^k}},\mu^{(i),\pi^k})\Big\Vert_2^2 + \sum_{s=1}^{k-1}\Big|\langle W_h^{(i)}(f^k,\pi^k) - W_h^{(i)}(f^{\mu^{(i),\pi^k}},\mu^{(i),\pi^k}),X_h(\pi^s)\rangle\Big|^2\right]^{1/2}\\
    &\quad\le 2\sqrt{\varepsilon}B_W+\left[\sum_{s=1}^{k-1}\Big|\langle W_h^{(i)}(f^k,\pi^k) - W_h^{(i)}(f^{\mu^{(i),\pi^k}},\mu^{(i),\pi^k}),X_h(\pi^s)\rangle\Big|^2\right]^{1/2}.
\end{align*}
\end{small}
The last inequality holds by $\sqrt{a+b} \le \sqrt{a} + \sqrt{b}.$
Then, we can get 
\begin{align}
    &\sum_{k=1}^K\sum_{h=1}^H\Vert W_h^{(i)}(f^k,\pi^k) - W_h^{(i)}(f^{\mu^{(i),\pi^k}},\mu^{(i),\pi^k})\Vert_{\Sigma_{k,h}}\min\Big\{\Vert X_h(\pi^k)\Vert_{\Sigma_{k,h}^{-1}},1\Big\}\nonumber\\&\quad\le \sum_{k=1}^K\sum_{h=1}^H\left(2\sqrt{\varepsilon}B_W+\left[\sum_{s=1}^{k-1}|\langle W_h^{(i)}(f^k,\pi^k) - W_h^{(i)}(f^{\mu^{(i),\pi^k}},\mu^{(i),\pi^k}),X_h(\pi^s)\rangle|^2\right]^{1/2}\right)\nonumber\\&\qquad\quad\cdot\min\Big\{\Vert X_h(\pi^k)\Vert_{\Sigma_{k,h}^{-1}},1\Big\} \le A_1+A_2. \label{eq: C, D bilinear}\end{align}
    where $A_1$ and $A_2$ are defined as follows:
    \begin{small}
    \begin{align}
        A_1 &= \left(\sum_{k=1}^K\sum_{h=1}^H 4\varepsilon B_W^2\right)^{1/2}\cdot \left(\sum_{k=1}^K\sum_{h=1}^H \min\Big\{\Vert X_h(\pi^k)\Vert^2_{\Sigma_{k,h}^{-1}},1\Big\}\right)^{1/2}\nonumber\\
        A_2 &= \left(\sum_{k=1}^K\sum_{h=1}^H \sum_{s=1}^{k-1}|\langle W_h^{(i)}(f^k,\pi^k) - W_h^{(i)}(f^{\mu^{(i),\pi^k}},\mu^{(i),\pi^k}),X_h(\pi^s)\rangle|^2\right)^{1/2}\nonumber\\& \hspace{10em} \cdot \left(\sum_{k=1}^K\sum_{h=1}^H \min\{\Vert X_h(\pi^k)\Vert^2_{\Sigma_{k,h}^{-1}},1\}\right)^{1/2},\nonumber
    \end{align}
    \end{small}
    Now we bound $A_1$ and $A_2$ respectively. First, for $A_1$, using \eqref{ineq:information gain}, we have 
    \begin{align}
        A_1 \le \sqrt{4\varepsilon KH B_W^2 \cdot 2\gamma_K(\varepsilon,\cX)}.\nonumber
    \end{align}
    Then, for $A_2$, we have  
    \begin{align}
        A_2 &= \left(\sum_{k=1}^K\sum_{h=1}^H \sum_{s=1}^{k-1}|\langle W_h^{(i)}(f^k,\pi^k) - W_h^{(i)}(f^{\mu^{(i),\pi^k}},\mu^{(i),\pi^k}),X_h(\pi^s)\rangle|^2\right)^{1/2}\sqrt{  2\gamma_K(\varepsilon,\cX)}\nonumber\\
        &=\left(\sum_{k=1}^K\sum_{s=1}^{k-1}\ell^{(i),s}(f^k,\pi^k)\right)^{1/2}\sqrt{ 2\gamma_K(\varepsilon,\cX)}.\nonumber
    \end{align}
    The equality holds by the definition of $\ell^{(i),s}$ and the definition of multi-agent bilinear class \eqref{eq: ma bilinear class}.

    Then, since $\sqrt{ab} \le a\mu + b/\mu$ for any $\mu>0$, we can further derive 
    \begin{align}
        A_2 \le 2R\mu \cdot 2\gamma_K(\varepsilon,\cX) + \frac{1}{2R\mu}\sum_{k=1}^K \sum_{s=1}^{k-1}\ell^{(i),s}(f^k,\pi^k).\nonumber
    \end{align}
    
Now by adding $A_1$ and $A_2$ and combining with \eqref{eq: A bound bilinear} and \eqref{eq: C, D bilinear}, we can finally get
\begin{align}
    A \le 2R(A_1 + A_2) &\le \sqrt{4\varepsilon K HB_W^2 \cdot 8R\gamma_K(\varepsilon,\cX)} + \mu \cdot 8R^2\gamma_K(\varepsilon, \cX) + \frac{1}{\mu} \sum_{k=1}^K \sum_{s=1}^{k-1} \ell^{(i),s}(f^k,\pi^k) \nonumber\\
    & \le 32R B_W^2 \varepsilon HK + \gamma_K(\varepsilon,\cX) + \mu \cdot 8R^2\gamma_K(\varepsilon,\cX) + \frac{1}{\mu} \sum_{k=1}^K \sum_{s=1}^{k-1} \ell^{(i),s}(f^k,\pi^k).\nonumber
\end{align}
Now we have complete the bound of $A$. For $B$, by \eqref{ineq:information gain}, since $\II\{x > 1\} \le \min\{1, x^2\}$, we know that 
\begin{align}
    \sum_{k=1}^K \sum_{h=1}^H \II\left\{\|X_h(\pi^k)\|_{\Sigma_{k,h}^{-1}}>1\right\} \le \sum_{k=1}^K \sum_{h=1}^H \min\left\{1,\|x_{k,h}\|^2_{\Sigma_{k,h}^{-1}}\right\} \le 2\gamma_K(\varepsilon,\cX).\label{B upper bound bilinear step1}
\end{align}
Thus, by the definition of $B$ in \eqref{def of B bilinear}, we can derive 
\begin{align}
    B\le 2 \sum_{k=1}^K \sum_{h=1}^H \II\left\{\|X_h(\pi^k)\|_{\Sigma_{k,h}^{-1}}>1\right\} \le 2\gamma_K(\varepsilon,\cX).\nonumber
\end{align}
Now note that $\text{Reg}(K) \le A + B$, then by choosing $\varepsilon = 1/32RKB_W^2$ and $d_{\mathrm{MADC}} = \max\{1,8R^2\gamma_K(\varepsilon,\cX)\}$   we can derive 
\begin{align}
    \text{Reg}(K) &\le A+ B\nonumber\\
    &\le 32RB_W^2 \varepsilon HK + 3\gamma_K(\varepsilon,\cX) + \mu \cdot 8R^2\gamma_K(\varepsilon,\cX) + \frac{1}{\mu}\sum_{k=1}^K \sum_{s=1}^{k-1}\ell^{(i),s}(f^k,\pi^k)\nonumber\\
    &= H + 3\gamma_K(\varepsilon,\cX) + \mu \cdot 8R^2\gamma_K(\varepsilon,\cX) + \frac{1}{\mu}\sum_{k=1}^K \sum_{s=1}^{k-1}\ell^{(i),s}(f^k,\pi^k)\nonumber\\
    &\le 6d_{\mathrm{MADC}}H + \mu d_{\mathrm{MADC}} + \frac{1}{\mu}\sum_{k=1}^K \sum_{s=1}^{k-1}\ell^{(i),s}(f^k,\pi^k).\nonumber
\end{align}
The last inequality uses the fact that $d_{\mathrm{MADC}}\ge 1, H\ge 1$. Hence, we complete the proof.


\end{proof}
\subsection{Proof of Theorem~\ref{thm:witness rank}}\label{appendix:proof of witnessrank}
\begin{proof}
In this subsection, we  give a detailed proof of Theorem~\ref{thm:witness rank}. First, similar to the performance difference lemma in \cite{jiang2017contextual}, we have
\begin{align}
    &\EE_{s_1\sim \rho}\left[V_{1,f}^{(i),\pi^k}(s_1)-V_{1}^{(i),\pi^k}(s_1)\right]\nonumber\\
    &\quad = \EE_{\pi^k}\left[Q_{1,f}^{(i),\pi^k}(s_1,a_1)\right]-\EE_{\pi^k}\left[\sum_{h=1}^H r_h^{(i)}(s_h,a_h)\right]\nonumber\\
    & \quad = \EE_{\pi^k}\left[\sum_{h=1}^H \left(Q_{h,f}^{(i),\pi^k}(s_h,a_h)-r_h^{(i)}(s_h,a_h)-Q_{h+1,f}^{(i),\pi^k}(s_{h+1}, a_{h+1})\right)\right]\label{eq: decompose sum}.
\end{align}
The last equality holds by splitting the term. 
Now, since \begin{align}\EE_{\pi^k}\left[Q_{h+1,f}^{(i),\pi^k}(s_{h+1},a_{h+1})\right] = \EE_{\pi^k}\left[V_{h+1,f}^{(i),\pi^k}(s_{h+1})\right] =\EE_{\pi^k} \left[\EE_{s_{h+1}\sim \PP_{h,f^*}(\cdot \mid s_h,a_h)}\left[V_{h+1,f}^{(i),\pi^k}(s_{h+1})\right]\right], \end{align}
we can rewrite \eqref{eq: decompose sum} as 
\begin{align}
    &\EE_{s_1\sim \rho}\left[V_{1,f}^{(i),\pi^k}(s_1)-V_{1}^{(i),\pi^k}(s_1)\right]\nonumber\\
    &= \EE_{\pi^k}\left[\sum_{h=1}^H \left(Q_{h,f}^{(i),\pi^k}(s_h,a_h)-r_h^{(i)}(s_h,a_h)-\EE_{s_{h+1}\sim \PP_{h,f^*}(\cdot \mid s_h,a_h)}V_{h+1,f}^{(i),\pi^k}(s_{h+1})\right)\right]\nonumber\\
    &=\sum_{k=1}^K\EE_{\pi^k}\left[\sum_{h=1}^H (\EE_{s_{h+1}\sim \PP_{h,f^k}(\cdot \mid s_h,a_h)}-\EE_{s_{h+1}\sim \PP_{h,f^*}(\cdot \mid s_h,a_h)})\left[V_{h+1,f^k}^{(i),\pi^k}(s_{h+1})\right]\right]\label{eq: witness step 1}.
\end{align}
Then, combining \eqref{eq: witness step 1} and the definition of multi-agent witness rank \eqref{condition: witness rank 2}, we can derive 
\begin{small}
\begin{align}
    &\sum_{k=1}^K \EE_{s_1\sim \rho}\left[V_{1,f^k}^{(i),\pi^k}(s_1)-V_{1}^{(i),\pi^k}(s_1)\right]\nonumber\\
    &\quad\le \sum_{k=1}^K\sum_{h=1}^H \min\left\{R,\frac{1}{\kappa_{\mathrm{wit}}}|\langle W_h(f^k), X_h(\pi^k)\rangle |\right\}\nonumber\\
    & \quad\le \sum_{k=1}^K\sum_{h=1}^H \min\left\{R,\frac{1}{\kappa_{\mathrm{wit}}}|\langle W_h(f^k), X_h(\pi^k)\rangle |\right\} \left(\II\Big\{\Vert X_h(\pi^k)\Vert_{\Sigma_{k,h}^{-1}}\le 1\Big\} + \II\Big\{\Vert X_h(\pi^k)\Vert_{\Sigma_{k,h}^{-1}}\ge 1\Big\}\right).\label{eq: witness step 3}
\end{align}
\end{small}
Now note that 
\begin{align*}
    \sum_{k=1}^K \min\left\{1,\Vert X_h(\pi^k)\Vert_{\Sigma_{k,h}^{-1}}^2\right\} \le 2d\log \left(\frac{\varepsilon+K}{\varepsilon}\right) \triangleq \cD(\varepsilon).
\end{align*}
and $ \II\{x>1\}\le \min\{1,x^2\}$, we can derive 
\begin{align}
    \sum_{k=1}^K \sum_{h=1}^H \II \{\Vert X_h(\pi^k)\Vert_{\Sigma_{k,h}^{-1}}> 1\} \le \cD(\varepsilon)H.\label{eq:witness step 2}
\end{align}
Then, combining \eqref{eq: witness step 3} and \eqref{eq:witness step 2}, we can get 
\begin{small}
\begin{align}
    &\sum_{k=1}^K \EE_{s_1\sim \rho}\left[V_{1,f^k}^{(i),\pi^k}(s_1)-V_{1}^{(i),\pi^k}(s_1)\right]\nonumber\\&\quad\le R\sum_{k=1}^K\sum_{h=1}^H \min\left\{1,\frac{1}{\kappa_{\mathrm{wit}}}|\langle W_h(f^k), X_h(\pi^k)\rangle |\right\}\cdot \left(\II\{\Vert X_h(\pi^k)\Vert_{\Sigma_{k,h}^{-1}}\le 1\} + \II\{\Vert X_h(\pi^k)\Vert_{\Sigma_{k,h}^{-1}}> 1\}\right)\nonumber\\
    &\quad\le R\sum_{k=1}^K\sum_{h=1}^H \min\left\{1,\frac{1}{\kappa_{\mathrm{wit}}}|\langle W_h(f^k), X_h(\pi^k)\rangle |\right\}\cdot \left(\II\{\Vert X_h(\pi^k)\Vert_{\Sigma_{k,h}^{-1}}\le 1\}\right) + \cD(\varepsilon)HR\nonumber\\
    &\quad\le R\underbrace{\sum_{k=1}^K\sum_{h=1}^H \frac{1}{\kappa_{\mathrm{wit}}}\Vert W_h(f^k)\Vert_{\Sigma_{k,h}}\min\left\{1,\Vert X_h(\pi^k)\Vert^2_{\Sigma_{k,h}^{-1}}\right\}}_{\displaystyle (A)} + \cD(\varepsilon)HR.\label{regret decompose step 1}
\end{align}
\end{small}
The last inequality uses the Cauchy's inequality $\langle X, Y \rangle \le \|X\|_A \|Y\|_{A^{-1}}$ and the fact that $x \cdot \II\{x\le 1\} \le \min\{1,x^2\}$.
Further, by the definition of $\Sigma_{k,h}$, we decompose the first term as 
\begin{align*}
    (A)&\le \frac{1}{\kappa_{\mathrm{wit}}}\sum_{k=1}^K\sum_{h=1}^H\left[\varepsilon\cdot \Vert W_h(f^k)\Vert_2^2 + \sum_{s=1}^k |\langle W_h(f^k), X_h(\pi^s)\rangle|^2\right]^{1/2}\min\left\{1,\Vert X_h(\pi^k)\Vert_{\Sigma_{k,h}^{-1}}\right\}\\
    &\le \frac{1}{\kappa_{\mathrm{wit}}}\sum_{k=1}^K\sum_{h=1}^H\left(\sqrt{\varepsilon}B_W  + \left[\sum_{s=1}^k |\langle W_h(f^k), X_h(\pi^s)\rangle|^2\right]^{1/2}\right)\min\left\{1,\Vert X_h(\pi^k)\Vert_{\Sigma_{k,h}^{-1}}\right\}.
\end{align*}
The second inequality is derived by the inequality $\Vert W_h(f^k)\Vert \le B_W$ and $\sqrt{a+b}\le \sqrt{a}+ \sqrt{b}$.
Now sum over $k \in [K]$ and $h \in [H]$, we can get
\begin{align}
    (A)&\le \sum_{k=1}^K\sum_{h=1}^H \frac{1}{\kappa_{\mathrm{wit}}}\left(\sqrt{\varepsilon}B_W  + \left[\sum_{s=1}^k |\langle W_h(f^k), X_h(\pi^s)\rangle|^2\right]^{1/2}\right)\min\left\{1,\Vert X_h(\pi^k)\Vert_{\Sigma_{k,h}^{-1}}\right\}\nonumber\\
    &\le \underbrace{\frac{1}{\kappa_{\mathrm{wit}}}\sum_{k=1}^K \sum_{h=1}^H \sqrt{\varepsilon} B_W \min\left\{1,\Vert X_h(\pi^k)\Vert_{\Sigma_{k,h}^{-1}}\right\}}_{\displaystyle (X)} \nonumber\\ &\qquad + \underbrace{\frac{1}{\kappa_{\mathrm{wit}}}\sum_{k=1}^K \sum_{h=1}^H \left[\sum_{s=1}^k |\langle W_h(f^k), X_h(\pi^s)\rangle|^2\right]^{1/2}\min\left\{1,\Vert X_h(\pi^k)\Vert_{\Sigma_{k,h}^{-1}}\right\}}_{\displaystyle (Y)}.\label{A decompose to X and Y}
    \end{align}
First, we try to give an upper bound for (X). By Cauchy's inequality and \eqref{eq:witness step 2}, we can derive 
\begin{align}
   \displaystyle (X) &\le \frac{1}{\kappa_{\mathrm{wit}}}\left(\sum_{k=1}^K \sum_{h=1}^H \varepsilon B_W^2 \right)^{1/2} \left(\sum_{k=1}^K \sum_{h=1}^H \min\left\{1,\Vert X_h(\pi^k)\Vert_{\Sigma_{k,h}^{-1}}^2\right\}\right)^{1/2}\nonumber\\
    &\le \frac{1}{\kappa_{\mathrm{wit}}}\sqrt{HK\varepsilon B_W^2 \cdot \cD(\varepsilon) H}\le \frac{HK\varepsilon B_W^2}{\kappa_{\mathrm{wit}}^2} + \cD(\varepsilon)H.\label{eq: X upper bound}
\end{align}
On the other hand, for (Y), we can bound it using Cauchy's inequality that $\sum_{a,b} \sqrt{ab} \le \sqrt{(\sum_a a)\cdot (\sum_b b)}$, 
\begin{align}
   \displaystyle (Y) &\le \frac{1}{\kappa_{\mathrm{wit}}}\left(\left(\sum_{k=1}^K \sum_{h=1}^H \sum_{s=1}^k |\langle W_h(f^k), X_h(\pi^s)\rangle|^2\right)\left(\sum_{k=1}^K \sum_{h=1}^H \min\left\{1,\Vert X_h(\pi^k)\Vert_{\Sigma_{k,h}^{-1}}^2\right\}\right)\right)^{1/2}\nonumber\\
    &\le \frac{1}{\kappa_{\mathrm{wit}}}\sqrt{\cD(\varepsilon)H\left(\sum_{k=1}^K \sum_{h=1}^H \sum_{s=1}^k |\langle W_h(f^k), X_h(\pi^s)\rangle|^2\right)}.\nonumber
\end{align}
The last inequality holds by the definition of $\cD(\varepsilon)$ in \ref{eq:witness step 2}.
Now by the definition of multi-agent witness rank \ref{condition: witness rank 1}, we note that 
\begin{align}
    &|\langle W_h(f^k), X_h(\pi^s)\rangle|^2\nonumber\\& \quad\le \left(\max_{v \in \cV_h}\EE_{(s_h,a_h)\sim \pi}[(\EE_{s_{h+1}\sim  \PP_{h,f^k}(\cdot \mid s_h,a_h)}-\EE_{s_{h+1}\sim \PP_{h,f^*}(\cdot \mid s_h,a_h)})v(s_h,a_h,s_{h+1})] \right)^2\nonumber\\&\quad\le \max_{v \in \cV_h}
     \EE_{(s_h,a_h)\sim \pi^s}\left[\left((\EE_{s_{h+1}\sim  \PP_{h,f^k}(\cdot \mid s_h,a_h)}-\EE_{s_{h+1}\sim \PP_{h,f^*}(\cdot \mid s_h,a_h)})v(s_h,a_h,s_{h+1})\right)^2\right]\nonumber\\
     &\quad\le 
     \EE_{(s_h,a_h)\sim \pi^s}\left[\max_{v \in \cV_h}\left((\EE_{s_{h+1}\sim  \PP_{h,f^k}(\cdot \mid s_h,a_h)}-\EE_{s_{h+1}\sim \PP_{h,f^*}(\cdot \mid s_h,a_h)})v(s_h,a_h,s_{h+1})\right)^2\right]\nonumber
\end{align}
The last two inequalities use Jensen's inequality. Hence, by the definition of total variation distance, we can get
\begin{align}
     |\langle W_h(f^k), X_h(\pi^s)\rangle|^2
     &\le \mathrm{TV}\left(\PP_{h,f^k}(\cdot \mid s_h,a_h), \PP_{h,f^*}(\cdot \mid s_h,a_h)\right)^2\label{eq:TV distance}\\
     &\le 2D_\mathrm{H}^2 \left(\PP_{h,f^k}(\cdot \mid s_h,a_h), \PP_{h,f^*}(\cdot \mid s_h,a_h)\right),\label{ineq: DH}
\end{align}
where the $\mathrm{TV}(\cdot,\cdot)$ denotes the total variation distance and $D_{\mathrm{H}}$ denotes the Hellinger divergence. The inequality \eqref{eq:TV distance} holds by the fact that $v(s_h,a_h,s_{h+1}) \in [0,1]$, and the \eqref{ineq: DH} holds by the relationship between TV distance and Hellinger distance.
Then, we can substitute the inequality \eqref{ineq: DH} and get 
\begin{align}
    (Y)&\le \frac{1}{\kappa_{\mathrm{wit}}}\sqrt{\cD(\varepsilon)H \left(\sum_{k=1}^K \sum_{h=1}^H \sum_{s=1}^k \EE_{(s_h,a_h)\sim \pi^s}2D_\mathrm{H}^2 \left(\PP_{h,f^k}(\cdot \mid s_h,a_h), \PP_{h,f^*}(\cdot \mid s_h,a_h)\right)\right)}\nonumber\\
    &\le \mu R\cdot \frac{2\cD(\varepsilon)H}{\kappa_{\mathrm{wit}}^2} + \frac{1}{\mu R}\left(\sum_{k=1}^K \sum_{h=1}^H \sum_{s=1}^k \EE_{(s_h,a_h)\sim \pi^s}D_\mathrm{H}^2 \left(\PP_{h,f^k}(\cdot \mid s_h,a_h), \PP_{h,f^*}(\cdot \mid s_h,a_h)\right)\right)\label{eq: Y upper bound}
\end{align}
Hence, combining \eqref{regret decompose step 1}, \eqref{A decompose to X and Y}, \eqref{eq: X upper bound} and \eqref{eq: Y upper bound}, we can get
\begin{align}
    \text{Reg}(K) &= \sum_{k=1}^K \EE_{s_1\sim \rho}\left[V_{1,f^k}^{(i),\pi^k}(s_1)-V_{1}^{(i),\pi^k}(s_1)\right]\nonumber\\
    &\le R\cdot A + \cD(\varepsilon) HR\nonumber\\
    &\le R(X+Y) + \cD(\varepsilon) HR\nonumber\\
    &\le HKR \varepsilon B_W^2/\kappa_{\mathrm{wit}}^2 + \cD(\varepsilon) HR \nonumber\\&\qquad \quad+ \mu R^2\cdot \frac{2\cD(\varepsilon)H}{\kappa_{\mathrm{wit}}^2} + \frac{1}{\mu}\left(\sum_{k=1}^K \sum_{h=1}^H \sum_{s=1}^k \EE_{(s_h,a_h)\sim \pi^s}D_\mathrm{H}^2 \left(\PP_{h,f^k}(\cdot \mid s_h,a_h), \PP_{h,f^*}(\cdot \mid s_h,a_h)\right)\right)\label{final regret: witness}.
\end{align}
Now by the definition of $\ell^{(i),s}$ of the model-based problem in \eqref{def:modelbased discrepancy}, choosing $\varepsilon=\kappa_{\mathrm{wit}}^2/HKB_W^2$ and $d_{\mathrm{MADC}} = \frac{2R^2\cD(\varepsilon)H}{\kappa_{\mathrm{wit}}^2},$ 
we can get 
\begin{align}
    \text{Reg}(K) 
    &\le 6d_{\mathrm{MADC}}H  + \mu \cdot d_{\mathrm{MADC}} + \frac{1}{\mu} \sum_{k=1}^K \sum_{s=1}^{k-1} \ell^{(i),s}(f^k,\pi^k)\nonumber
\end{align}

complete the proof by $\cD(\kappa_{\mathrm{wit}}^2/HKB_W^2) = \widetilde{\cO}(d)$.
\end{proof}
\subsection{Proof of Theorem \ref{thm:linear mixture}}\label{appendix:proof of linear mixture}
\begin{proof}
First, we fix an index $i \in [n]$. Similar to Section \ref{appendix:proof of witnessrank}, we can get 
\begin{align*}
    &\sum_{k=1}^K \EE_{s_1\sim \rho}\left[V_{1,f^k}^{(i),\pi^k}(s_1)-V_{1}^{(i),\pi^k}(s_1)\right]\\
    &\quad=\sum_{k=1}^K\EE_{\pi^k}\left[\sum_{h=1}^H (Q_{h,f^k}^{(i),\pi^k}(s_h,a_h)-r_h^{(i)}(s_h,a_h)-\EE_{s_{h+1}\sim \PP_{h,f^*}(\cdot \mid s_h,a_h)}V_{h+1,f^k}^{(i),\pi^k}(s_{h+1}, a_{h+1}))\right]\\
    &\quad= \sum_{k=1}^K\EE_{\pi^k}\left[\sum_{h=1}^H (\EE_{s_{h+1}\sim \PP_{h,f^k}(\cdot \mid s_h,a_h)}-\EE_{s_{h+1}\sim \PP_{h,f^*}(\cdot \mid s_h,a_h)})[V_{h+1,f^k}^{(i),\pi^k}(s_{h+1})]\right]\\
    &\quad= \sum_{k=1}^K \sum_{h=1}^H (\theta_{h,f^k}-\theta_h^*)^T\EE_{\pi^k}\left[\int_\cS \phi_h(s'\mid s,a) V_{h+1,f^k}^{(i),\pi^k}(s') \mathrm{d}s \right],
\end{align*}
where the last equality is because of the property of the linear mixture MG.

Now we denote \begin{gather}W_{h}(f) = R(\theta_{h,f}-\theta_h^*)\label{def of W in linear mixture}\\ X_{h}(f,\pi) = \EE_{\pi}\left[\frac{\int_\cS \phi_h(s'\mid s,a) V_{h+1,f}^{(i),\pi}(s') \mathrm{d}s}{R} \right].\label{def of X in linear mixture}\end{gather} Then, we have  $\Vert W_h(f)\Vert \le 2\sqrt{d}$, $\Vert X_h(f,\pi)\Vert \le 1$ and  
\begin{align*}
    \sum_{k=1}^K \EE_{s_1\sim \rho}\left[V_{1,f^k}^{(i),\pi^k}(s_1)-V_{1}^{(i),\pi^k}(s_1)\right] \le \sum_{k=1}^K \sum_{h=1}^H \min\{\langle W_{h}(f^k) , X_h(f^k,\pi^k)\rangle,R\}.
\end{align*}
Now similar to Section \ref{appendix:proof of witnessrank}, if we replace $X_h(\pi^k)$ to $X_h(f^k,\pi^k)$, from \eqref{eq: X upper bound} and \eqref{eq: Y upper bound} with $B_W = 2\sqrt{d}R$ we can 
get 
\begin{align}
    &\sum_{k=1}^K \EE_{s_1\sim \rho}\left[V_{1,f^k}^{(i),\pi^k}(s_1)-V_{1}^{(i),\pi^k}(s_1)\right]\nonumber\\
    &\quad\le HKR\varepsilon 4d R^2 +  \cD(\varepsilon) H R + \mu R^4 \cdot 2\cD(\varepsilon)H + \frac{1}{\mu R^2}\left(\sum_{k=1}^K \sum_{h=1}^H \sum_{s=1}^k \langle W_h(f^k), X_h(f^s,\pi^s)\rangle^2\right),
 \label{ineq:linear mixture}
\end{align}
where $\cD(\varepsilon) = 2d\log \left(\frac{\varepsilon+K}{\varepsilon}\right)$.
Moreover, by \eqref{def of W in linear mixture} and \eqref{def of X in linear mixture}, note that 
\begin{align*}
    \langle W_h(f^k), X_h(f^s,\pi^s)\rangle &=(\theta_{h,f^k}-\theta_h^*)^T \EE_{\pi^s}\left[\int_\cS \phi_h(s'\mid s,a) V_{h+1,f^s}^{(i),\pi^s}(s')\mathrm{d}s\right]\\
    &= \EE_{\pi^s}\left[(\EE_{s_{h+1}\sim \PP_{h,f^k}(\cdot \mid s_h,a_h)}-\EE_{s_{h+1}\sim \PP_{h,f^*}(\cdot \mid s_h,a_h}) [V_{h+1,f^s}^{(i),\pi^s}(s_{h+1})]\right]\\
    &\le \EE_{\pi^s}\left[2\Vert V_{h+1,f^s}^{(i),\pi^s}(\cdot) \Vert_\infty \cdot d_{\mathrm{TV}}(\PP_{h,f^k}(\cdot \mid s_h,a_h)\Vert \PP_{h,f^*}(\cdot \mid s_h,a_h))\right]\\
    &\le \EE_{\pi^s}\left[2\sqrt{2}R D_{\mathrm{H}}(\PP_{h,f^k}(\cdot \mid s_h,a_h)\Vert \PP_{h,f^*}(\cdot \mid s_h,a_h))\right].
\end{align*}
Hence, from \eqref{ineq:linear mixture} and Jensen's inequality that $(\EE[X])^2 \le \EE[X^2]$, we can have 
\begin{align*}
    \text{Reg}(K) &\le \sum_{k=1}^K \EE_{s_1\sim \rho}\left[V_{1,f^k}^{(i),\pi^k}(s_1)-V_{1}^{(i),\pi^k}(s_1)\right]\\&  \le HKR\varepsilon 4d R^2 +  \cD(\varepsilon) H R + \mu R^4 \cdot 2\cD(\varepsilon)H \\&\qquad  + \frac{1}{\mu  R^2}\left(\sum_{k=1}^K \sum_{h=1}^H \sum_{s=1}^k \EE_{(s_h,a_h)\sim \pi^s}\left[8R^2D_\mathrm{H}^2 \left(\PP_{h,f^k}(\cdot \mid s_h,a_h), \PP_{h,f^*}(\cdot \mid s_h,a_h)\right)\right]\right).
\end{align*}
By the definition of discrepancy function $\ell^{(i),s}$ in \eqref{def:modelbased discrepancy}, and choosing $\varepsilon=1/HKd$, $d_{\mathrm{MADC}} = HR^4\cD(1/HKd) = \widetilde{\cO}(HdR^4)$, we can derive 
\begin{align}
    \text{Reg}(K) &\le 4R^3 + \cD(1/HKd) HR + \mu d_{\mathrm{MADC}} + \frac{1}{\mu} \sum_{k=1}^K \sum_{s=1}^{k-1}\ell^{(i),s}(f^k,\pi^k)\nonumber\\
    &\le 6d_{\mathrm{MADC}} H + \mu d_{\mathrm{MADC}} + \frac{1}{\mu} \sum_{k=1}^K \sum_{s=1}^{k-1}\ell^{(i),s}(f^k,\pi^k).\nonumber
\end{align}
Hence, we complete the proof.
\end{proof}
\subsection{Proof of Lemma~\ref{lemma:concentration}}\label{sec: proof of lemma concentration}
\begin{proof}
    The proof is modified from \cite{zhong2022gec}. Define $\cW_{j,h}$ be the filtration induced by $\{s_1^k, a_1^k, r_1^{(i),k},\cdots, s_{H}^k, a_H^k, r_H^{(i),k}\}_{k=1}^{j-1}$.
    First, for $h \in [H], i \in [n], f \in \cF^{(i)}$ and $ \pi \in \Pi$, we define the random variable 
    \begin{align*}
        Y_j^{(i)}(h,f,\zeta^k) = \Big(&f_h(s_h^j, a_h^j)-r_h^{(i)}(s_h^j,a_h^j)-\langle f_{h+1}(s_{h+1}^j\cdot), \zeta^k_{h+1}(\cdot \mid s_{h+1}^j)\rangle \Big)^2 \\&-\Big(\cT^{ (i),\zeta^k}_h(f)(s_h^j, a_h^j)-r_h^{(i)}(s_h^j,a_h^j)-\langle f_{h+1}(s_{h+1}^j\cdot), \zeta^k_{h+1}(\cdot \mid s_{h+1}^j)\rangle\Big)^2.
    \end{align*}
    By taking conditional expectation of $Y_j$ with respect to $a_h^j,s_h^j$, 
    we can get 
    \begin{align*}
        \EE[Y_j^{(i)}(h,f,\zeta^k)\mid \cW_{j,h}] &= \EE_{s_h,a_h\sim \zeta^j}[(f_h-\cT_h^{(i),\zeta^k}(f)) (s_h,a_h)]^2
    \end{align*}
    and 
    \begin{align*}
        \EE[(Y_j^{(i)}(h,f,\zeta^k))^2\mid \cW_{j,h}] &\le 2R^2 \EE[Y_j^{(i)}(h,f,\zeta^k)\mid \cW_{j,h}],
    \end{align*}
    where $3R\ge |
f_h(s_h^j, a_h^j) - r_h^{(i)}(s_h^j,a_h^j)-\langle f_{h+1}(s_{h+1}^j\cdot), \zeta^k_{h+1}(\cdot \mid s_{h+1}^j)\rangle )|$ is the constant upper bound. 
Denote $Z_j = Y_j^{(i)}(h,f,\zeta^k)-\EE_{s_{h+1}}[Y_j^{(i)}(h,f,\zeta^k)\mid \cW_{j,h}]$ with $|Z_j|\le 4R^2$. By the Freedman inequality, for any $0<\eta<\frac{1}{4R}$, with probability at least $1-\delta$,
\begin{align*}
    \sum_{j=1}^k Z_j &= \cO\left(\eta \sum_{j=1}^k \mbox{Var}[Y_j^{(i)}(h,f,\zeta^k)\mid \cW_{j,h}] + \frac{\log(1/\delta)}{\eta} \right)\\
    &\le \cO\left(\eta\sum_{j=1}^k \EE[(Y_j^{(i)}(h,f,\zeta^k))^2\mid \cW_{j,h}] + \frac{\log(1/\delta)}{\eta}\right)\\
    &\le \cO\left(\eta\sum_{j=1}^k 2R^2 \EE[Y_j^{(i)}(h,f,\zeta^k)\mid \cW_{j,h}] + \frac{\log(1/\delta)}{\eta}\right).
\end{align*}
By choosing $\eta = \min\left\{\frac{1}{4R}, \frac{\sqrt{\log (1/\delta)}}{\sqrt{2}R\sqrt{\sum_{j=1}^k\EE[Y_j^{(i)}(h,f,\zeta^k)\mid \cW_{j,h}]}}\right\}$, we will have 
\begin{align*}
    \sum_{j=1}^k Z_j &=\cO\left(R\sqrt{\sum_{j=1}^k\EE[Y_j^{(i)}(h,f,\zeta^k)\mid \cW_{j,h}]\log(1/\delta)} + R^2\log(1/\delta)\right).
\end{align*}
Similarly, if we apply the Freedman's inequality with $-\sum_{j=1}^k Z_j$, with probability at least $1-2\delta$,
\begin{align*}
    \Bigg|\sum_{j=1}^k Z_j \Bigg|&=\cO\left(R\sqrt{\sum_{j=1}^k\EE[Y_j^{(i)}(h,f,\zeta^k)\mid \cW_{j,h}]\log(1/\delta)} + R^2\log(1/\delta)\right).
\end{align*}
Denote the $\rho$-covering set of $\cF^{(i)}$ as $\cC_{\cF^{(i)}}(\rho)$, 
then for any $f \in \cF^{(i)}, \zeta \in \Pi_i^{\text{pur}}$, there exists a pair $\tilde{f}  \in \cC_{\cF^{(i)}}(\rho)$ such that 
\begin{align*}
    \Big|\Big(f_h(s_h, a_h)-&r_h^{(i)}(s_h,a_h)-\langle f_{h+1}(s_{h+1},\cdot),\zeta_{h+1}(\cdot\mid s_{h+1})\rangle\Big)\\&-\Big(\tilde{f}_h(s_h, a_h)-r_h^{(i)}(s_h,a_h)-\langle\tilde{f}_{h+1}(s_{h+1},\cdot),\zeta_{h+1}(\cdot \mid s_{h+1})\rangle\Big)\Big|\le 3\rho
\end{align*}
for all $(s_h, a_h, s_{h+1}) \in \cS \times \cA\times \cS.$ Now by taking a union bound over $\cC_{\cF^{(i)}}(\rho)$, we have that with probability at least $1-\delta$, for all $\tilde{f} \in \cC_{\cF^{(i)}}(\rho)$,
\begin{align}
    &\ \ \ \ \ \Bigg|\sum_{j=1}^k \tilde{Y}_j^{(i)}(h,\tilde{f}, \zeta) -\sum_{j=1}^k \EE[\tilde{Y}_j^{(i)}(h,\tilde{f}, \zeta)\mid \cW_{j,h}]\Bigg|\nonumber\\&=\cO\left(R\sqrt{\sum_{j=1}^k\EE[\tilde{Y}_j^{(i)}(h,\tilde{f},\zeta)\mid \cW_{j,h}]\iota} + R^2\iota\right),\label{ineq:loose}
\end{align}
where $\iota = 2\log(HK|\cC_{\cF^{(i)}}(\rho)|/\delta) \le 2\log(HK\cN_{\cF^{(i)}}(\rho))$.

Now note that for all $f \in \cF^{(i)}, \zeta \in \Pi_i^{\mathrm{pur}}$, we have 
\begin{align*}
   &\sum_{h=1}^H \sum_{j=0}^{k-1} Y_j^{(i)}(h,f,\zeta) \\&\quad= \sum_{h=1}^H \sum_{j=0}^{k-1}(f_h(s_h^j, a_h^j)-r_h^{(i)}(s_h^j,a_h^j)-\langle f_{h+1}(s_{h+1}^j,\cdot),\zeta_{h+1}(\cdot \mid s_{h+1}^j)\rangle)^2 \\&\qquad\quad-(\cT^{(i),\zeta}_h(f)(s_h^j, a_h^j)-r_h^{(i)}(s_h^j,a_h^j)-\langle f_{h+1}(s_{h+1}^j,\cdot),\zeta_{h+1}(\cdot \mid s_{h+1}^j)\rangle)^2\\
    &\quad\le \sum_{h=1}^H \sum_{j=0}^{k-1}(f_h(s_h^j, a_h^j)-r_h^{(i)}(s_h^j,a_h^j)-\langle f_{h+1}(s_{h+1}^j,\cdot),\zeta_{h+1}(\cdot \mid s_{h+1}^j)\rangle)^2 \\&\qquad\quad-\inf_{f'_h \in \cF_h^{(i)}}(\cT^{(i),\zeta^k}_h(f')(s_h^j, a_h^j)-r_h^{(i)}(s_h^j,a_h^j)-\langle f_{h+1}(s_{h+1}^j,\cdot),\zeta_{h+1}(\cdot \mid s_{h+1}^j)\rangle)^2\\&\quad=  L^{(i),k-1}(f,\zeta,\tau^{1:k-1}).
\end{align*}
Then, by \eqref{ineq:loose} we can get 
\begin{align*}
    \sum_{h=1}^H \sum_{j=0}^{k-1} \EE[\tilde{Y}_j^{(i)}(h,\tilde{f}, \zeta)\mid \cW_{j,h}] \le 4 L^{(i),k-1}(\tilde{f},\zeta,\tau^{1:k-1})+\cO(HR^2\iota).
\end{align*}

Now similar to \citep{jin2021bellman}, by the definition of $\rho$-covering number, for any $k \in [K]$, $f \in\cF^{(i)}$ and $\zeta \in \Pi_i^{\text{pur}}$, 
\begin{align*}
    \sum_{h=1}^H\sum_{j=0}^{k-1} \EE[Y_j^{(i)}(h,f, \zeta)\mid \cW_{j,h}] \le 4 L^{(i),k-1}(f,\zeta,\tau^{1:k-1})+\cO(HR^2\iota + HRk\rho).
\end{align*}
Now since $s_h^j,a_h^j \sim \zeta^j$, we can have 
\begin{align*}
    \sum_{j=0}^{k-1}\ell^{j,(i)}(f,\zeta^k)& = \sum_{j=0}^{k-1} \EE[Y_j^{(i)}(h,f, \zeta^k)\mid \cW_{j,h}]\\ &\le 4 L^{(i),k-1}(f,\zeta^k,\tau^{1:k-1})+\cO(HR^2\iota + HRk\rho).
\end{align*}
We complete the proof by choosing $\rho = 1/K$ and choose $\varepsilon_{\mathrm{conc}} = \cO(HR^2\iota + HRk\rho) = \cO(HR^2\iota)$.
\end{proof}
\subsection{Proof of Lemma~\ref{lemma:optimal concentration}}\label{sec: proof of lemma optimal concentration}
\begin{proof}
    First, for any $f \in \cF^{(i)}$ and $\pi \in \Pi^{\mathrm{pur}}$ we define  the random variable 
    \begin{align*}
        Q_j^{(i)}(h,f, \pi) = (f_h&(s_h^j,a_h^j)-r_h^{(i)}(s_h^j,a_h^j) - \langle f^*_{h+1}(s_{h+1}^j,\cdot),\pi_{h+1}(\cdot \mid s_{h+1}^j)\rangle)^2  \\&-(f_h^*(s_h^j,a_h^j)-r_h^{(i)}(s_h^j,a_h^j) - \langle f^*_{h+1}(s_{h+1}^j,\cdot),\pi_{h+1}(\cdot \mid s_{h+1}^j)\rangle)^2.
    \end{align*}
    Then, by similar derivations in Lemma~\ref{lemma:concentration}, we can get 
    \begin{align*}
        \EE[Q_j^{(i)}(h,f,\pi)\mid \cW_{j,h}] = \EE_{s_h,a_h\sim \zeta^j}[(f_h - \cT^{(i),\pi}(f^*))(s_h,a_h)]^2 \ge 0,\\
        \EE[(Q_j^{(i)}(h,f,\pi))^2\mid \cW_{j,h}] \le 2R^2   \EE[Q_j^{(i)}(h,f,\pi)\mid \cW_{j,h}].
    \end{align*}
    Then, by Freedman's inequality, with probability at least $1-\delta$,  for all elements in $\tilde{f}\in \cC_{\cF^{(i)}}(\rho)$, we have 
    \begin{align*}
        &\ \ \ \ \ \Bigg|\sum_{j=0}^{k-1} \tilde{Q}_j^{(i)}(h,\tilde{f}, \pi) -\sum_{j=0}^{k-1} \EE[\tilde{Q}_j^{(i)}(h,\tilde{f}, \pi)\mid \cW_{j,h}]\Bigg|\\&=\cO\left(R\sqrt{\sum_{j=0}^{k-1}\EE_{s_{h+1}}[\tilde{Q}_j^{(i)}(h,\tilde{f},\pi)\mid \cW_{j,h}]\iota} + R^2\iota\right),
    \end{align*}
    then we can have
    \begin{align*}
        \sum_{j=0}^{k-1} \tilde{Q}_j^{(i)}(h,\tilde{f},\pi) \ge -\cO(R^2\iota).
    \end{align*}
    Thus, by the definition of $\cC_{\cF^{(i)}}(\rho)$, for all $f \in \cF^{(i)}$ and $\pi \in \Pi_i^{\text{pur}}$, we have 
    \begin{align*}
        -\sum_{j=0}^{k-1} Q_j^{(i)}(h,f, \pi) \le \cO(R^2 \iota + Rk\rho).
    \end{align*}
    Thus, 
    \begin{align*}
        L^{(i),k}(f,\pi)  = \sum_{h=1}^H\left(-\inf_{f \in \cF^{(i)}}\sum_{j=0}^{k-1} Q_j^{(i)}(h,f, \pi)\right)\le \cO(HR^2\iota + HRk\rho) = \cO(HR^2\iota).
    \end{align*}
    Thus, we complete the proof.
\end{proof}
\subsection{Proof of Lemma \ref{lemma:concentration:model-based}}\label{sec: proof of lemma concentration model based}
\begin{proof}
For simplicity, we first assume $\cF$ is a finite class.
Given a model $f \in \cF$ and $h \in [H]$, we define $X_{h,f}^j = \log \frac{\PP_{h,f^*}(s_{h+1}^j\mid s_h^j, a_h^j)}{\PP_{h,f}(s_{h+1}^j\mid s_h^j, a_h^j)}$.
Thus, 
\begin{align}\label{eq:decompose L diff}
    L^{(i),k}(f^*, \tau^{1:k}) - L^{(i),k}(f,\tau^{1:k}) = -\sum_{h=1}^H \sum_{j=1}^k X_{h,f}^j.
\end{align}
Now we define the filtration $\cG_{j}$ as 
\begin{align*}
    \cG_{j} = \sigma(\{s_h^1,a_h^1,\cdots,s_h^j,a_h^j\}).
\end{align*}
Then, by Lemma~\ref{lemma:foster}
for all $f \in \cF$, with probability at least $1-\delta$, we have 
\begin{align*}
    -\sum_{j=1}^k X_{h,\bar{f}}^j \le \sum_{j=1}^k \log \EE\left[\exp\left\{-\frac{1}{2}X_{h,\bar{f}}^j\right\}\Bigg|\  \cG_{j-1}\right] + \log (H |\cF|/\delta).
\end{align*}
Now we decompose the first term at the right side as 
\begin{align*}
     &\EE\left[\exp\left\{-\frac{1}{2}X_{h,\bar{f}}^j\right\}\Bigg|\  \cG_{j-1}\right]\\
     &\quad= \EE\left[\sqrt{\frac{\log \PP_{h,f}(s_{h+1}^j\mid s_h^j,a_h^j)}{\PP_{h,f^*}(s_{h+1}^j\mid s_h^j,a_h^j)}}\Bigg|\  \cG_{j-1}\right]\\
     &\quad= \EE_{(s_h^j,a_h^j)\sim \pi^j}\EE_{s_{h+1}\sim \PP_{h,f^*}(\cdot \mid s_h^j,a_h^j)}\left[\sqrt{\frac{ \PP_{h,f}(s_{h+1}^j\mid s_h^j,a_h^j)}{\PP_{h,f^*}(s_{h+1}^j\mid s_h^j,a_h^j)}}\Bigg|\  \cG_{j-1}\right]\\&\quad= 
      \EE_{(s_h^j,a_h^j)\sim \pi^j}\left[\int\sqrt{ \PP_{h,f}(s_{h+1}^j\mid s_h^j,a_h^j)\PP_{h,f^*}(s_{h+1}^j\mid s_h^j,a_h^j)}d_{s_{h+1}^j}\right]\\
      &\quad=1-\frac{1}{2}\EE_{(s_h^j,a_h^j)\sim \pi^j}[D_\mathrm{H}^2(\PP_{h,f}(s_{h+1}^j\mid s_h^j,a_h^j) \Vert\PP_{h,f^*}(s_{h+1}^j\mid s_h^j,a_h^j) )].
\end{align*}
Now by the inequality $\log x \le x-1$, we have 
\begin{align*}
    -\sum_{j=1}^k X_{h,f}^j &\le \sum_{j=1}^k\left(1-\frac{1}{2}\EE_{(s_h^j,a_h^j)\sim \pi^j}[D_\mathrm{H}^2(\PP_{h,f}(s_{h+1}^j\mid s_h^j,a_h^j) \Vert\PP_{h,f^*}(s_{h+1}^j\mid s_h^j,a_h^j) )]\right) \\&\qquad\quad- 1 + \log(H|\cF|/\delta)\\
    &\le -\sum_{j=1}^k\frac{1}{2}\EE_{(s_h^j,a_h^j)\sim \pi^j}\left[D_\mathrm{H}^2(\PP_{h,f}(s_{h+1}^j\mid s_h^j,a_h^j) \Vert\PP_{h,f^*}(s_{h+1}^j\mid s_h^j,a_h^j) )\right] \\&\qquad\quad+ \log(H|\cF|/\delta).
\end{align*}
Sum over $h \in [H]$ with \eqref{eq:decompose L diff}, we can complete the proof by
    \begin{align*}
        -\sum_{h=1}^H \sum_{j=1}^k X_{h,f}^j \le -\sum_{j=1}^k \ell^{(i),j}(f) + \kappa_{\mathrm{conc}}, 
    \end{align*}
    where $\kappa_{\mathrm{conc}} = H\log (H|\cF|/\delta)$. For infinite model classes $\cF$, we can use $1/K$-bracketing number $\cB_\cF(1/K)$ to replace the cardinality $|\cF|$ \citep{liu2022partially,zhong2022gec,zhan2022pac}.
    \end{proof}
\subsection{Proof of Corollary \ref{coroll:sample complexity}}\label{sec: corollary}
\begin{proof}
We provide the proof for NE. The proof for CCE/CE are the same by replacing the NE-regret to the CCE/CE-regret.
    By Theorem \ref{thm:mainresult}, with probability at least $1-\delta$, 
    \begin{align*}
        \frac{1}{K}\left(\sum_{k=1}^K \sum_{i=1}^n \bigl(V ^{(i),\mu^{(i),\pi^k}}(\rho)-V ^{(i),\pi^k}(\rho)\bigr)\right)&=\frac{1}{K}\text{Reg}_{\mathrm{NE}}(K)\\
        &\le \widetilde{\cO}\left(\frac{nH\Upsilon_{\cF,\delta}}{\sqrt{K}} + \frac{nd_{\mathrm{MADC}}}{\sqrt{K}} + \frac{nd_{\mathrm{MADC}}H}{K}\right).
    \end{align*}
    Hence, by choosing $K = \widetilde{\cO}\left((n^2H^2\Upsilon_{\cF,\delta}^2 + n^2d_{\mathrm{MADC}}^2 )\cdot \varepsilon^{-2} + nd_{\mathrm{MADC}} H \cdot \varepsilon^{-1}\right)$ with $\varepsilon<1$, we have 
    \begin{align*}
        &\max_{i \in [n]} \bigl(V ^{(i),\mu^{(i),\pi_{\mathrm{out}}}}(\rho)-V ^{(i),\pi_{\mathrm{out}}}(\rho)\bigr)\\
        &\quad \le  \sum_{i=1}^n \bigl(V ^{(i),\mu^{(i),\pi_{\mathrm{out}}}}(\rho)-V ^{(i),\pi_{\mathrm{out}}}(\rho)\bigr)\\&\quad \le \frac{1}{K}\left(\sum_{k=1}^K \sum_{i=1}^n \bigl(V ^{(i),\mu^{(i),\pi^k}}(\rho)-V ^{(i),\pi^k}(\rho)\bigr)\right) \\&\quad \le \varepsilon,
    \end{align*}
    where the second inequality holds from $\pi_{\mathrm{out}} = \mathrm{Unif}(\{     \pi^k \}_{k\in[K]})$.
    Hence, $\pi_{\mathrm{out}}$ is a $\varepsilon$-NE.
\end{proof}
\section{Technical Tools}
We provide the following lemma to complete the proof of model-based RL problems. The detailed proof can be found in \citep{foster2021statistical}.
\begin{lemma}\label{lemma:foster}
    For any real-valued random variable sequence $\{X_k\}_{k \in [K]}$ adapted to a filtration $\{\cG_k\}_{k \in [K]}$, with probability at least $1-\delta$, for any $k \in[K]$, we can have 
    \begin{align*}
        -\sum_{s=1}^k X_k \le \sum_{s=1}^k \log\EE[\exp(-X_s)\mid \cF_{s-1}]+\log (1/\delta).
    \end{align*}
    
\end{lemma}

In the next lemma, we introduce the Freedman's inequality, which has been commonly used in previous RL algorithms. \citep{jin2021v,chen2022abc,zhong2022gec}
\begin{lemma}[Freedman's Inequality \citep{agarwal2014taming}]\label{lemma:freedman}
Let $\{Z_k\}_{k \in [K]}$ be a martingale difference sequence that adapted to filtration $\{\cF_k\}_{k \in [K]}$. If $|Z_k| \le R$ for all $k \in [K]$, then for $\eta \in (0,\frac{1}{R})$, with probability at least $1-\delta$, we can have 
\begin{align*}
    \sum_{k=1}^K X_k =  \cO\left(\eta \sum_{k=1}^K \EE[X_k^2\mid \cF_{k-1}] + \frac{\log(1/\delta)}{\eta}\right).
\end{align*}
\end{lemma}
The next elliptical potential lemma is first introduced in the linear bandit literature \citep{dani2008stochastic, abbasi2011improved} and then applied to the RL problems with Bilinear Classes \citep{du2021bilinear} and the general function approximation \citep{chen2022partially,zhong2022gec}.
\begin{lemma}[Elliptical Potential Lemma]\label{lemma:elliptical }
    Let $\{x_k\}_{k=1}^K$ be a sequence of real-valued vector, i.e. $x_k \in \RR^d$ for any $k \in [K]$. Then, if we define $\Lambda_i = \varepsilon I + \sum_{k=1}^K x_k x_k^T$, we can get that 
    \begin{align*}
        \sum_{k=1}^K \min\left\{1, \Vert x_i \Vert^{2}_{\Lambda_i^{-1}}\right\}\le 2\log\left(\frac{\mathrm{det}(\Lambda_{K+1})}{\mathrm{det}(\Lambda_1)}\right)\le 2\log \det\left(I + \frac{1}{\varepsilon}\sum_{k=1}^K x_kx_k^T\right).
    \end{align*}
\end{lemma}
\begin{proof}
    The proof is provided in  Lemma 11 of \citep{abbasi2011improved}.
\end{proof}
}

\end{sloppypar}
\end{document}